\documentclass[journal,twoside]{IEEEtran}
\usepackage[english]{babel}
\usepackage[T1]{fontenc}

\usepackage{calc}
\usepackage[square,numbers,sort&compress]{natbib}
\usepackage{graphicx}
\usepackage{algpseudocode}
\usepackage[ruled,vlined,linesnumbered]{algorithm2e}
\SetCommentSty{textrm}
\SetAlCapNameFnt{\small}
\SetAlCapFnt{\small}
\usepackage{etoolbox}
\makeatletter
\patchcmd{\@algocf@start}% <cmd>
{-1.5em}% <search>
{-5pt}% <replace>
{}{}% <success><failure>
\makeatother
\usepackage{float}
\usepackage[usenames]{color}
\usepackage{afterpage}
\usepackage{multirow}
\usepackage{bm}
\usepackage{amsmath,amssymb}
\usepackage{amsthm}
\usepackage{booktabs}
\usepackage[colorlinks,allcolors=blue]{hyperref}
\usepackage{url}
\usepackage{ifpdf}
\usepackage{mathtools} % \DeclarePairedDelimiter

\usepackage{array}
\usepackage[newcommands]{ragged2e}
\newcommand{\PreserveBackslash}[1]{\let\temp=\\#1\let\\=\temp}
\newcolumntype{L}[1]{>{\PreserveBackslash\raggedright\hyphenpenalty=10000}p{#1}}
\newcolumntype{R}[1]{>{\PreserveBackslash\raggedleft}p{#1}}
\newcolumntype{C}[1]{>{\PreserveBackslash\centering}m{#1}}
\usepackage{newrevisor}
\newrevisor{manuel}{violet!75}
\newrevisor{miqing}{blue}
\newrevisor{xin}{orange!95!black}
\newtheorem{definition}{Definition}
\newtheorem{theorem}{Theorem}
\newtheorem{lemma}{Lemma}
\newtheorem{property}{Property}

\DeclareMathOperator*{\argmin}{arg\,min}

\DeclarePairedDelimiter\floor{\lfloor}{\rfloor}

\DeclarePairedDelimiter\abs{\lvert}{\rvert}
\newcommand{\Better}{\ensuremath{\vartriangleleft}}

\newcommand{\assign}{\ensuremath{\leftarrow}}
\newcommand{\AHV}{\textrm{A}\ensuremath{_\textsc{hv}}\xspace}
\newcommand{\Adom}{\textrm{A}\ensuremath{_\text{dom}}\xspace}
\newcommand{\ARtwo}{\textrm{A}\ensuremath{_\text{R2}}\xspace}
\newcommand{\IGDplus}{IGD$^{+}$}

\begin{document}
	\title{Multi-Objective Archiving}
	
	\author{Miqing~Li, Manuel L\'{o}pez-Ib\'{a}\~{n}ez, Xin~Yao%
		
		\thanks{M. Li is with the School of Computer Science, University of Birmingham, U.~K. (e-mail: m.li.8@bham.ac.uk).} %
		
		\thanks{M. L\'{o}pez-Ib\'{a}\~{n}ez is with Alliance Manchester Business School, University of Manchester, U.~K. (e-mail:~manuel.lopez-ibanez@manchester.ac.uk).}%
		
		\thanks{X. Yao is with the Research Institute of Trustworthy Autonomous Systems and Guangdong Provincial Key Laboratory of Brain-inspired Intelligent Computation, Department of Computer Science and Engineering, Southern University of Science and Technology, Shenzhen, P. R. China,
			and is also with the School of Computer Science, University of Birmingham, U.~K. (e-mail: xiny@sustech.edu.cn).}}
	
	\markboth{IEEE Transactions on Evolutionary Computation,~Vol.~xx, No.~xx, Year.~xx}%
	{Li, López-Ibáñez, and Yao: Multi-Objective Archiving}
	
	
	\maketitle
	
	\begin{abstract}
		Most multi-objective optimisation algorithms maintain an archive explicitly or implicitly during their search.
		Such an archive can be solely used to store high-quality solutions presented to the decision maker,
		but in many cases may participate in the search process
		(e.g., as the population in evolutionary computation).
		%(e.g., as a solution pool from which the parental solutions are selected,
		%which can be regarded as the population in evolutionary search).
		Over the last two decades,
		archiving,
		the process of comparing new solutions with previous ones and deciding how to update the archive/population,
		stands as an important issue in evolutionary multi-objective optimisation (EMO).
		This is evidenced by constant efforts from the community on developing various effective archiving methods,
		ranging from conventional Pareto-based methods to more recent indicator-based and decomposition-based ones.
		However,
		the focus of these efforts is on empirical performance comparison in terms of specific quality indicators;
		there is lack of systematic study of archiving methods from a general theoretical perspective.
		%\manuel{A consequence of ignoring that is that
		%the archive/population may deteriorate with respect to Pareto dominance during the search.}{}\MANUEL{I removed this because it doesn't flow very well from the previous sentence.}
		In this paper,
		we attempt to conduct a systematic overview of multi-objective archiving,
		in the hope of paving the way to
		understand archiving algorithms from a holistic perspective of theory and practice,
		and more importantly providing a guidance on
		how to design theoretically desirable and practically useful archiving algorithms.
		In doing so,
		we also present that archiving algorithms based on weakly Pareto compliant indicators (e.g., $\epsilon$-indicator),
		as long as designed properly,
		can achieve the same theoretical desirables as archivers based on Pareto compliant indicators (e.g., hypervolume indicator).
		Such desirables include the property \emph{limit-optimal},
		the limit form of the possible optimal property that a bounded archiving algorithm can have with respect to the most general form of superiority between solution sets.
		
		%which it is commonly believed that only Pareto compliant indicator-based archivers can hold.

	\end{abstract}
	
	
	%\MANUEL{The title is perhaps too short, if it is a survey, it should mention the word survey or review, no?}
	
	%\MANUEL{I'm changing the references little-by-little to use this database: \url{https://iridia-ulb.github.io/references/}, but I'm checking that the two versions are not different, or if they are, that the version in the IRIDIA database is the correct one.}
	
	%\XIN{Test}
	
	\section{Introduction}
	\IEEEPARstart{M}{ulti}-objective optimisation refers to an optimisation scenario
	where several conflicting objectives are optimized simultaneously.
	A prominent feature of a multi-objective optimisation problem (MOP) is that,
	in contrast to its single-objective counterpart,
	it does not have a single optimal solution,
	but rather a set of trade-off solutions,
	called Pareto-optimal solutions
	or the Pareto front in the objective space,
	whose size is usually prohibitively large or even infinite.
	
	A common way to solve an MOP is to find a good, but smaller size, approximation to its Pareto front
	and present it to the decision maker, who chooses one solution from the approximation to deploy.
	Multi-objective optimisers designed for this purpose
	maintain an \emph{archive}, i.e., a set of high-quality solutions discovered during the search.
	In this context, \emph{archiving} is the process of updating the archive by comparing new solutions with those already in the archive and deciding which ones to keep and which ones to discard~\cite{KnoCor2003tec,KnoCor2004lnems}.
	
	Archiving becomes even more relevant in evolutionary multi-objective optimisation (EMO),
	in which population-based evolutionary search is performed.
	The population update (i.e., environmental selection)
	in EMO can be seen as an archiving process~\cite{Hanne1999ejor,LopKnoLau2011emo},
	where the population is regarded as an archive
	updated at each generation;
	%and the sequence is ordered by the time of solutions discovered;
	that is,
	new offspring solutions are compared with the ones already in the population,
	either chunk by chunk, e.g., in the generational evolution of NSGA-II~\cite{Deb02nsga2} and SPEA2~\cite{ZitLauThie2002spea2}) or one-by-one, e.g., in the steady-state evolution of SMS-EMOA~\cite{BeuNauEmm2007ejor} and MOEA/D~\cite{ZhaLi07:moead}.
	%possibly replacing previous solutions with new ones.
	%the population takes newly generated solutions,
	%chunk by chunk (e.g., in the generational-state evolution like NSGA-II~\cite{Deb02nsga2} and SPEA2~\cite{ZitLauThie2002spea2})
	%or one by one (e.g., in the steady-state evolution like SMS-EMOA~\cite{BeuNauEmm2007ejor} and MOEA/D~\cite{ZhaLi07:moead}),
	%compares them with its current solutions
	%and then making an update of itself.\MANUEL{I do not think we should say that the population is making actions, it is not an agent. What about: new offspring solutions are compared with the ones already in the population, either chunk by chunk \ldots or one by one \ldots, possibly replacing previous solutions with new ones.}
	
	Ideally, the optimiser would store all Pareto-optimal solutions ever evaluated in an unbounded archive~\citep{FieEveSing2003tec}. In practice, however, the computational cost of maintaining an unbounded archive often makes this approach infeasible.  Thus, all archivers share the common purpose of preserving a solution set of bounded size. Thus, the core issue in archiving is to decide which solutions are kept and which are removed
	when new solutions arrive.
	Many archiving algorithms (\emph{archivers}) have been proposed in the literature using different rules (or selection criteria in the EMO),
	ranging from traditional Pareto-based criteria, which use Pareto dominance to distinguish between solutions,
	to more recent criteria such as indicator-based and decomposition-based ones,
	which rely on a quality indicator and a set of weight vectors, respectively.

	Although there are a number of works dedicated to archiving \cite{Hanne1999ejor,KnoCor2003tec,LopKnoLau2011emo,Li2021telo}, most archiving methods were proposed and studied as part of complete MOEAs and not as independent algorithmic components; which is at odds with the fact
	that the key
	% SHORTEN: distinctive
	characteristic of many MOEAs is their population update (i.e., archiving) rule.
	% Despite standing as one of the most critical components that determine the performance of multi-objective evolutionary algorithms (MOEAs),
	% archiving
	% %\manuel{, on its own, has not received much attention.
	% %It instead has been intensively studied after being incorporated into an MOEA,
	% %i.e., studied as the form of an MOEA rather than an archiver.}
	% has been studied mostly as part of complete MOEAs.
	% This contrasts with the fact that the key distinguishing characteristic of many MOEAs is essentially their archiving rule, i.e., their environmental selection.
	For example,
	the innovation of NSGA-III~\cite{DebJain2014:nsga3-part1} is the method for selecting the solutions that will form the population in the next generation,
	and all of its other components
	% SHORTEN:
	%(e.g., mating selection, crossover and mutation)
	follow common practice.
	% %\MANUEL{to be precise, they are the same as in NSGA-II, no?}
	% %\MIQING{In fact, NSGA-II does consider the rules of mating selection.}
	% Therefore,
	% it could be useful to look into archiving independently,
	% %such as on the issues of the order of solutions fed into the archive,
	% %the batch size, and its connection to other evolutionary components,
	% avoiding the interference of other algorithmic components~\cite{KnoCor2004lnems,LopKnoLau2011emo,Li2021telo}.

	Moreover, most studies of archiving from a theoretical perspective have focused on specific quality indicators~\cite{KnoCor2003tec,LauZen2011ejor,AugBadBroZit2012tcs,BriFri2014convergence}. Except for few exceptions~\citep{Rudolph1998ep,Hanne1999ejor,RudAga2000cec,Hanne2001emo,LopKnoLau2011emo}, it remains largely unclear what are the theoretical properties held by the archiving algorithms used within the state-of-the-art MOEAs; for example,  whether their population may suffer from deterioration in terms of Pareto-optimality,
	or whether they are able to return a maximal subset of the Pareto-optimal solutions discovered so far.
	% Moreover, despite the abundance of empirical investigation and comparison of various archiving criteria
	% (e.g., Pareto-based, indicator-based and decomposition-based ones),
	% there is lack of systematic study of archiving from a theoretical perspective.
	% Although there have been several studies on specific quality indicators~\cite{KnoCor2003tec,LauZen2011ejor,AugBadBroZit2012tcs,BriFri2014convergence},
	% it remains largely unclear how state-of-the-art archiving algorithms behave theoretically and
	% what theoretical properties they may hold,
	% for example,
	% %whether the population maintained by an archiving algorithm can deteriorate with respect to Pareto dominance,
	% whether the population maintained by an archiving algorithm suffers from deterioration,
	% or whether an archiving algorithm is able to return a subset of the Pareto optimal solutions
	% discovered so far.
	Such properties matter, as the decision maker certainly would be unhappy if she is forced to select an inferior solution from the final archive/population
	instead of a better solution, in terms of Pareto-optimality, that was generated but later removed from the archive.
	Many unwelcome phenomena caused by the lack of such properties have been reported in the literature,
	on synthetic solution sequences~\cite{KnoCor2004lnems,LopKnoLau2011emo,Li2021telo},
	benchmark test problems~\cite{LTDZ2002b,FieEveSing2003tec,LiYao2019emo}
	and real-life scenarios~\cite{Fie2017gecco,ReeHadHerKas2013water,CheLiYao2019}.
	%\MANUEL{Note to myself: What phenomena do these papers report?}
	%\MIQING{They reported/thought that the poor performance obtained by the considered MOEAs comes from the possible deterioration of the population.}
	
	In this paper,
	we aim to conduct a comprehensive review of archiving in multi-objective optimisation that includes theoretical studies as well as archivers proposed as part of complete MOEAs.
	We begin by providing background knowledge of multi-objective optimisation (Sec.~\ref{sec:preliminaries}).
	We then define the archiving problem and summarise its history (Sec.~\ref{sec:archiving_problem}).
	In Sec.~\ref{sec:theory},
	we focus on theory of archiving and give desirable properties for archiving algorithms to hold. Our proposed formulation of these properties covers, in an uniform manner, both bounded-size archivers, which only store nondominated solutions, and fixed-size archivers, which may store dominated solutions.
	Moreover,
	we prove that archiving algorithms based on weakly Pareto compliant indicators
	(e.g., $\epsilon$-indicator~\cite{ZitThiLauFon2003:tec})
	can achieve the same theoretical desirability
	as those based on Pareto-compliant indicators
	(e.g., hypervolume~\cite{ZitThi99:spea}).
	%we show the equivalence between archivers
	%based on weakly Pareto compliant indicators (e.g., $\epsilon$-indicator~\cite{ZitThiLauFon2003:tec})
	%and archivers based on Pareto compliant indicators (e.g., hypervolume~\cite{Zitzler1999}).
	%We prove that archiving algorithms based on weakly Pareto compliant criteria can achieve the same theoretical desirability
	%as those based on Pareto compliant criteria by a slight modification of their archiving criteria.
	Next,
	based upon those theoretical properties,
	we classify well-known archiving algorithms into four classes
	(Sec.~\ref{sec:classification}).
	Afterwards,
	we discuss important issues in multi-objective archiving (Sec.~\ref{sec:issues}),
	suggest several future research lines (Sec.~\ref{sec:future}),
	and provide guidance on choosing and identifying archivers (Sec.~\ref{sec:guide}).
	Lastly,
	we conclude the paper in Sec.~\ref{sec:conclusion}.
	
	%\MANUEL{We may need to distinguish ourselves from recent surveys in archiving, such as \cite{PajBlaHerMar2021archiving,SchHer2021archiving}}
	%\MIQING{Will add and discuss them later.}
	
	\section{Preliminaries}\label{sec:preliminaries}
	
	In multi-objective archiving,
	we are interested in point vectors (i.e., objective vectors)
	in finite and multidimensional objective spaces.
	For simplicity,
	we always use the term ``solutions'' to refer to points in the objective space,
	even though  this term is often reserved for points in the decision space.
	For any finite and multidimensional objective space,
	an order relation
	%which multi-objective optimisation is based upon,
	can be defined as follows
	(w.l.o.g., we consider a minimisation scenario throughout the paper).
	
	%\MANUEL{In the discussion below you talk about objective vectors as solutions, when many people may understand that solutions are points in the decision space and objective vectors are their image in the objective space. We should either use the term ``objective vector'' when talking about the objective space, or make clear in this paragraph that we understand the distinction but for simplicity, when we use the term solution we always mean their corresponding vector of objective values (do we always mean that?)}
	
	% \MANUEL*{Why it has to be finite?}
	\begin{definition}[Pareto dominance relation]
		Let $Y \subset \mathbb{R}^d$ be a finite, $d$-dimensional objective space ($d > 1$).
		For two solutions $y, y' \in Y$,
		$y$ is said to weakly dominate $y'$
		($y\preceq y'$),
		iff $\forall i\in 1,\dots,d$, $y_i \leq y'_i$.
		More strictly,
		$y$ is said to dominate $y'$
		($y\prec y'$),
		iff $y\preceq y'$ and $y \neq y'$.
		In addition,
		we say that two solutions are (mutually) nondominated iff $y\nprec y'$ and $y'\nprec y$.
		%        \MANUEL*{According to this definition, if $y=y'$ then $y\nprec y'$ and $y'\nprec y$ then they are nondominated. Is this what we want? If not, then it should be iff $y\npreceq y'$ and $y'\npreceq y$ }
		%        \MIQING*{Yes, following the convention, in our paper duplicate solutions are included in nondominated ones as well.}
	\end{definition}
	
	The Pareto dominance relation immediately leads to the notion of optimality in multi-objective optimisation. Thus, we can define the set of minimal elements of a given subset $P \subseteq Y$ as~\cite{LopKnoLau2011emo}:
	% \MANUEL*{I changed this a bit because it was too similar to \cite{LopKnoLau2011emo}}\MIQING*{Great, thank you}
	%
	\begin{equation} \label{eq:minimal}
	\min(P,\prec)=\{y\in P \mid \nexists y'\in P, y' \prec y\}.
	\end{equation}
	
	From Eq.~\eqref{eq:minimal}, we can define the Pareto front of $Y$ as the set of minimal elements of $Y$.
	
	\begin{definition}[Pareto-optimal solution, Pareto-optimal set and Pareto front]
		A solution $y^\ast \in Y$ is called Pareto optimal
		iff $\nexists y \in Y$, $y \prec y^\ast$.
		The set of all Pareto optimal solutions of $Y$ is called its Pareto optimal set (or Pareto front),
		i.e., $Y^\ast = \min(Y, \prec) = \{y \in Y \mid \nexists y \in Y, y \prec y^*\}$.
	\end{definition}
	
	% Note that the Pareto front of $Y$ can be seen as the set of minimal elements of $Y$,
	% defined as~\cite{LopKnoLau2011emo}
	% %
	% \begin{equation}
	%   \label{eq:minimal}
	% Y^\ast =\min(Y,\prec)=\{y\in Y, \nexists y'\in Y, y'\prec y\}.
	% \end{equation}
	
	We use the term \emph{nondominated set} for any set $P\subseteq Y$ with the property $P=\min(P,\prec)$.
	
	The order relations between solutions can be readily extended to sets of solutions.
	
	\begin{definition}[Pareto dominance relation between sets]
		For two solution sets $A, B \subset Y$,
		$A$ is said to weakly dominate $B$
		($A\preceq B$),
		iff $\forall b\in B, \exists a\in A$,
		$a \preceq b$.
		More strictly,
		$A$ is said to dominate $B$
		($A\prec B$),
		iff $\forall b\in B, \exists a\in A$,
		$a \prec b$.
	\end{definition}
	
	% \MANUEL*{I don't think we ever use dominance between sets in the paper and I wish Zitzler et al. had not used the term ``better'' for $\Better$. It would be clearer to replace ``better'' with ``dominates'' and never define what we currently call ``dominance'' between sets (we can clarify that our ``dominates'' is called ``better'' by Zitzler et al.). If you agree, I can update the paper. }
	% \MIQING*{Good point! I fully agree that the term ``better'' is not very good. The only concern from me is that term seems to be widely used and also accepted in the literature. Is there any other paper mentioning this and replacing ``better'' by ``dominance'' when defining the relation between two sets?}
	
	It can be seen that the set-based weak-dominance relation cannot rule out the equivalence between two sets,
	while the set-based dominance relation does not allow the equality between any two solutions.
	Therefore,
	\citet{ZitThiLauFon2003:tec} proposed another order relation between sets,
	called \emph{better} relation,
	which represents the most general and weakest form of superiority between two sets.
	\begin{definition}[Better relation between sets~\cite{ZitThiLauFon2003:tec}]
		For two solution sets $A, B \subset Y$,
		$A$ is said to be better than $B$
		($A\Better B$),
		iff $A \preceq B$ and $\exists a \in A$, $\forall b \in B$, $b \npreceq a$.
		\label{def:better}
	\end{definition}
	In other words,
	$A \Better B$ means that
	$A$ is at least as good as $B$,
	but $B$ is not as good as $A$,
	i.e., $A\preceq B$ but $B\npreceq A$.
	%\MANUEL{The choice of ``better'' is quite unfortunate because it leads to ambiguity. I wonder if other authors have used a different term. In some papers, I just call it the $\vartriangle$-relation and instruct the reader to read it as ``better'' or ``Pareto-better''}
	
	Unfortunately,
	the above set order relations are in general not sufficient to distinguish between solution sets.
	Two sets are incomparable as soon as there exist two mutually nondominated solutions from different sets.
	This is particularly the case in many-objective optimisation~\cite{LiLiTanYao2015many}
	as the chance of two solutions being nondominated increases exponentially with the number of objectives~\cite{FarAma2002nafips}.
	
	A total order between sets can be defined by means of a quality indicator $I$ that maps a set to a real number,
	formally,
	$I\colon \mathcal{P}(Y)\setminus\emptyset \to \mathbb{R}$,
	where $\mathcal{P}(Y)$ denotes the power set of~$Y$.\
	Yet,
	mapping a set of solution vectors to a real number inevitably results in information loss.
	We certainly hope that the mapping of a quality indicator is always compliant with the $\Better$-relation, the most general form of superiority between two solution sets.
	
	\begin{definition}[weakly Pareto-compliant indicator~\cite{ZitThiLauFon2003:tec}]
		A quality indicator $I$ is called weakly Pareto compliant
		iff $\forall A, B \subset Y$,
		if $A \Better B$
		then $I(A) \leq I(B)$ (assuming w.l.o.g. that smaller values of $I$ are preferable).
		\label{Def:weak_compliance}
	\end{definition}
	
	An indicator being weakly Pareto compliant implies that \mbox{$I(A) < I(B) \implies B \not\Better A$}, that is, if the quality indicator says that $A$ is better than $B$,
	then $B$ cannot be better than $A$ in terms of Pareto optimality, which is  the weakest requirement of an indicator.
	This property prevents a solution set being evaluated better than another by the quality indicator, yet the former will never be preferred by the decision maker
	according to Pareto-optimality.
	% under any circumstance since it is dominated by the other.
	This unwelcome situation can happen in application scenarios
	if the indicator used does not hold this property~\cite{LiCheYao2022ieeese}.
	Fortunately,
	there are several weakly Pareto-compliant indicators in the literature,
	such as the $\epsilon$-indicator~\cite{ZitThiLauFon2003:tec}, $R2$~\cite{BroWagTrau2012r2} and \IGDplus~\cite{IshMasTanNoj2015igd}
	(see~\cite{LiYao2019qual} for a review of quality indicators).
	However,
	indicators holding this property may still fail to distinguish between two solution sets that satisfy the \Better-relation, that is, it may happen that $I(A) = I(B)$ given $A \Better B$.
	Thus, a more strict property is useful.
	
	\begin{definition}[Pareto-compliant indicator~\cite{ZitThiLauFon2003:tec}]
		A quality indicator $I$ is called Pareto compliant
		iff $\forall A, B \subset Y$,
		if $A \Better B$
		then $I(A) < I(B)$, which implies $I(A) \leq I(B) \implies B \not\Better A$.
	\end{definition}
	
	An indicator being Pareto compliant means that if a solution set is better in terms of Pareto optimality,
	then its quality value must be strictly lower.
	This implies that only the Pareto-optimal set achieves the minimum value of the indicator.
	This property is very strict and very few quality indicators hold it
	(the hypervolume indicator~\cite{ZitThi1998ppsn} is one).%
	\footnote{%
		In the literature,
		Pareto compliance  is sometimes called strong Pareto compliance and weak Pareto compliance is called Pareto compliance~\cite{KnoThiZit06:tutorial, LiYao2019qual}.}
	% Pareto compliance is sometimes called strong Pareto compliance~\cite{KnoThiZit06:tutorial, LiYao2019qual};
	% %\MANUEL{Pareto complete and compatible mean something different. Pareto compliance is a combination of specific compatibility and completeness requisites. See \cite{ZitThiLauFon2003:tec} }
	% accordingly,
	% weak Pareto compliance is also sometimes called Pareto compliance~\cite{KnoThiZit06:tutorial, LiYao2019qual}.}

	%\MANUEL{Note that this implies $I(A) < I(B) \implies B \not\Better A$, which is the same condition as above. Thus, both weak and non-weak require this condition, however, the non-weak also requires that $A \Better B$ $\implies$ $I(A) < I(B)$. Thus, my preferred way to explain this is to say: The weakest  requirement is that we do not choose a set that is worse, thus if the quality indicator says that A is better than B, then B cannot be better than A in terms of Pareto-optimality. The best possible requirement is that, in addition to the previous requirement, if A is better in terms of Pareto optimality, then the quality value must be strictly lower, which implies that only the Pareto-optimal set achieves the minimum value of $I$. Then say that the weakest form relaxes the second requirement.}

	\section{The Archiving Problem}\label{sec:archiving_problem}
	
	\subsection{Brief History}
	
	Study on the archiving problem starts from the convergence analysis of MOEAs in the late 90s~\cite{Hanne1999ejor,RudAga2000cec}.
	Unlike the single-objective case
	where constructing a convergent EA is generally straightforward
	(via the elitist-preserving rule),
	constructing a convergent MOEA
	(i.e., the sequence of the populations produced by the MOEA converges into a subset of the Pareto front~\cite{ZitLauBleu2004tutorial})
	is non-trivial.
	The main difficulty is that, in contrast to the single-objective case where there is a total order relation between solutions, Pareto dominance is a partial order, which leads to solutions (and solution sets) being incomparable.
	% This is because there is not a total order relation between solutions as in the single-objective case,
	% but rather a partial order relation resulting from Pareto dominance,
	% which leads to many solutions incomparable (nondominated).
	An MOEA using a fixed-size population needs to decide which nondominated solutions are removed from the current population to allow the entry of newly generated ones. This decision is taken by the environmental selection method, which can be seen as an archiving method that receives new solutions as input, compares them with the ones in the population (archive), and decides which solutions are kept and which ones are thrown away.
	
	% To allow the entry of newly generated solutions,
	% an MOEA needs to decide which nondominated solutions to be removed from the current population.
	% This is down to the population update (i.e., environmental selection) operation,
	% which can be regarded as an archiving process of taking new solutions,
	% comparing them with old ones,
	% and deciding how to update the population.

	Environmental selection is usually designed on the basis of two principles:
	(1)~dominated solutions should be removed earlier than nondominated ones and
	(2)~solutions in crowded regions should be removed earlier than ones in sparse regions when all of them are mutually nondominated.
	Different ways of implementing these two principles resulted in many successful MOEAs during the period of 1999--2002,
	such as SPEA~\cite{ZitThi99:spea}, PAES~\cite{KnoCor1999cec}, PESA-II~\cite{CorKno2001pesa2}, NSGA-II~\cite{Deb02nsga2} and SPEA2~\cite{ZitLauThie2002spea2}.
	However,
	such a ``\emph{Pareto dominance $+$ density}'' criterion does not guarantee a convergent MOEA.
	Solutions in the population can deteriorate with time
	since the population may accept solutions that are dominated by a solution removed previously,
	provided that these solutions are not dominated within the current population and are located in a less crowded region
	(an illustration will be given in the next subsection).
	
	In the meanwhile,
	researchers attempted to develop MOEAs with guaranteed convergence~\cite{Rudolph1998ep,Hanne1999ejor,RudAga2000cec,Hanne2001emo} by
	dropping the density criterion and removing solutions only if they are  dominated by newcomers. This criterion ensures the the monotonicity of the populations with respect to Pareto dominance, but the final population returned may end up crowding a small region of the Pareto front.
	% They did this through dropping the second principle in population update.
	% Solutions are removed from the population only if they are dominated by newcomers.
	% This ensures the monotonicity of the populations with respect to Pareto dominance,
	% but apparently fails to diversify solutions of the population and
	% the final population returned may concentrate in a small part of the Pareto front.

	To address the above issues, in 2002,
	\citet{LauThiZitDeb2002archiving,LTDZ2002b} proposed the concept of $\epsilon$-approximation in archiving,
	aiming to bridge the gap of MOEAs between theoretical desirability and practical performance.
	The idea is to ensure that every solution in the Pareto front can be represented
	(i.e., $\epsilon$-dominated) by at least a solution in an archive of (polynomially) bounded size.
	However,
	the choice of the parameter $\epsilon$ becomes critical and it may not be practical to set an appropriate $\epsilon$ value for a problem whose Pareto front is unknown,
	while adapting $\epsilon$ on the fly may easily end up with too few solutions in the archive~\cite{KnoCor2004lnems}.

	By 2003,
	Knowles and Corne formalised the archiving problem and separated it from EMO as an independent research topic~\cite{KnoCor2003tec,KnoCor2004lnems}.
	They highlighted the importance of archiving in multi-objective optimisation and
	showed that, from the perspective of the no-free-lunch theorem~\cite{WolMac97:ieee-tec}, the archiving method is a critical component that distinguishes between MOEAs~\cite{CorKno2003cec}.
	%\MANUEL{It is unclear to me what you mean by the ``no-free-lunch-theorem perspective'' in this context}
	They also listed several desirable properties of archivers~\cite{KnoCor2004lnems}
	and then investigated several representative archivers and their convergence properties~\cite{Knowles2002PhD,KnoCor2003tec}.
	Moreover,
	they proved that in general no archiving algorithm is able to maintain an ``optimal approximation'' of the Pareto front (see Def.~\ref{Def:Optimal_approximation} on page~\pageref{Def:Optimal_approximation}) of the sequence solutions at every timestep~\cite{KnoCor2004lnems}.
	%(A definition of optimal approximation will be given in Definition~(\ref{Def:Optimal_approximation}) in next section).
	%\MANUEL{First, this applies to archivers of bounded size or, more generally, an archiver that removes a solution that is not known to be dominated (are we assuming here that the archive is always bounded size?). Second, we have not defined what an optimal representation of the Pareto front is. A definition is given in \cite{LopKnoLau2011emo}. Third, \cite{LopKnoLau2011emo} says that  ``\cite{KnoCor2003tec} previously showed that this online nature of the task means that no archiver can guarantee to have in its archive $\min(N, \abs{Y_t^{*}})$ where $Y_t^{*}$ is the set of minimal elements of the input sequence up to a time $t$.'' This is not exactly what is claimed here (but maybe \cite{KnoCor2004lnems} proved something else?). Fourth, \cite{LopKnoLau2011emo} further states ``A corollary of this, not previously stated explicitly, is that no online archiver of bounded size can deliver an ‘optimal approximation set of bounded size’''.}
	%when the size of the front is beyond the archive's bound~\cite{KnoCor2004lnems}.

	Since then,
	more archiving algorithms with desirable theoretical properties, such as  solution monotonicity (Prop.~\ref{prop:point_monotone} on page~\pageref{prop:point_monotone}) have emerged.
	These archivers  either were based on existing concepts such as the $\epsilon$-dominance~\cite{SchLauCoeDelTal2008,Hanne2007ejor} and hypervolume~\cite{KnoCorFle2003},
	or developed new archiving criteria such as the open rectangle~\cite{JinWon2010adapt}
	% \MANUEL*{Is this one in our table?}\MIQING*{no, this is not a well-known archiver and I think it essentially similar to $\epsilon$-Pareto.}
	and multi-level grid~\cite{LauZen2011ejor}.
	In 2011,
	\citet{LopKnoLau2011emo} systematically analysed representative archiving algorithms and presented several properties desirable for an archiver to hold,
	including $\Better$-monotonicity which is based on the \emph{better} relation defined above (Def.~\ref{def:better}),
	the weakest form of superiority between two solution sets. They also showed empirically that archiving methods used in well-known MOEAs, such as NSGA-II and SPEA2, do not hold this property and, thus, they may produce a population that is worse, in terms of Pareto-optimality, than a previous one.
	%In addition,
	%it is worth mentioning that some studies focused on convergence quality of archiving algorithms with respect to specific indicators
	
	More recent studies have focused on the convergence of archiving algorithms with respect to specific quality indicators
	% rather than Pareto dominance
	% (rather than general properties)
	such as the hypervolume~\cite{BriFri2010ppsn,BriFri2014convergence}, $\epsilon$-box~\cite{HorNeu2008} and Hausdorff metrics~\cite{RudSchGri2016coa,SchHerTal2019archiver,HerSch2022hausdorff}
	(see \cite{SchHer2021archiving} for a review study).
	In addition, several studies have empirically investigated archivers in isolation using artificial sequences of solutions~\cite{LopKnoLau2011emo,MedGolGol2014bracis,Li2021telo}, using solution neighborhoods in combinatorial problems~\cite{LopLieVer2014ppsn} and the practical effectiveness of satisfying theoretical properties~\cite{LiYao2019emo}.

	In contrast to the relatively few theoretical results,
	developing practically effective population update (i.e., archiving) methods, regardless of their theoretical properties,
	%(despite without theoretical desirables, e.g., monotonicity and optimal approximation)
	has become the most active direction in EMO research,
	resulting in numerous MOEAs.
	These methods can mainly be categorised into three mainstream selection paradigms~\cite{WagBeuNau2007:many,EmmDeuw2018tutorial}:
	the Pareto-based (Pareto dominance $+$ density)
	%\MANUEL*{the key is the density mechanism, so the Pareto-dominance, both indicator-based and decomposition-based also rely on Pareto-dominance either directly or indirectly}
	%\MIQING*{Yes, I agree. Usually the Pareto-based is are thought to be the paradigm that considers the Pareto criterion first and then the crowding degree of solutions.}
	\cite{Goldberg89}, the indicator-based~\cite{ZitKun2004ppsn} and the decomposition-based~\cite{ZhaLi07:moead} paradigms.
	Some researchers have also introduced an external archive to guide the evolution of the population~\cite{TiwKochFad2008amga,BriPoz2012cec,LwiQuZhe2013moss,CaiLiFan2015archive,LiYao2020ec,WanMeiZha2021mogp,Zhong2021}.
	To complement a non-Pareto selection criterion (e.g., the decomposition-based criterion) in the evolutionary population,
	the external archive is mainly based on the pattern of ``Pareto dominance $+$ density'', with the exception of some work using an indicator (e.g., hypervolume) as the criterion in archiving (see \cite{BezLopStu2019gecco}).
	In those Pareto-based archivers,
	the density estimators more frequently used are crowding distance~\cite{TiwKochFad2008amga,MouPetMcC2014mopso,CaiLiFan2015archive}, niching~\cite{TiwFadDeb2011amga2,LiYao2020ec} and grid techniques~\cite{KnoCor00paes,KnoCor2000mpaes,LuoBos2012elitist}.
	%\MANUEL*{For grid techniques, shouldn't we mention PAES \cite{KnoCor00paes} \cite{KnoCor2000mpaes} and PESA-II\cite{CorKno2001pesa2}?}
	%\MIQING*{will sort it out later.}
	%\MANUEL*{Are these papers that we cite here good papers with original ideas: \cite{Patil2020} \cite{Zhong2021}? These two seem to me not very good quality and we do not need to cite every paper ever published if we don't think they are good ones. }
	%\MIQING*{I tried to mainly consider the MOEAs which use an external archive, apart from the main evolutionary population, but will check it again.}
	%\MANUEL*{Should we mention the studies that add external archivers based on hypervolume and MGA? See \cite{BezLopStu2019gecco}.}
	%\MIQING*{good point, will add it later.}
	%
	Experimental work has shown that adding an external archive is often beneficial~\cite{TanIshOya2017,SchHer2021archiving}, specially if the parameters of the MOEAs are configured after adding the archive~\cite{BezLopStu2019gecco}.
	Recently,
	there is a trend in MOEA design that considers two archives, that is, two simultaneous populations that participate in solution generation, each of them updated by a different selection criterion~\cite{PraYao2006taa,UlrBadThi2010ppsn,ChenChenGon2014tcyb,LiYanLiu2016tec};
	%\MANUEL*{We do we call two-archives approaches? A population $+$ external archive or population $+$ 2 archives or two populations or we don't distinguish between these alternatives?}
	%\MIQING*{good point! essentially, they are the two-population-based approach, i.e., both populations participate in the evolution, with each considering a different selection criterion. Unfortunately, they are all called two-archive methods...}
	for instance, one archive for promoting convergence and the other for promoting diversity~\cite{PraYao2006taa,UlrBadThi2010ppsn}.
	Such a two-archive approach is particularly suitable for multi-objective problems with additional features, e.g., with many objectives),
	as one archive can be designed specifically for dealing with those features.
	This is why the two-archive approach has now been used in various challenging multi-objective scenarios,
	such as many-objective optimisation~\cite{WanJiaYao2015twoarch2,LiLiTanYao2014taa,DinDonHeLi2019twoarch,MMalDas2022twiarch},
	%\MANUEL*{Have you read \cite{Wu2022}? It seems to just take TwoArch2 and modify it to use $\epsilon$-dominance. Is this something notable enough to cite?}
	%\MIQING*{thanks for pointing out this. will reconsider it.}
	constrained optimisation~\cite{LiCheFuYao2018twoarch,Xia2021,Li2022constrained},
	dynamic optimisation~\cite{CheLiYao2017dynamic},
	multi-model optimisation~\cite{LiuYenGon2018twoarch,LiZouYan2021twoarch},
	expensive optimisation~\cite{SonWanHeJin2021kriging},
	and real-world problems~\cite{GarLopGod2016pso,Chhabra2018,Zhang2019}.
	Finally, some archivers~\cite{SchVasCoe2011space,SchHerTal2019archiver} aim to preserve not only Pareto-optimal solutions but also \emph{nearly-optimal} ones, also called non-epsilon dominated or $\epsilon$-efficient, that is, solutions that are not too far from being Pareto-optimal. There are also archivers that maximise diversity in the decision space, typically for multimodal problems~\cite{UlrBadThi2010ppsn,DebTiw2008omni,HerSchSun2020noneps} and a number of recent archivers (see \cite{PajBlaHerMar2021archiving} for a survey) consider both near-optimality and diversity in decision space. We will not study these types of archivers in this paper since their aims are rather different from most other archivers that focus on the objective space.

	\subsection{Formal Definition}\label{sec:archiving}
	
	The archiving process can be described as updating a set of solutions $A$, an \emph{archive}, by an \emph{input sequence}
	$\mathcal{S}=\langle S^{(1)}$, $S^{(2)},\dots$, $S^{(t)},\dots\rangle$,
	%\MANUEL{Do we really need a symbol for the sequence? If not, I'd rather write: an \emph{input sequence} of solution sets $S^{(1)}, S^{(2)},\dots,S^{(t)},\dots$}
	which may be generated by a solution generator (e.g., an evolutionary algorithm) iteratively.
	At iteration $t$,
	the generator may generate one or multiple solutions, i.e., $\forall t, \abs{S^{(t)}} \geq 1$,
	possibly using the contents of the old archive $A^{(t-1)}$,
	where $A^{(t-1)}$ denotes the archive \emph{after} updating it with $S^{(t-1)}$,
	and $A^{(0)}$ is the empty set.
	%, w.l.o.g.\MANUEL*{Why do we say w.l.o.g. here?} \MIQING*{you are right, yes, no need.}
	%\MANUEL{I feel we need to be more precise. We could say that the archive $A$ is a set of solutions, and that $A^{(t)}$ denotes the archive \emph{after} updating it with $S^{(t)}$. $A^{(0)}$ is the empty set, w.l.o.g.}
	%Thus,
	% \MANUEL{If we use the notation $(\mu + \lambda)$ for archivers, then we have to use $\mu$ instead of $N$ everywhere for the capacity of the archive.}
	% \MIQING{may be more convenient to use $N$.}
	Solutions may be fed to the archive one-at-a-time %$(\mu + 1)$ form
	as in $\epsilon$-MOEA~\cite{DebMohMis2005epsilon} and SMS-EMOA~\cite{BeuNauEmm2007ejor}
	%\MANUEL{\cite{Emmerich2005} is not as complete as \cite{BeuNauEmm2007ejor} }
	or many-at-a-time % $(\mu+\lambda)$ form
	%\MANUEL{Rather than one-by-one and bunch-by-bunch or chunk-by-chunk, I would suggest one-at-a-time and many-at-a-time or perhaps $(\mu + 1)$ and $(\mu+\lambda)$}
	as in NSGA-II~\cite{Deb02nsga2} and SPEA2~\cite{ZitLauThie2002spea2}.
	There is no requirement that the elements in the sequence are unique.
	
	We are interested here in archives of bounded capacity,
	i.e., $\forall t, \abs{A^{(t)}} \leq N$ for some constant $N \in \mathbb{N}^+$,  smaller than the number of Pareto-optimal solutions in the input sequence, $N < \abs{Y^{*}}$. Thus, the archive $A^{(t)}$ stores a subset of the solutions in the input sequence up to time $t$.
	%
	%There may be a bound on the capacity of the archive,
	%Sometimes, for convenience,
	%we drop the $t$ superscript and just refer to $A$;
	%for example $\abs{A} \leq N$ is shorthand for $\forall t, \abs{A^{(t)}} \leq N$.
	Now we can define an archiving algorithm as follows.
	
	\begin{definition}[Archiving algorithm or Archiver]
		% An archiving algorithm can be seen as a state transition and defined as
		An archiving algorithm takes as input the previous archive $A^{(t-1)}$ and the current set in the sequence $S^{(t)}$ and returns the updated archive $A^{(t)}$, i.e., $A^{(t)} \assign \textup{Archiver}(A^{(t-1)}, S^{(t)})$, where $A^{(t)} = \{ A \subseteq A^{(t-1)} \cup S^{(t)} \mid  1 \leq \abs{A} \leq N\}$ and $S^{(t)} \subset Y$,
		where $Y \subset \mathbb{R}^d$ is a finite, $d$-dimensional objective space from which the solutions are generated.
	\end{definition}
	%
	% \begin{definition}[Archiving algorithm]
	%	An archiving algorithm is defined as a mapping: $\mathcal{A}_N \times \mathcal{P}(Y)/\{\emptyset\} \rightarrow \mathcal{A}_N$,
	%	where $Y \subset \mathbb{R}^d$ is a finite, $d$-dimensional objective space from which the solutions are generated,
	%	$\mathcal{P}(Y)$ is the power set of $Y$,
	%	and	$\mathcal{A}_N$ denotes the set of all the subsets of $Y$
	%	whose cardinality is less than or equal to $N$: $\mathcal{A}_N = \{A\in \mathcal{P}(Y) \mid 1 \leq \abs{A} \leq N\}$.
	%\end{definition}
	%
	The condition \mbox{$A^{(t)} \subseteq A^{(t-1)} \cup S^{(t)}$} implies that the archiving algorithm is not allowed to revisit nor store previous solutions in the input sequence beyond those present in $A^{(t-1)}$ or duplicated in $S^{(t)}$.
	% \MANUEL*{``solely based'' is not strictly true, since the archiving algorithm may have other inputs such as reference points, bounds in maximum or minimum values, various parameters, reference vectors, so there is state inside beyond the inputs.}
	% \MIQING*{good point!}
	%The archiving algorithm is not allowed to revisit the previous solutions in the input sequence;
	%that is, the update of the archive at time $t$ is solely based on $A^{(t-1)}$ and $S^{(t)}$.
	%\MANUEL{I wonder if we can change the definition to make this clear. I feel that the above definition does not express relationship between archives. An archiving algorithm is more like a state transition, where $A^{(t)} \assign \text{Archiver}(A^{(t-1)}, S^{(t)})$, where $A^{(t)} = \{ A \subseteq A^{(t-1)} \cup S^{(t)} \mid  1 \leq \abs{A} \leq N\}$ and $S^{(t)} \subset Y$.}
	
	%\begin{algorithm}[h]
	%	\caption{Generic Archiving Procedure}
	%	\label{A1}
	%	\SetAlgoLined
	%	\KwIn{$A^{(t)}$, $P^{(t+1)}$}
	%	$A^{(t+1)} \leftarrow update(A^{(t)}, P^{(t+1)})$\\
	%	\KwOut{$A^{(t+1)}$}
	%\end{algorithm}

	In EMO,
	there are generally three ways to use the archive.
	Firstly,
	the archive may be used solely to store high-quality solutions found by a search algorithm but it does not influence the generation of solutions (i.e., the sequence),
	such as in PAES~\cite{KnoCor00paes}, MOEA/D~\cite{ZhaLi07:moead} and others~\cite{Veldhuizen1999phd,FieEveSing2003tec,BezLopStu2019gecco}. This is often called an \emph{external} archive.
	%\MANUEL{This gives the impression that this option is rare, but in fact it is very very common and advocated for: \cite{KnoCor00paes,FieEveSing2003tec,BezLopStu2019gecco}}
	
	In the second way,
	the archive not only stores high-quality solutions but also participates in some way in the generation of new solutions;
	for example,
	in the crossover operation,
	one parent solution is from the population and the other solution is from the archive,
	such as in SPEA~\cite{ZitThi99:spea}, $\epsilon$-MOEA~\cite{DebMohMis2005epsilon} and others~\cite{WanMeiZha2021mogp}.
	
	The third way is what the vast majority of MOEAs follow,
	in which the archive is essentially the evolutionary population of the MOEA and new solutions are generated solely from it, that is:
	$S^{(t)} \assign \textup{Generator}(A^{(t-1)})$; $A^{(t)} \assign \textup{Archiver}(A^{(t-1)}, S^{(t)})$.
	% \begin{equation*}
	% \begin{split}
	%   S^{(t)} &\assign \textup{Generator}(A^{(t-1)})\\
	%   A^{(t)} &\assign \textup{Archiver}(A^{(t-1)}, S^{(t)})
	% \end{split}
	% \end{equation*}
	%
	Different MOEAs use different terms for this type of archive, e.g., NSGA-II~\cite{Deb02nsga2} calls it ``population'' whereas SPEA2~\cite{ZitLauThie2002spea2} calls it ``archive''. However, not all populations are archives, e.g., the (offspring) populations in NSGA-II and SPEA2, which temporarily store newly generated solutions, are elements of the input sequence.
	% Actually,
	% different MOEAs may have different names for it.
	% For example,
	% in NSGA-II~\cite{Deb02nsga2} there is a population and a temporary population (for storing newly generated solutions),
	% which correspond to the archive and the population in SPEA2~\cite{ZitLauThie2002spea2}, respectively.
	%\MANUEL{but the temporary population (NSGA-II) and population (SPEA2) are not archives in the sense defined above since they are generated only using the previous population/archive, they are more like the $Set$s that you defined above.}
	% But,
	% whatever named,
	% they all follow the pattern of ``$S^{(t)} \leftarrow \text{Generator} (A^{(t-1)})$ and $A^{(t)} \leftarrow \text{Archiver} (A^{(t-1)},S^{(t)})$''.
	%\MANUEL{Note that this expression matches closely what I proposed before for the definition of Archiver. }
	
	%\MANUEL{Could you add the source files for the figures?}
	%%%% Fig. 1 %%%%
	\begin{figure}[tbp]
		\centering%
		\footnotesize%
		\includegraphics[width=\linewidth]{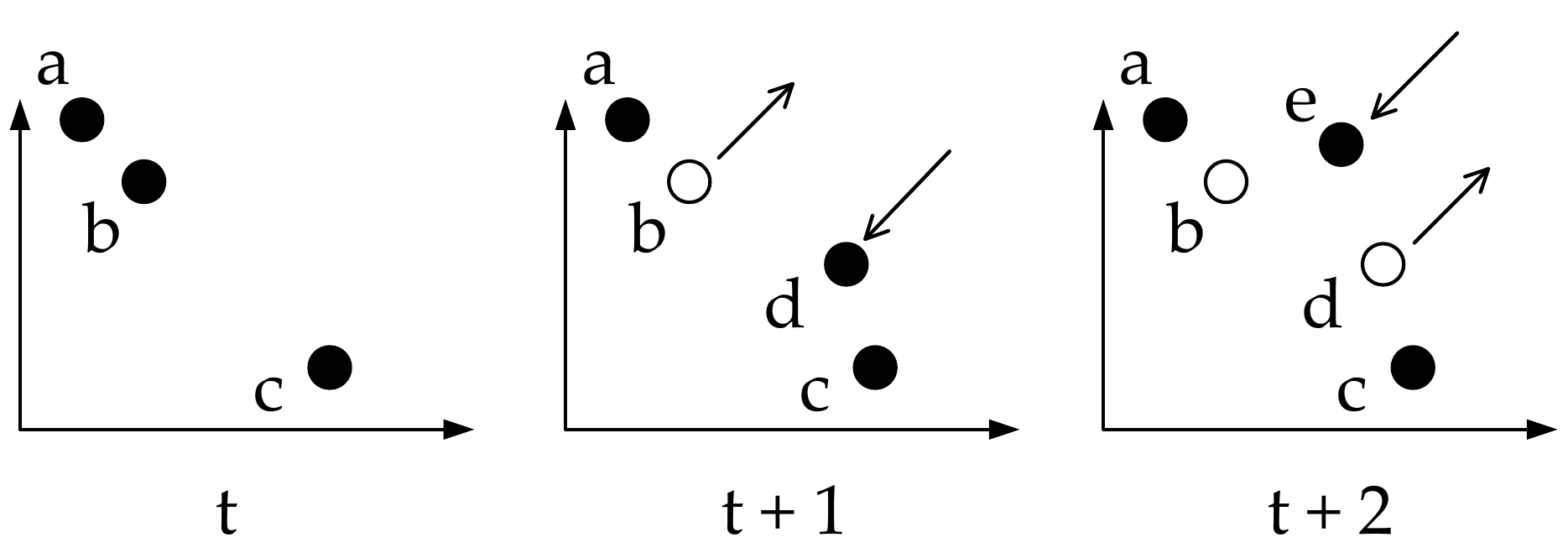}
		\caption{Illustration of an archiver based on the archiving rule ``Pareto dominance $+$ density'' deteriorating. The capacity of the archive is $3$.
			Black circles denote solutions in the archive and hollow circles denote solutions removed.
			At the timestep $t+1$, solution $d$ enters the archive and $b$ is removed
			since $d$ has less crowding degree than $b$.
			At the timestep $t+2$,
			solution $e$ enters the archive and edges out $d$ since $e$ has less crowding degree than $d$.
			Now the archive consists of $\{a, e, c\}$, which is worse than the archive of $\{a, b, c\}$ at timestep $t$.}
		\label{fig:Pareto_deterioration}
	\end{figure}
	
	% \MANUEL*{The hypervolume archiver in \cite{LopKnoLau2011emo} does not deteriorate (the hypervolume cannot decrease). I changed the example to SMS-EMOA, which does deteriorate with dynamic reference point.}
	% \MIQING*{excellent! thanks for the change.}
	% \manuel{
	% An issue in MOEAs is that their archive/population (or some of its solutions) may deteriorate.
	% This has been reported at the very beginning of the study on archiving~\cite{Everson2002,FieEveSing2003tec,ZitLauBleu2004tutorial,KnoCor2004lnems}.
	% Figure~\ref{fig:Pareto_deterioration} gives an illustration that the archive
	% based on the Pareto-based selection criteria (i.e., Pareto dominance $+$ density) deteriorates (a formal definition of deterioration will be given in next section).
	% %\MANUEL{We use the term ``deterioration'' sparingly before we ever properly define it. Also, we fail to distinguish between point-deterioration and set-deterioration (the latter implies the former but not vice versa).}
	% As shown,
	% the archive at the timestep $t$ is \emph{better} (see Definition~\ref{def:better}) than the archive after two timesteps later.
	% In fact,
	% this not only happens to the Pareto-based criteria,
	% but also applies to the indicator-based (e.g., hypervolume-based~\cite{LopKnoLau2011emo}) and decomposition-based (e.g., in NSGA-III~\cite{Fie2017gecco}) criteria.
	% %\MANUEL{which decomposition-based? We later give one example of decomposition-based that does not deteriorate.}
	% }{
	An undesirable property of most MOEAs is that their archive/population may deteriorate in quality, i.e., the final archive returned by the MOEA may be worse than an archive in a previous step. This issue has been reported very early in the study of archiving~\cite{EveFieSin2002full,FieEveSing2003tec,ZitLauBleu2004tutorial,KnoCor2004lnems}. Figure~\ref{fig:Pareto_deterioration} illustrates how an archiver based on the ``\emph{Pareto dominance $+$ density (crowding distance)}'' criterion may result in an archive at timestep $t$ that is \emph{better} (see Def.~\ref{def:better}) than the archive two timesteps later. Not only density-based archivers suffer from deterioration, but also some indicator-based (e.g., in SMS-EMOA~\cite{BeuNauEmm2007ejor} as shown later in Fig.~\ref{fig:SMS}) and decomposition-based archivers (e.g., in NSGA-III~\cite{DebJain2014:nsga3-part1} as shown by~\cite{Fie2017gecco}).

	\section{Theory}\label{sec:theory}
	
	When designing an archiver,
	one may wish it to hold some theoretical properties.
	For example,
	one may wish its archive to consist of only Pareto-optimal solutions with respect to the input sequence;
	not to deteriorate (i.e., the current archive cannot be worse than at a previous timestep);
	%to accept solutions which not being dominated by some eliminated previously;
	to be able to converge with sufficient timesteps;
	and to contain as many nondominated solutions as possible within its capacity.

	Unfortunately,
	archivers in most well-known MOEAs do not hold such theoretical desirables.
	For example,
	a large portion of their final population are not Pareto optimal with respect to the solutions generated;
	e.g., as reported in~\cite{LiYao2019emo},
	the majority of solutions in the final population of NSGA-II, SPEA2 and MOEA/D
	are dominated by other solutions generated during the search on some problems, such as FON~\cite{FonFle1995ec}.
	This not only fails to return the decision maker the best solutions found,
	which could be remedied by using an external unbounded archive,
	but may also affect the search progress since the evolutionary population cannot represent the best solutions discovered.
	
	In most MOEAs, the population/archive has two roles: (1) storing the best solutions found for the decision-maker and (2) maintaining a solution set as the source for generating new solutions during the search. Here, we focus on the properties of archivers that are desirable for the first role, but not necessarily desirable for the second one.
	% archivers to have the following properties when being used as a memory to archive solutions,
	% while it may not be necessary when being used as a pool to generate new solutions during the search.

	\subsection{Properties}
	
	%\MANUEL{%
	%This paragraph repeats part of what we explained before. What about the following?
	%In accordance with existing studies~\cite{KnoCor2004lnems,CorKno2003cec,LauThiZitDeb2002archiving,LopKnoLau2011emo},
	%we assume the input sequence is a sequence of individual solutions presented to the archiving algorithm one at a time, that is,
	%$\forall t, \abs{S^{(t)}} = 1$ (Section~\ref{sec:archiving}). We can always convert a many-at-a-time sequence into a one-at-a-time sequence, thus properties of the latter hold for the former. The opposite is not true, however, as we will discuss later. To make explicit when we refer to a one-at-a-time sequence, we denote its elements by $s^{(t)}$ and the sequence up to the time $t$ by $\mathcal{S}^{(t)} = \langle s^{(1)}$, $s^{(2)},$ $\dots,$ $s^{(t)}\rangle$.}
	
	In general,
	there are two types of properties that an archiving algorithm may have:
	anytime properties and limit properties.
	% \MANUEL{Is this the same as ``properties'' or ``desirables''? I would prefer to not use too many different terms to refer to the same thing.}
	Anytime properties must hold at any timestep $t$, whereas limit properties must hold after a finite number of timesteps under the assumption that any solution
	%\MANUEL{actually, it does not have to be any arbitrary solution, only any Pareto-optimal solution} 
	may appear an infinite number of times in the input sequence~\cite{RudAga2000cec}.
	% When considering the anytime properties,
	% we wish to know if the archive always respects some desirables at every timestep.
	% When considering the limit properties,
	% %(i.e., behaviour on solutions drawn indefinitely from a finite space),
	% we wish to know whether some kind of convergence ever occurs when points in the sequence are seen an infinite number of times~\cite{RudAga2000cec}.
	%\MANUEL{We need to say that the limit behavior assumes that points in the sequence are seen an infinite number of times~\cite{RudAga2000cec}}
	In the following,
	we introduce six properties from the literature~\cite{LopKnoLau2011emo,Knowles2002PhD,KnoCor2004lnems}  but formulated in a manner that is applicable to both bounded-size archivers, which only store nondominated solutions, and fixed-size archivers, which may store dominated solutions.
	The first three properties are anytime properties,
	and the last three ones are limit properties.
	Any anytime property that holds for one pass over a finite sequence should also hold in the limit, that is, unlimited passes over a finite sequence or an infinite sequence drawn from a finite set $Y$. We denote the union of all solutions seen up to time $t$ by $Y^{(t)} = \bigcup_{i=1}^t S^{(i)}$.
	%\MANUEL{If our definition of archive does not require that all points are dominated, then we need to be careful when we refer to all points in the archive or only to the nondominated subset. For example, an archiver that keeps dominated points will never be Pareto-subset according to our definition, however, the nondominated subset of the archive could be Pareto-subset. Thus, I believe our definition should only require this property of such nondominated set $\min(A^{(t)}, \prec)$.}

	\begin{property}[Pareto-subset~\cite{Knowles2002PhD}]
		An archiver has the \emph{Pareto-subset} property
		if no nondominated solution in its archive at any timestep is dominated by a solution in the input sequence seen so far:
		$\forall t \in \mathbb{N}^{+}$, $\forall a\in \min(A^{(t)}, \prec)$, $\nexists s\in Y^{(t)}$, $s\prec a$.
		\label{prop:Pareto_subset}
	\end{property}
	%\MIQING{I think it would be good if the definition of the archive does not require that all points are nondominated. Manuel, do you think if it is better off if we allow dominated solutions in archives for those properties?}
	
	\begin{property}[point-monotone~\cite{LopKnoLau2011emo}]
		An archiver has the \emph{point-monotone} property
		if $\forall t, \forall i\in\mathbb{N}^{+}$,
		there does not exist a pair of solutions $a \in \min(A^{(t)}, \prec)$ and $a' \in \min(A^{(t+i)}, \prec)$,
		such that $a \prec a'$.
		An archiver that does not have this property is said to \emph{point-deteriorate}.
		\label{prop:point_monotone}
		%\MANUEL{Perhaps we should define $\min(A^{(t+1)}, \prec)$ earlier so that we can say now $a' \in \min(A^{(t+1)}, \prec)$. The assumption that the archiver may contain dominated points will require us to be precise on when we refer to all points or only the nondominated ones.}
		%\MIQING{Good point! I will add the definition of $\min(A^{(t)},\prec)$ in the Section II. Maybe add the $\min(A^{(t)},\preceq)$ as well since our archiving algorithm (i.e., the one based on a weakly Pareto compliant indicator) needs unique nondominated solutions in the archive. Will double check which one ($\min(A^{(t)},\prec)$ or $\min(A^{(t)},\preceq)$) is more accurate for properties, lemmas, theorems throughout the paper, and also check if it is better to use $\min(A^{(t)},\preceq)$ in Algorithm 1.}
		%\MIQING{In addition, we may want to check if $\min(A^{(t)}, \prec)$ allows duplicate nondominated solutions? I suspect our theorem (i.e., an archiver based on a weakly Pareto compliant indicator can have the \emph{limit-optimal} property) would require unique nondominated solutions.}
	\end{property}

	\begin{property}[set-monotone~\cite{LopKnoLau2011emo}]
		An archiver has the \emph{set-monotone} property
		if $\forall t, \forall i\in\mathbb{N}^{+}$,
		there does not exist a pair of archives $A^{(t)}, A^{(t+i)}$,
		such that $A^{(t)} \Better A^{(t+i)}$, i.e., $A^{(t)}$ is \emph{better} in terms of Pareto optimality than $A^{(t+i)}$.
		An archiver that does not have this property is said to \emph{set-deteriorate}.
		\label{prop:set_monotone}
	\end{property}

	Amongst the three properties,
	Property~\ref{prop:Pareto_subset}
	%
	%which requires that any solution that is nondominated within the archive at this timestep is also nondominated within the sequence seen so far,
	is the strictest one to hold for an archiver and implies Property~\ref{prop:point_monotone} which implies Property~\ref{prop:set_monotone}.
	Property~\ref{prop:point_monotone}
	%, which is implied by Property~\ref{prop:Pareto_subset},
	%\MANUEL{can? It always does, no? In fact, Property~\ref{prop:Pareto_subset} implies Property~\ref{prop:point_monotone} which implies Property~\ref{prop:set_monotone}. The implication does not hold in the other direction and an archiver may never deteriorate yet contain points that are nondominated in the archive but dominated by a point in the input sequence~\cite{LopKnoLau2011emo}. It would be useful perhaps to have a table or Venn diagram regarding implications.}
	requires the archiver to have a rule
	prohibiting solutions from entering the archive if they are dominated by %(or equivalent to)
	solutions eliminated previously.
	%\MANUEL{I don't understand this comment. How can a solution enter the archive repetitively in the one-pass setting.}
	%\MANUEL{It seems to apply to ``efficiency preserving'' which is a different property altogether.}
	Such a rule can be implemented by setting a box (defined by $\epsilon$) for every solution in the current archive and
	rejecting any oncoming solution in those boxes~\cite{LauThiZitDeb2002archiving,LTDZ2002b}.

	Property~\ref{prop:set_monotone} is more attainable; for example, if the sequence of archives never decreases the value of a Pareto compliant indicator, %\MANUEL{Does it really require Pareto compliance or is sufficient with weak Pareto compliance?} as the rule for accepting new solutions may hold this property,
	%\MANUEL{I have the feeling that we should explain how an archiver can use an indicator earlier.}
	such as the multi-level grid archiver~\cite{LauZen2011ejor}, the hypervolume-based archiver proposed by \citet{Knowles2002PhD} or the hypervolume-based environmental selection of SMS-EMOA~\cite{BeuNauEmm2007ejor} (on the condition
	that the reference point used in the calculation of hypervolume does not change).
	% We will give a classification and detailed explanations of representative archivers
	% based on their theoretical properties in the next section.
	%\MANUEL{Another implication: set-monotone does not imply limit-optimal, as it is the case for the adaptive $\epsilon$-approx archiver~\cite{LTDZ2002b}.}
	
	Now, we introduce three limit properties:
	% we consider three properties with respect to limit properties of an archiver,
	% where solutions of the sequence are drawn indefinitely with a positive probability from a finite set $Y$.
	
	\begin{property}[limit-stable~\cite{LopKnoLau2011emo}]\label{def:limit-stable}
		An archiver has the \emph{limit-stable} property
		if for any sequence
		there exists a timestep $t \in \mathbb{N}^{+}$ such that $\forall i\in\mathbb{N}^{+}$, $A^{(t)} = A^{(t+i)}$.
		That is,
		the archive converges to a stable solution set in finite time.
	\end{property}
	
	Being \textit{limit-stable} has some practical benefits, e.g., used as a stop condition during the search. Moreover, \emph{set-monotone} implies \emph{limit-stable}, thus an archiver that is not \emph{limit-stable} cannot be \emph{set-monotone}.
	But it may be of more interest
	if all solutions of the converged archive are Pareto-optimal solutions of the sequence:
	%So we have the following property.
	
	\begin{property}[limit-Pareto-subset]\label{def:limit-pareto}
		An archiver has the \emph{limit-Pareto-subset} property
		if for any sequence
		there exists a timestep $t \in \mathbb{N}^+$ such that $\forall i\in\mathbb{N}^{+}$, $A^{(t)} = A^{(t+i)}$ and
		$\min(A^{(t)},\prec) \subseteq Y^*$, where $Y^* = \min(Y,\prec)$.
		%\MANUEL{we said above that Pareto subset is a property of the archiver not the archive.}
		%    \MANUEL{This seems to require that all solutions in $A^{(t)}$ are nondominated, however, it would be enough if $\min(A^{(t)}, \prec)$ are.}
		%    \MIQING{I have changed to the case that does not require all solutions in $A^{(t)}$ being nondominated. Is the change valid?}
		That is,
		the archive converges to a subset of the Pareto optimal set of $Y$.
	\end{property}
	
	Converging to a stable Pareto subset is desirable,
	but one may also care about the number of Pareto optimal solutions in the converged archive.
	The decision-maker may not be very happy if an archive of capacity $N=100$ ends up with only one Pareto-optimal solution after being feed with hundreds of them.
	This intuition leads to the following definition and property.
	
	\begin{definition}[optimal approximation of bounded size~\cite{LopKnoLau2011emo}]
		Let a solution set $A \subseteq Y$, $1 \leq \abs{A}\leq N$, be a nondominated set,
		i.e., $A=\min(A,\prec)$.
		%(i.e., $\forall a \in A, \nexists a' \in A, a' \prec a$).
		If $\nexists B \subseteq Y$, $\abs{B}\leq N$, such that $B \Better A$,
		then $A$ is called an optimal approximation with bounded size $N$ of $Y^*$,
		where $Y^* = \min(Y,\prec)$.
		%\MANUEL{We could say $\min(Y,\prec)$ if we define this concept above}
		%\MIQING{does an optimal approximation allow containing duplicate solutions? If so, is it possible we allow it to contain dominated solutions?}
		\label{Def:Optimal_approximation}
	\end{definition}
	
	An optimal approximation of bounded size $N$ of %\mbox{$\min(Y^{(t)}, \prec)$}, i.e.,
	the nondominated solutions in the sequence seen so far is the best possible archive that a bounded archiver can produce with respect to Pareto optimality.
	%This definition is the possible optimal property that a bounded archiving algorithm can have
	%with respect to the most general form of superiority between sets.
	Unfortunately, as proved by Knowles and Corne~\cite{KnoCor2004lnems},
	no archiver can guarantee to store at least $N$ nondominated solutions (or the number of nondominated solutions in the sequence seen so far, if the latter is smaller than $N$). As a consequence, no archiver can guarantee to store an optimal approximation of bounded size $N$ at every timestep and for any finite sequence~\cite{LopKnoLau2011emo}.
	% \MANUEL{What \cite{KnoCor2004lnems} proves is that $\abs{\min(A^{(t)}, \prec)} = \min \{N, \abs{Y^{*(t)}}\}$, where $Y^{*(t)} = \min(\bigcup_{i=1}^t S^{(i)}, \prec)$, it is a statement about the number of nondominated solutions. It doesn't even require that the solutions in the archive are Pareto-optimal just that they are mutually nondominated. Do we want to have the formula in the text?}
	% \MIQING{Yes, you are right. I should have been more rigorous. Thanks for pointing out this! I guess we may not need to have the formula.}
	Having said that,
	its limit form may be achievable.
	
	\begin{algorithm*}[tbp]
		\caption{Archiving algorithm based on a weakly Pareto-compliant indicator}
		\label{Alg:Weak_compliance}
		\DontPrintSemicolon
		\SetNoFillComment
		\SetAlgoLined
		\KwIn{$A^{(t-1)}$, $s^{(t)}$}
		\uIf(\tcp*[f]{Rule 1: if the new solution is weakly dominated by a solution in the archive.}){$\exists a \in A^{(t-1)}$, $a \preceq s^{(t)}$}{
			$A^{(t)} \assign A^{(t-1)}$
		}
		\uElseIf(\tcp*[f]{Rule 2: if the number of all nondominated solutions is less than or}){$\abs{\min(A^{(t-1)} \cup \{s^{(t)}\},\prec)} \leq N$}{
			$A^{(t)} \leftarrow \min(A^{(t-1)} \cup \{s^{(t)}\},\prec)$
			\tcp*[f]{\hspace{3.25em}equal to the archive capacity after adding the new solution.}
		}
		\uElseIf(\tcp*[f]{Rule 3: if the new solution cannot lead to a better}){$ I(A^{(t-1)}) \leq \min_{a\in A^{(t-1)}}\{I(A^{(t-1)} \cup \{s^{(t)}\} \setminus \{a\})\} $}{
			$A^{(t)} \leftarrow A^{(t-1)}$\tcp*[f]{$\phantom{Rule 3: }$   indicator value.\hspace{11em}}
		}
		\Else(\tcp*[f]{Rule 4: the new solution can lead to a better indicator value.}){
			$a' \leftarrow \argmin_{a\in A^{(t-1)}}\{I(A^{(t-1)} \cup \{s^{(t)}\} \setminus \{a\})\}$\;
			$A^{(t)} \leftarrow  A^{(t-1)} \cup \{s^{(t)}\} \setminus \{a'\}$
		}
		\KwOut{$A^{(t)}$}
	\end{algorithm*}

	\begin{property}[limit-optimal~\cite{LopKnoLau2011emo}]\label{def:limit-optimal}
		An archiver has the \emph{limit-optimal} property
		if for any sequence
		% for any given archive's capacity $N$
		there exists a timestep $t \in \mathbb{N}^+$ such that $\forall i\in\mathbb{N}^{+}$, $A^{(t)} = A^{(t+i)}$ 
		and
		$\min(A^{(t)}, \prec)$ is an optimal approximation of size $N$ of $Y^*$,
		where $Y^* = \min(Y, \prec)$ and $N$ is the capacity of the archive.
	\end{property}

	An archiver that optimizes a Pareto-compliant indicator when updating the
	archive will be \emph{limit-optimal}~\cite{LopKnoLau2011emo}.
	%\MANUEL*{Optimizes is
	%  ambiguous here. An archiver that never worsens the value of a
	%  Pareto-compliant indicator will be set-monotone, but not limit-optimal. An archiver that always improves the value as much as possible will be limit-optimal AND set-monotone. Interestingly, \cite{BriFri2014convergence} prove that ``there are nondecreasing archiving algorithms which are better and faster than all increasing archiving algorithms''}
	%\MANUEL{I think it is not just ``may''. It is easy to prove that the archiver must be limit-optimal if the archive never worsens in indicator value. I have the impression this is even true for weak indicators (but some of those such as the unary $\epsilon$ require knowing the Pareto-front, so they are not useful in practice). It should also be true for binary indicators, such as the binary-$\epsilon$ indicator used in IBEA if I remember correctly.}
	Holding this property can bring practical benefits, as shown in~\cite{LiYao2019emo},
	where a hypervolume-based archiver significantly outperforms other non-convergence-guaranteed archivers.
	In the next subsection,
	we prove that an archiver optimising a weakly Pareto-compliant indicator
	as the archiving criterion may also hold this property.
	That means that some popular indicators in the EMO area,
	such as the $\epsilon$-indicator~\cite{ZitThiLauFon2003:tec} and \IGDplus~\cite{IshMasTanNoj2015igd},
	can also lead to convergence-guaranteed archivers. % as the hypervolume indicator.
	%\MANUEL{But those two need a reference front. And the wrong reference front will lead to the wrong result, no?}
	%\MIQING{I guess the reference front used in $\epsilon$-indicator and IGD$^+$ affect the quality of the archiving algorithms based on them but not affecting them meeting the theoretical property.}

	\subsection{Convergence-Guaranteed Archivers with Weakly-Pareto-Compliant Indicators}\label{sec:theory_weakly}
	
	%In accordance with existing studies~\cite{KnoCor2004lnems,CorKno2003cec,LauThiZitDeb2002archiving,LopKnoLau2011emo},
	%we consider solutions of the sequence fed to the archive in an one-by-one form,
	%i.e., $S=\langle Set^{(1)}, Set^{(2)},...,Set^{(t)},...\rangle$, $\forall t, \abs{Set^{(t)}} = 1$;
	%for convenience,
	%we replace $Set^{(t)}$ with $s^{(t)}$.
	%We denote the sequence up to the timestep $t$ by $S^{(t)}$,
	%i.e., $S^{(t)}=\langle s^{(1)}, s^{(2)},...,s^{(t)}\rangle$.
	
	Weak Pareto compliance is a weaker version of Pareto compliance for quality indicators.
	If an indicator is weakly Pareto compliant,
	then it may not be able to distinguish between two sets, even if one is better  than the other with respect to Pareto optimality (Def.~\ref{def:better}).
	For example, given the sets
	$A = \{(0,1), (0.5,0.5), (1,0)\}$ and $B=\{(0,1),(1,0)\}$, then $A \Better B$
	($A$ provides the decision maker with one more option than $B$).
	A weakly Pareto compliant indicator may evaluate them to be the same,
	but a Pareto compliant indicator will evaluates $A$ to be better than $B$.
	
	We show here that, with respect to the theoretical properties for archiving presented above, 
	weakly-Pareto-compliant indicators are as desirable as Pareto-compliant ones, if the archiving rules are designed properly.
	Until now, only archivers based on a Pareto-compliant indicator were known to hold properties such as \emph{limit-optimal}
	%\MANUEL*{commonly believed but we do not provide a reference.}
	%\MIQING*{it seems hard to find a reference... this is actually what I believed before coming up with that weakly Pareto compliant indicator-based archiver in the paper. We may need to adjust the narrative a bit here.}
	(see Table~\ref{Table:classifcation} in the next section).
	
	In this section, in accordance with previous studies~\cite{KnoCor2004lnems,CorKno2003cec,LauThiZitDeb2002archiving,LopKnoLau2011emo},  we assume the input sequence is a sequence of individual solutions presented to the archiving algorithm one at a time, that is,
	$\forall t, \abs{S^{(t)}} = 1$ (Section~\ref{sec:archiving}).
	We can always convert a many-at-a-time sequence into a one-at-a-time sequence,
	thus properties that hold for archivers that handle the latter also hold for archivers that handle the former.
	The opposite is not true, however, as we will discuss later in Section~\ref{sec:size}.
	To make explicit when we refer to a one-at-a-time sequence,
	we denote its elements by $s^{(t)}$ and the sequence up to the time $t$ by $\mathcal{S}^{(t)} = \langle s^{(1)}$, $s^{(2)},$ $\dots,$ $s^{(t)}\rangle$.
	% \MANUEL*{I moved this here because we did not need this assumption until now.}
	% \MIQING*{good!}
	
	We consider a generic archiving algorithm (Algorithm~\ref{Alg:Weak_compliance}) based on a weakly Pareto-compliant indicator $I$.
	The archiving rules we propose are very general and similar to those used in most indicator-based archivers: (1) uses weak Pareto dominance relation to compare solutions in the archive with the new solution;
	(2) checks if the nondominated set after adding the solution exceeds the archive capacity;
	(3) checks if adding the new solution does not lead to a better indicator value;
	otherwise (4) removes the archived solution that contributes the least to the indicator value after adding the new solution.  

	We now prove that, assuming the indicator $I$ is weakly Pareto compliant, the archiver in Algorithm~\ref{Alg:Weak_compliance}
	holds the three limit properties:
	\textit{limit-stable}, \textit{limit-Pareto-subset} and \textit{limit-optimal}.
	To do so, we first need to introduce several auxiliary properties of the archiver.

	\begin{lemma}
		The  $I$ value of the archive under Algorithm~\ref{Alg:Weak_compliance} never degrades:
		$\forall t \in\mathbb{N}^{+}, I(A^{(t+1)}) \leq I(A^{(t)})$.
		% \MANUEL*{I changed this because we said we minimize $I$, so it seems wrong. Also we do not need the $i$. I have updated the rest.}
		% \MIQING*{good point! thanks. Just wonder if it is better to use $t+1$ rather than $t-1$ since $t \in\mathbb{N}^{+}$. $t-1$ can make the existence of $A^{(0)}$.}
		\label{lemma:1}
	\end{lemma}
	\begin{proof}
		Rules 1 and 3 will not change the $I$ value as the archive remains the same.
		Rule~4 will lead to a better $I$ value according to the definition.
		For Rule~2, $A^{(t)} = \min(A^{(t-1)} \cup \{s^{(t)}\}, \prec)$ and
		the nondominated set of the union of the archive and the new solution is \emph{better}
		than the archive, i.e., $ A^{(t)} \Better A^{(t-1)}$,
		which implies $I(A^{(t)}) \leq I(A^{(t-1)})$
		because of the definition of a weakly Pareto-compliant indicator 
		(Definition~\ref{Def:weak_compliance}).
		Thus, the $I$ value of the archive will never degrade.
		% we have $\min(A^{(t-1)} \cup \{s^{(t)}\}, \prec) \preceq A^{(t-1)}$, we know either $\min(A^{(t-1)} \cup \{s^{(t)}\}, \prec) = A^{(t-1)}$ or $\min(A^{(t-1)} \cup \{s^{(t)}\}, \prec) \Better A^{(t-1)}$. If $\min(A^{(t-1)} \cup \{s^{(t)}\}, \prec) = A^{(t-1)}$, then $I(\min(A^{(t-1)} \cup \{s^{(t)}\},\prec)) = I(A^{(t-1)})$. If $\min(A^{(t-1)} \cup \{s^{(t)}\}, \prec) \Better A^{(t-1)}$, then $I(\min(A^{(t-1)} \cup \{s^{(t)}\},\prec)) \leq I(A^{(t-1)})$ according to the definition of an indicator being weakly Pareto compliant (Definition~\ref{Def:weak_compliance}). So, the archive's $I$ value will never degrade.
	\end{proof}

	\begin{lemma}
		Under Algorithm~\ref{Alg:Weak_compliance},
		if the archive is different at two different timesteps, $t$ and $t+i$,
		then for any timestep after $t+i$,
		the archive is always different from the archive at timestep $t$:
		$\forall t, i \in\mathbb{N}^{+}, A^{(t)} \neq A^{(t+i)}   \implies \forall j \in \mathbb{N}^{+}, A^{(t)} \neq A^{(t+i+j)}$.
		\label{lemma:2}
	\end{lemma}
	\begin{proof}
		By contradiction. %
		Assume $A^{(t)} \neq A^{(t+i)}$ and $\exists j \in \mathbb{N}^{+}$ such that $A^{(t)} = A^{(t+i+j)}$.
		Due to $I(A^{(t)}) = I(A^{(t+i+j)})$ and since the $I$ value of the archive never degrades (Lemma~\ref{lemma:1}),
		we have $I(A^{(t)}) = I(A^{(t+i)}) = I(A^{(t+i+j)})$.
		This implies that, from timestep $t$ to $t+i+j$,
		the archiving process never goes through Rule 4,
		since Rule 4 necessarily leads to a better $I$ value.
		
		Consider the archiving process from the timestep $t$ to $t+i$.
		Since $A^{(t)} \neq A^{(t+1)}$ and Rules 1 and 3 do not change the archive,
		the archiving process must go through Rule 2 at least once,
		where it accepts a new solution $a$ that is not dominated by any solution in $A^{(t)}$.
		%Let us denote that new solution by $a$.
		
		To satisfy our assumption that $A^{(t)} = A^{(t+i+j)}$, the archiving process must eliminate solution $a$ between timestep $t+i$ and $t+i+j$,
		but without accepting any new solution. % since $A^{(t+j)} = A^{(t)}$.
		This must happen at Rule 2 since Rules 1 and 3 will not change the archive.
		But Rule 2 will only eliminate solution $a$ in the archive if it is dominated by the new solution added to the archive. Likewise, this new solution added to the archive cannot be eliminated through Rule 2 without accepting a newer and better solution.
		Therefore, the archiving process cannot remove $a$ without accepting a new solution that dominates it and $A^{(t+i+j)}$ can never go back to $A^{(t)}$, thus the assumption cannot hold.
	\end{proof}
	
	Lemma~\ref{lemma:2} means that Algorithm~\ref{Alg:Weak_compliance} cannot revisit again a previous archive after the archive has changed.
	With that in mind,
	we are in a position to prove that the three limit desirable properties hold for Algorithm~\ref{Alg:Weak_compliance}.
	
	\begin{theorem}
		Algorithm~\ref{Alg:Weak_compliance} is limit-stable:\newline
		\hspace*{\fill} $\exists t, \forall i \in\mathbb{N}^{+}$,
		\mbox{$A^{(t)} = A^{(t+i)}$}.
		\label{theorem:1}
	\end{theorem}
	\begin{proof}
		By contradiction.
		Assume the archive never converges,
		i.e. $\forall t, \exists i\in\mathbb{N}^{+}$, $A^{(t)} \neq A^{(t+i)}$.
		This implies that there are an infinite number of different archives since none can be revisited (Lemma~\ref{lemma:2}).
		However, since input solutions are drawn from the finite set $Y$,
		there must be a finite number of different archives,
		thus a contradiction.
	\end{proof}
	
	\begin{theorem}
		Algorithm~\ref{Alg:Weak_compliance} is limit-Pareto-subset:\newline
		\hspace*{\fill}$\exists t, \forall i\in\mathbb{N}^{+}$, \mbox{$A^{(t)} = A^{(t+i)}$} and \mbox{$\min (A^{(t)}, \prec) \subseteq Y^{*}$},
		where \mbox{$Y^{*} = \min(Y,\prec)$}.
		\label{theorem:2}
	\end{theorem}
	% \MANUEL*{I think we do not need $A^{(t+i)} = A^{(t)}$ here, that is a different property. In principle, you could have an archiver that is limit-Pareto-subset and not limit-stable (it replaces a Pareto-optimal solution with a different one)}
	% \MIQING*{interesting point. I suspect it would be reasonable to include that; from the top of my head, I can't think of any archiver which is limit-Pareto-subset but not limit-table. An archiver with the rule that replaces a Pareto-optimal solution with a different one cannot guarantee a limit-Pareto subset. Note that it is an archiver to have those properties (i.e., the properties hold for any sequence), rather than for a specific sequence. Having that said, I will think about it later.}
	\begin{proof}
		According to Theorem~\ref{theorem:1}, we have
		\mbox{$\exists t, \forall i \in\mathbb{N}^{+}$}, \mbox{$A^{(t)} = A^{(t+i)}$}.
		Moreover, the archive is always a nondominated set, i.e.,
		$A^{(t)} = \min (A^{(t)}, \prec)$, because Rule 1 prevents adding a new
		solution that is weakly dominated by any solution in the archive and Rule 2
		removes archived solutions that are dominated by the new solution. Thus, we
		only need to prove $A^{(t)} \subseteq Y^{*}$.
		
		By contradiction: Let us assume $A^{(t)} \nsubseteq Y^*$.
		As $Y^*$ is the set of all the Pareto-optimal solutions of $Y$,
		there exists at least one solution in $A^{(t)}$ that is dominated by at least one solution $y^* \in Y^*$.
		Since all solutions have a non-zero probability of being generated in a future timestep, then $\exists i \in\mathbb{N}^{+}$, such that the archiver receives $s^{(t+i)} =y^*$. 
		Then,
		the algorithm must go to Rule~2 since there is no solution in $A^{(t+i-1)}$
		weakly dominating $y^*$, which is Pareto optimal.
		%(due to $A^{(t+i-1)}=A^{(t)}$).\MANUEL*{I don't understand this ``due''. Is it not enough that $y^{*}$ is Pareto-optimal, so it can never be weakly dominated by a solution in the archive except it is already in the archive?}
		%\MIQING*{yes, that is correct! we don't need ``due''.}
		Since we assumed that there is at least one solution in $A^{(t)} = A^{(t+i-1)}$ dominated by $y^*$, then $\abs{\min(A^{(t+i-1)} \cup \{y^*\},\prec)}\leq N$ and Rule~2 will accept $y^*$, which implies that 
		$A^{(t)} \neq A^{(t+i)}$, thus contradicting Theorem~\ref{theorem:1}.
	\end{proof}
	
	\begin{theorem}
		Algorithm~\ref{Alg:Weak_compliance} is limit-optimal:\newline
		\hspace*{\fill}$\exists t, \forall i \in\mathbb{N}^{+}$, $A^{(t)} = A^{(t+i)}$ and $\min(A^{(t)}, \prec)$ is an optimal approximation of size $N$ of $Y^*$, where $Y^* = \min(Y, \prec)$ and
		$N$ is the capacity of the archive.
		\label{theorem:3}
	\end{theorem}
	\begin{proof}
		According to Theorem~\ref{theorem:1},
		we have $\exists t, \forall i \in\mathbb{N}^{+}$, $A^{(t)} = A^{(t+i)}$.
		Its proof also shows that \mbox{$A^{(t)} = \min (A^{(t)}, \prec)$}. 
		Thus, we need to prove that $A^{(t)}$ is an optimal approximation of bounded size (Definition~\ref{Def:Optimal_approximation}), i.e., 
		$\nexists B \subseteq Y$, $\abs{B}\leq N$ such that $B \Better A^{(t)}$.

		It is easy to see that $\abs{A^{(t)}} \leq \min\{N,\abs{Y^{*}}\}$
		because $N$ is the capacity of the archive and $A^{(t)} \subseteq Y^{*}$ (Theorem~\ref{theorem:2}). 
		Now we prove the theorem by considering three cases: 
		(i) $\abs{A^{(t)}} = \abs{Y^{*}}$, 
		(ii) $\abs{A^{(t)}} = N$,
		and (iii) $\abs{A^{(t)}} < \min \{N,\abs{Y^{*}}\}$.
		
		Let us first consider the case $\abs{A^{(t)}} = \abs{Y^{*}}$.
		Rule 1 in Algorithm~\ref{Alg:Weak_compliance} forbids duplicated solutions in $A^{(t)}$. Thus, $\abs{A^{(t)}} = \abs{Y^{*}}$ implies $A^{(t)} = Y^{*}$ (Theorem~\ref{theorem:2}) and the archive is optimal (it contains the complete Pareto front).
		Thus, $\nexists B \subseteq Y$ such that $B \Better A^{(t)}$.
		
		Let us now consider the case $\abs{A^{(t)}} = N$.
		According to Rule~1 and $A^{(t)} \subseteq Y^{*}$ (Theorem~\ref{theorem:2}), we know that all solutions in $A^{(t)}$ are unique elements of $Y^{*}$. Assume that $\exists B \subseteq Y$, $\abs{B} \leq N$, such that $B \Better A^{(t)}$, which implies $B \preceq A^{(t)} \land A^{(t)} \npreceq B$. 
		Since $B \preceq A^{(t)}$, $B$ should contain all solutions in $A^{(t)}$ as they are unique elements of $Y^{*}$, i.e., $A^{(t)} \subseteq B$. 
		In addition, $A^{(t)} \npreceq B$ implies that $A^{(t)} \neq B$, thus $A^{(t)} \subset B$ 
		%$\exists b \in B$, $\forall a \in A^{(t)}$, such that $a \npreceq b$. 
		and $\abs{B}>N = \abs{A^{(t)}}$, a contradiction with $\abs{B} \leq N$, thus there is not such $B \Better A^{(t)}$.
		
		%The latter implies $A^{(t)} \neq B$ and, together with $\abs{B} \leq \abs{A^{(t)}}$, means that $\exists a \in A^{(t)}$, $\forall b \in B$, such that $(b \prec a)$ or $(a \prec b)$ or $(a \npreceq b \land b \npreceq a)$.\  The alternative \mbox{$b \prec a$} cannot happen because no solution in $B$ can dominate a solution in $A^{(t)} \subseteq Y^{*}$. The two remaining alternatives \mbox{$(a \prec b)$} or $(a \npreceq b \land b \npreceq a)$ both imply that $b \npreceq a$, which means that \mbox{$B \npreceq A^{(t)} \implies B \not\Better A^{(t)}$}.

		Lastly, let us consider the case  $\abs{A^{(t)}} < \min \{N,\abs{Y^{*}}\}$. $\abs{A^{(t)}} < \min \{N,\abs{Y^{*}}\}$ implies that $A^{(t)}$ is missing at least one solution from $Y^{*}$. If the missed solution(s) are already duplicated in $A^{(t)}$, the archive is optimal. 
		Let us assume one of the missed solutions is not duplicated; then Rule~2 will accept the solution because $\abs{A^{(t)}} <  N$, contradicting the initial assumption that $A^{(t)} = A^{(t+i)}$. 
		Therefore, when $\abs{A^{(t)}} < \min \{N,\abs{Y^{*}}\}$, 
		$A^{(t)}$ must consist of all unique elements of $Y^{*}$, thus
		$\nexists B \subseteq Y$ such that $B \Better A^{(t)}$.
		%Hence, if the archive converges, then it must be either full, $\abs{A^{(t)}} = N$, or contain all solutions from $Y^{*}$ without duplicates, which is optimal.
		%  
		%By contradiction.
		%	According to Lemma~\ref{lemma:3},
		%	we have $\exists i, \forall i \in\mathbb{N}, A^{(t+i)} = A^{(t)}$.
		%	Let us assume $\abs{A^{(t)}} < N$.
		%	In this case,
		%	after the archive converges,
		%	any solution from $Y$ which is not dominated by any solution in $A^{(t)}$ will be still accepted by the archive (Rule~2),
		%	thus a contradiction to the fact that the archive converges.
		%	So, $\abs{A^{(t)}} = N$.
		%	From Lemma~\ref{lemma:4},
		%	we already know $A^{(t)} \subseteq Y^*$.
		%	Therefore,
		%	we can conclude that when the archive converges,
		%	it has the full size and each of its solutions is from $Y^*$.
	\end{proof}

	In summary, an archiver based on a weakly Pareto-compliant indicator can respect the three limit properties.
	In addition,
	Algorithm~\ref{Alg:Weak_compliance} also respects the \textit{set-monotone}  property (Property~\ref{prop:set_monotone});
	the proof is straightforward given the definition of a weakly Pareto-compliant indicator.
	The overall conclusion is that archivers based on a weakly Pareto-compliant indicator
	can hold
	the same theoretical desirables as archivers based on a Pareto-compliant indicator.
	This conclusion may explain recent empirical observations~\cite{FalZapGar2021gecco}
	showing no significant difference between MOEAs guided by either weakly Pareto-compliant indicators or Pareto-compliant indicators.
	%\MANUEL{Intuitively, there must be some difference \ldots }
	%\MIQING{I think that the archiver based on a weakly Pareto compliant indicator needs to ensure there is no duplicate nondominated solutions. This is not the case for the archiver based on a Pareto compliant indicator.}
	These theoretical and empirical results should encourage the study of archivers based on weakly Pareto-compliant indicators, 
	since many indicators meet the condition of being weakly Pareto compliant, including $\epsilon$-indicator~\cite{ZitThiLauFon2003:tec}, \IGDplus~\cite{IshMasTanNoj2015igd},
	R2~\cite{HanJas1998}, PCI~\cite{LiYanLiu2015pci}, IPF~\cite{BozFowGelKim2010or} and others~\cite{LiYao2017arxiv,CaiXiaLiHu2021grid,ValMarWanDeb2021mipdom,CaiXiaLiSun2021kernel}.
	%\MANUEL*{Is there a published version of \cite{LiYao2017arxiv}? I would prefer to not cite arxiv papers unless they are essential to our argument. There are many examples to choose from. }
	%\MIQING*{unfortunately no, this is the paper that I still try to work on. The idea is interesting: calculating the minimal move of a set of solutions to make to dominate another set. But I still cannot figure out a polynomial-time method for MOPs with more than two objectives though the bi-objective case can be polynomial time complexity. Recently, someone else tries to regard calculating this indicator as an mixed integer programming problem and solve it by methods in MIP.}
	%\MANUEL{how many of these require a reference front?}
	%\MIQING{some do not require a reference point such as R2 and PCI.}

	%This together implies that weakly Pareto compliant indicators have the same theoretical desirables as Pareto compliant indicators for archiving.

	%%%% Table I %%%%
	\begin{table*}[htbp]
		\newcommand{\propcell}[1]{\multirow[b]{3}{=}{\centering\textit{#1}}}
		\newcommand{\cell}[2]{\multicolumn{1}{#1}{#2}}
		\caption{Classification of representative archiving algorithms and their theoretical and practical desirables.
			% Class 1 (C1) refers to archivers that do not meet any theoretical properties.
			% Class 2 (C2) refers to archivers that meet some theoretical properties but are not very useful in practice.
			% Class 3 (C3) refers to archivers that meet some theoretical properties and are useful in practice, but does not meet the \textit{limit-optimal} property.
			% Class 4 (C4) refers to archivers that meet the \textit{limit-optimal} property and are useful in practice.
		}\label{Table:classifcation}
		\footnotesize%
		\centering%
		\def\arraystretch{1.2}%
		\begin{tabular}{@{}c|@{}C{9em}@{}|@{}C{1cm}@{}|@{}C{1.2cm}@{}|@{}C{1.2cm}@{}|@{}C{1cm}@{}|@{}C{1cm}@{}|@{}C{1cm}@{}|@{}C{1.3cm}@{}|@{}C{1.5cm}@{}|@{}C{1.4cm}@{}|@{}C{\widthof{No problem-}+.5ex}@{}|}
			\cline{3-12}
			\cell{c}{}		& & \multicolumn{6}{c|}{Theoretical desirables} &  \multicolumn{4}{c|}{Practical desirables}  \\\cline{3-12}
			\cell{c}{}           &
			& \propcell{Pareto subset} & \propcell{Point-monotone} & \propcell{Set-monotone}  & \propcell{Limit-stable} & \propcell{Limit-Pareto subset}  & \propcell{Limit-optimal}  & \propcell{Diversifies}   & \propcell{Controllable size} & \propcell{Polynomial time}  &  \propcell{No problem-specific parameter}
			\\
			\cell{c}{}&&&&&&&&&&&\\\cline{1-2}
			Class & Archiver &&&&&&&&&& \\\hline
			\multirow{3}{*}{I} & NSGA-II~\cite{Deb02nsga2} &  &  &  &  &  &  & + & + & + & + \\
			& SPEA2~\cite{ZitLauThie2002spea2} &  &  &  &  &  &  & + & + & + & +\\
			% & AGA~\cite{KnoCor2003tec} &  &  &  &  &  &  & + & + & + & +\\ \cline{2-12}
			% & IBEA~\cite{ZitKun2004ppsn} &  &  &  &  &  &  & + & + & + & +\\
			& NSGA-III~\cite{DebJain2014:nsga3-part1} &  &  &  &  &  &  & + & + & + & +\\\hline
			% & AA$_{\Delta_1}$~\cite{RudTraSen2013evenly} &  &  &  &  &  &  & + & + & + & +\\\hline\hline
			\multirow{3}{*}{II} & \Adom~\cite{RudAga2000cec} &  & + & + & +  & + & + & & + & + & + \\
			& $\epsilon$-approx~\cite{LTDZ2002b} & & + & + & + &  &  & + &  & + & \\
			& $\epsilon$-Pareto~\cite{LTDZ2002b} & + & + & + & + & + &  & + &  & + & \\\hline
			%& MOEA/D-WS~\cite{ZhaLi07:moead} & + & + & + & + & + &  & -- & +  & + & + \\\hline
			%& A$_{-\epsilon+box}$~\cite{SchHerTal2019archiver} & &  &  & + &  &  &  &  & + & \\\hline\hline

			%\multirow{2}{*}{C3} & $\epsilon$-approx~\cite{LTDZ2002b} & & + & + & + &  &  & + &  & + & \\
			%& $\epsilon$-Pareto~\cite{LTDZ2002b} & + & + & + & + & + &  & + &  & + & \\\hline
			%& $\epsilon$-MOEA~\cite{DebMohMis2005epsilon} & + & + & + & + & + &  & + &  & + & \\\hline\hline

			\multirow{3}{*}{III} & MOEA/D-PBI~\cite{ZhaLi07:moead} &  &  &  & + &  &  & + & + & + & +\\
			& MOEA/D-TCH~\cite{ZhaLi07:moead} & & -- & -- & + & + &  & + & + & + & + \\
			& $\ARtwo^*$~\cite{BroWagTrau2015r2} & & -- & -- & + & + &  & + & + & + & + \\\hline

			\multirow{3}{*}{IV}	& \AHV~\cite{KnoCor2003tec} & & & + & +  & + & + & + & + & &  \\
			& SMS-EMOA~\cite{BeuNauEmm2007ejor} &  &  & -- & --  & -- & --  & + & + &  & + \\
			& MGA~\cite{LauZen2011ejor} & & & + & +  & + & + & + & + & + & +\\\hline
		\end{tabular}
		\vskip 1ex
		\parbox{\linewidth}{%
			\footnotesize $^*$Here \ARtwo is slightly different from~\cite{BroWagTrau2015r2}, in which the new solution will be rejected if it has the same lowest fitness as the old ones (see Algorithm~\ref{Alg:IBEA}). \\
			``+'' indicates that the archiver can fully respect the specified desirable
			and ``--'' indicates that the archiver can respect the desirable under certain condition.}
	\end{table*}

	%\MANUEL{The table mentions Polynomial time and ``No prior knowledge'' and it is unclear what this means.}
	%\MANUEL{In Table I, it would be great to align Class and Archive to the bottom of their cells. }
	
	\section{Classification of Existing Archivers}\label{sec:classification}
	% \MANUEL*{We talk about classes not categories so I changed the word categorisation to classification}
	% \MIQING*{good point!}
	
	In this section,
	we review existing archivers in the literature
	% from the theoretical perspectives on the basis of properties in the last section.
	on the basis of the theoretical properties in the previous section
	Apart from those theoretical desirables,
	there may also exist practical desirables for archivers to respect.
	
	For example,
	one may wish that (1) an archiver \emph{diversifies}, i.e., avoids convergence to a small region of the Pareto front;
	(2) the size of its archive is controllable, not only respecting any user-defined maximum capacity, but being as full of nondominated solutions as possible;
	(3) the archiving operation does not take too much time,
	e.g., not exponentially increasing with the number of objectives;
	and (4) the archiving process does not need any problem-dependent parameter set by the user.
	%Here, we also include these practical desirables along with the theoretical desirables.
	Non-diversifying archivers include efficiency preserving archivers~\cite{Hanne1999ejor} that, when full, only accept solutions that dominate an archived solution, thus they often converge to one or few small regions of the Pareto front. Non-efficiency preserving archivers may also fail to diversify; for example, an archiver that removes the solution farthest away from an ideal point or an archiver that select solutions according to their distance to reference vectors, if the vectors used are not well-distributed along the Pareto front.

	According to these properties, %Loosely,
	archivers can be categorised into four classes.
	The first class (I) refers to archiving algorithms (or selection criteria)
	that do not hold any theoretical desirables.
	Archivers in many well-established MOEAs belong to this class.
	%It includes not only Pareto-based algorithms
	%but also indicator-based and decomposition-based algorithms.
	The second class (II) refers to those that never deteriorate (i.e., hold the \emph{point-monotone} property) but are not very useful in practice due to failing to diversify or not using their full capacity to store nondominated solutions.
	The third class (III) refers to those that have some good theoretical and practical properties, but are not \emph{limit-optimal}.
	The fourth class (IV) refers to those that, in addition to having have  good theoretical and practical properties, are also \emph{limit-optimal} (under certain conditions).
	% \MANUEL*{This list and Table I use arabic numerals (C1, C2, etc.) but the section names use roman numerals (Class I, Class II, etc.). I don't mind which one to use but we should be consistent.}
	% \MIQING*{Roman ones look slightly better to me, but it is up to you.}
	
	Table~\ref{Table:classifcation} shows several representative archivers in the four classes
	and their theoretical and practical properties.
	For some classes,
	there are numerous archivers in the area (such as Class I and III)
	and we only consider representative algorithms.
	For classes where there are very few archivers (such as Class IV),
	we aim to list them completely provided that they are significantly different in terms of the archiving criteria used.
	%In addition, for convenience we use the name of a well-known algorithm to represent its archiver in result descriptions.
	
	\subsection{Class I: Archivers Holding No Theoretical Properties}

	\begin{algorithm}[tb]
		\caption{Archiver based on NSGA-II's selection rules.}
		\label{Alg:NSGA-II}
		\small%
		\SetAlgoLined
		\DontPrintSemicolon
		\KwIn{$A^{(t-1)}$, $s^{(t)}$}
		\tcp{Partition all the solutions into different nondominated fronts}
		\tcp{and identify the last front $F_l$.}
		$(F_1,F_2,\dots,F_l) \assign \text{nondom\_sorting}(A^{(t-1)} \cup s^{(t)})$\;
		\tcp{Find solution in $F_l$ with the minimum crowding distance.}
		$a \assign \argmin_{a\in F_l} \text{crowding\_distance}(F_l)$\;
		$A^{(t)} \assign  A^{(t-1)} \cup \{s^{(t)}\} \setminus \{a\}$\;
		\KwOut{$A^{(t)}$}
	\end{algorithm}

	The first class contains archivers from many well-established MOEAs that
	do not hold any theoretical desirables.
	It includes all Pareto-based (Pareto dominance $+$ density) algorithms  % \MANUEL*{``\emph{Pareto dominance $+$ density}''}
	% \MIQING*{Technically, it should be ``Pareto dominance $+$ density'', but usually when saying a Pareto-based algorithm, people may think it is the one with Pareto dominance + density}
	and some archivers in other types of algorithms (e.g., decomposition-based ones).
	
	In general,
	the archiving procedure (i.e., environmental selection procedure)
	in Pareto-based MOEAs,
	%\MANUEL*{The key is the density which is not Pareto-compliant, so I would not call them Pareto-based but density-based.}
	%\MIQING*{Technically, it may be better to call density-based, but Pareto-based unfortunately is widely called already for over 15 years, back to 2007 \cite{Wagner2007}}
	such as NSGA-II~\cite{Deb02nsga2} and SPEA2~\cite{ZitLauThie2002spea2},
	consists of two steps:
	considering Pareto dominance first and then solutions' density.
	As an example, in NSGA-II, first
	a nondominated sorting procedure divides the archive into different nondominated fronts,
	and then a density metric (crowding distance)
	is used to select among solutions in the last front;
	Algorithm~\ref{Alg:NSGA-II} gives the procedure of an archiver based on NSGA-II's selection rules.
	In such archivers,
	%the first step is in line with theoretical desirables
	%as its rules are based on solutions' Pareto dominance relation.
	%However,
	%the density-based rules in the second step may cause the deterioration.
	the density-based rules are the cause of set-deterioration\footnote{Set-deterioration (see Prop.~\ref{prop:set_monotone}) implies point-deterioration (Prop.~\ref{prop:point_monotone}).} because they may eliminate a nondominated, or even Pareto-optimal, solution that may end up dominating another solution later accepted. 
	As shown in Figure~\ref{fig:Pareto_deterioration} previously,
	%\MANUEL{If this must serve as a counter-example, then it should be more precise and show the axes like the examples in \cite{ZitThiLauFon2003:tec}}
	an inferior solution can enter the archive
	provided that it is located in a sparser region and there is no solution in the current archive dominating it.
	% The fundamental reason behind the deterioration is that
	% in such an archiving process there may exist different timesteps
	% where two nondominated solutions may have different superiority relations
	% (e.g., one being more crowded in the archive at a timestep but may become less crowded after some timesteps).
	
	Set-deterioration not only arises with density-based rules
	but also with the decomposition-based rules in NSGA-III~\cite{DebJain2014:nsga3-part1}
	and the rules that combine solutions' density with proximity to the Pareto front in SDE~\cite{LiYanLiu2014shift}.
	Figure~\ref{fig:NSGA-III} illustrates how NSGA-III's archiving rules produce set-deterioration.
	The rules in NSGA-III first consider the Pareto dominance relation between solutions
	and, if they are nondominated, then compare their closeness to the weight vectors.
	As can be seen in Figure~\ref{fig:NSGA-III},
	the archive eliminates solution $a$, but after two timesteps it accepts solution $c$, which is dominated by $a$, thus the archive at $t+2$ is worse than at $t$.

	\begin{figure}[tbp]
		\centering%
		\footnotesize%
		\includegraphics[scale=0.45]{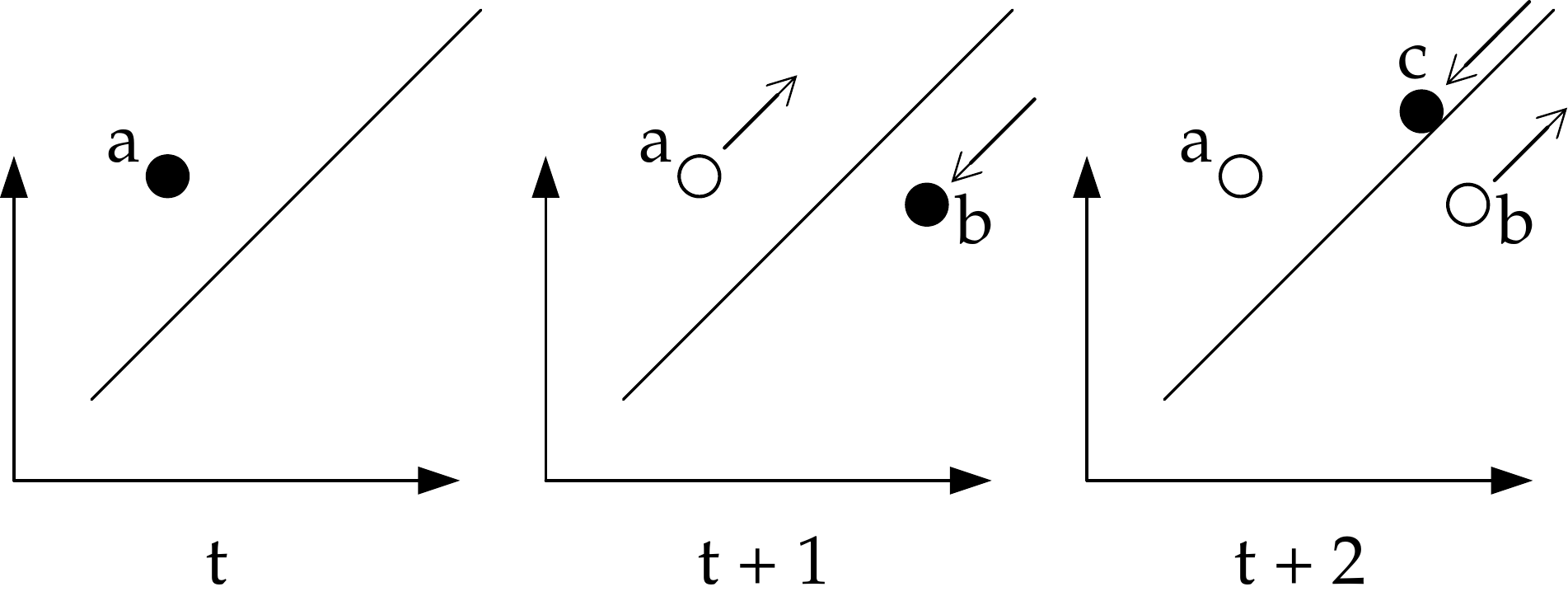}
		\caption{Illustration of deterioration when using NSGA-III's archiving rules (adapted from~\cite{Fie2017gecco}). Black circles denote solutions in the archive,
			hollow circles denote solutions removed,
			and the solid line is the weight vector considered.
			At the timestep $t$, the archive contains solution $a$.
			At the timestep \mbox{$t+1$}, solution $b$ replaces $a$ since $b$ is closer than $a$ to the line.
			At the timestep \mbox{$t+2$}, solution $c$ replaces $b$ since $c$ is closer to the line.
			However, $c$ is dominated by $a$ and the archive at $t$ was better ($\Better$) than the one at $t+2$, thus the archive both \emph{set-deteriorates} and \emph{point-deteriorates}.}
		\label{fig:NSGA-III}
	\end{figure}

	\subsection{Class II: Archivers Holding Some Theoretical Properties but not Useful in Practice}\label{sec:class2}
	
	%%%% Fig. 3 %%%%
	\begin{figure}[tbp]
		\centering%
		\footnotesize%
		\includegraphics[scale=0.5]{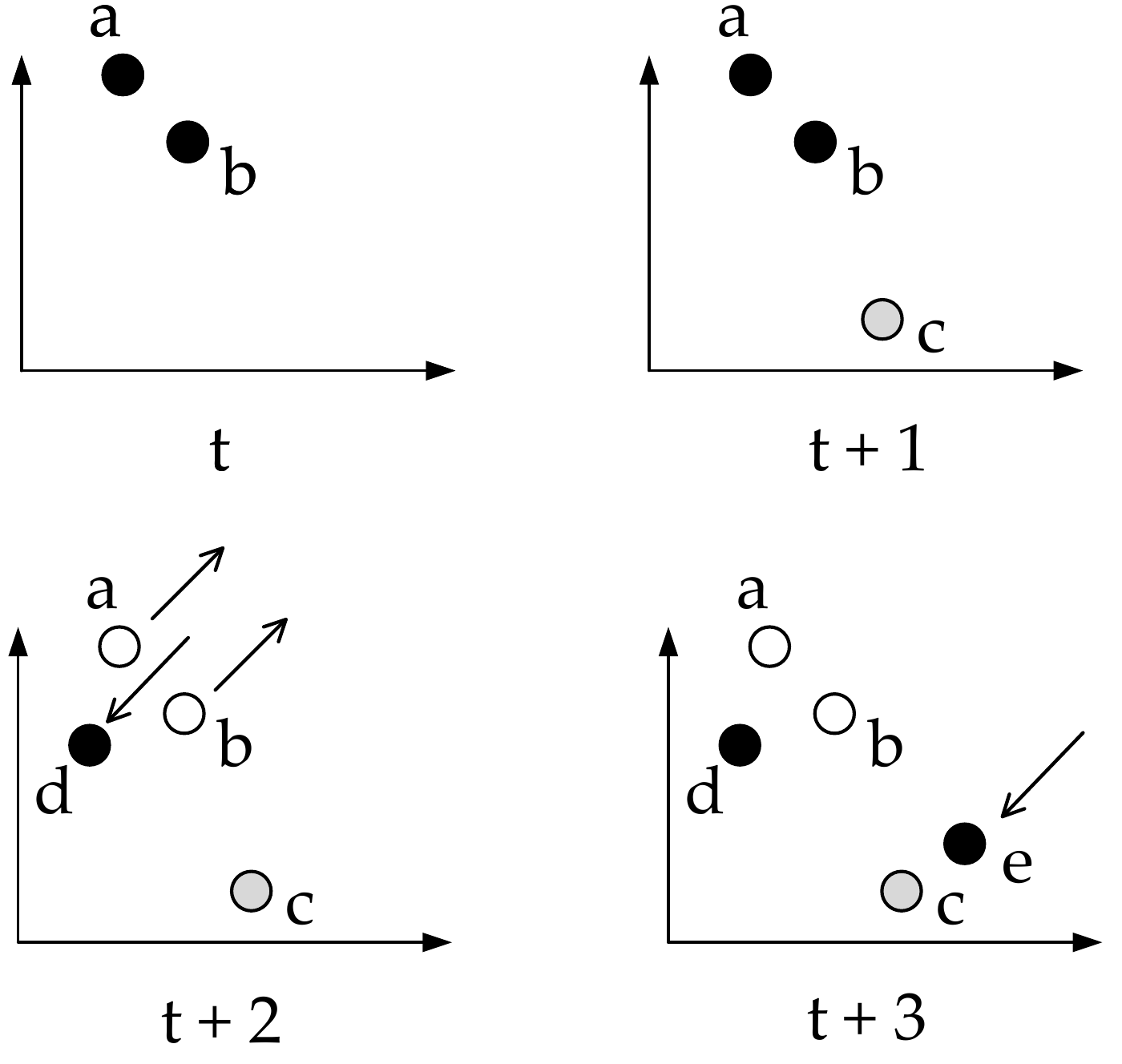}
		\caption{Illustration that the archiver \Adom does not hold the property
			\textit{Pareto-subset} (Property~\ref{prop:Pareto_subset}).
			The capacity of the archive is 2.  Here, black circles denote solutions in the
			archive, hollow circles denote solutions removed from the archive and grey
			circles denote solutions that are not allowed to enter the archive.  At the
			timestep $t+$1, solution $c$ cannot enter the archive because it does not
			dominate any solution in the current archive and the archive is at full
			capacity.  At the timestep $t+$2, solution $d$ enters the archive and
			solutions $a$ and $b$ are removed since they are dominated by $d$.  At the
			timestep $t+$3, solution $e$ enters the archive since it is not dominated by
			$d$ and the archive is not at full capacity.  However, solution $e$ is
			dominated by a solution of the sequence (i.e., $c$), thus the archiver does not
			hold the \textit{Pareto-subset} property.}\label{fig:dominating}
	\end{figure}
	
	% \begin{algorithm}[tb]
	%   \caption{\Adom Archiver}
	%   \label{Alg:Dominating}
	%   \small%
	%   \DontPrintSemicolon
	%   \SetAlgoLined
	%   \KwIn{$A^{(t-1)}$, $s^{(t)}$}
	%   \uIf{$\exists a \in A^{(t-1)}$, $a \prec s^{(t)}$}{
	%     $A^{(t)} \assign A^{(t-1)}$}
	%   \Else{
	%     $D \assign \{a \in A^{(t-1)} \mid s^{(t)} \prec a\}$\;
	%     \uIf{$D \neq \emptyset$}
	%     {
	%       $A^{(t)} \assign A^{(t-1)} \cup \{s^{(t)}\} \setminus D$
	%     }
	%     \uElseIf{$\abs{A^{(t-1)} \cup \{s^{(t)}\}} \leq N$}
	%     {
	%       $A^{(t)} \assign A^{(t-1)} \cup \{s^{(t)}\}$
	%     }
	%     \Else
	%     {$A^{(t)} \assign A^{(t-1)}$}
	%   }
	%   %	\uElseIf{$\abs{nondom(A^{(t-1)} \cup \{s^{(t)}\})} \leq N$}{
	%   % $A^{(t)} \assign nondom(A^{(t-1)} \cup \{s^{(t)}\})$
	%   % \tcp*{If the number of all nondominated solutions is less than or equal to the archive capacity after adding the new solution}
	%   % }
	%   %   \Else{
	%   %   $A^{(t)} \assign A^{(t-1)}$
	%   %	}
	%   \KwOut{$A^{(t)}$}
	% \end{algorithm}
	\begin{algorithm}[tb]
		\caption{\Adom Archiver}
		\label{Alg:Dominating}
		\small%
		\DontPrintSemicolon
		\SetAlgoLined
		\KwIn{$A^{(t-1)}$, $s^{(t)}$}
		\uIf{$\exists a \in A^{(t-1)}$, $a \preceq s^{(t)}$}{
			$A^{(t)} \assign A^{(t-1)}$
		}
		\uElseIf{$\abs{\min(A^{(t-1)} \cup \{s^{(t)}\},\prec)} \leq N$}{
			$A^{(t)} \leftarrow \min(A^{(t-1)} \cup \{s^{(t)}\},\prec)$
		}\Else{
			$A^{(t)} \assign A^{(t-1)}$
		}
		%	\uElseIf{$\abs{nondom(A^{(t-1)} \cup \{s^{(t)}\})} \leq N$}{
		% $A^{(t)} \assign nondom(A^{(t-1)} \cup \{s^{(t)}\})$
		% \tcp*{If the number of all nondominated solutions is less than or equal to the archive capacity after adding the new solution}
		% }
		%   \Else{
		%   $A^{(t)} \assign A^{(t-1)}$
		%	}
		\KwOut{$A^{(t)}$}
	\end{algorithm}

	This class comprises archivers holding some theoretical desirables
	but not being very useful in practice.
	It includes (1) archivers that do not diversify,
	e.g., the \emph{dominating} archiver (\Adom)
	found in the AR1 algorithm~\cite{RudAga2000cec}, which when full only accepts a new solution if it dominates an archived one,
	and (2) archivers that do not ``use'' their full capacity to store nondominated solutions,
	e.g., $\epsilon$-approx and $\epsilon$-Pareto~\cite{LTDZ2002b}
	whose archive size cannot be controlled if the parameter $\epsilon$ is pre-defined, and adapting $\epsilon$ to bound the maximum size tends to archive too few solutions~\cite{KnoCor2004lnems}, thus they are not limit-optimal.
	
	%and archivers that cannot reach some parts of the Pareto front,
	%e.g., MOEA/D-WS~\cite{ZhaLi07:moead}
	%where the weight sum scalarising function cannot find unsupported (non-convex) solutions of the Pareto front.
	%\MANUEL{I don't know the details of this MOEA/D, but why is this a problem? Also, how do we know its other properties? It is Pareto-subset  but not limit-optimal? How?}

	As the earliest archiver that holds theoretical desirables,
	\Adom is set to only accept a new solution if it dominates some solution in the archive
	(at the time when the archive is at full capacity).
	Algorithm~\ref{Alg:Dominating} gives the procedure of \Adom.
	%\MANUEL{Do we need this algorithm? In fact \Adom should be the application of rules 1 and 2 of Algorithm 1. It is not even necessary to remove all dominated solutions as long as we have space to fit the new one (it doesn't change its properties) }
	Since a solution in the archive cannot be removed unless a dominating one arrives,
	the archiver respects the property \textit{point-monotone}, which implies \emph{set-monotone}, and the property \emph{limit-optimal}, which implies all the other limit properties.
	Yet,
	\Adom does not respect the property \textit{Pareto-subset}.
	% \MANUEL*{Why are we showing the example about Pareto-subset? No archiver, except one, has this property, so it is not very useful. What is interesting in this section is that $\Adom$ does not diversify because is it efficiency preserving, thus that is why it is not practical despite having all theoretical properties except one, no? If you agree, I can come up with an example showing that it does not diversify.}
	% \MIQING*{The reason that I mentioned that \Adom does not meet the \emph{Pareto-subset} is that I suspect people might wrongly think \Adom has that property (at least myself did when having a first glimpse at it:-), whereas it seems to me that it is a bit straightforward to see \Adom does not diversify...but maybe my feeling is not accurate}
	% \MANUEL*{OK, let's leave it for now and see what the reviewers say.}
	An example is given in Figure~\ref{fig:dominating},
	where the archive rejects solution $c$
	but later accepts a solution (i.e., solution $e$) dominated by $c$.
	This is because when $c$ arrives,
	the archive is full,
	but after several timesteps when $e$ arrives,
	there is a slot available in the archive.
	%\MANUEL{I wonder if we can say anything about the many-at-a-time scenario.}
	%\MIQING{we may not need to add the many-at-a-time scenario since here it is focused on the one-at-a-time scenario.}
	%\manuel{}{A major drawback is that this archiver is negatively efficiency preserving, which means that it is efficiently preserving when full in the one-at-at-time scenario, and thus it does not \emph{diversify} and its archive will often shrink to a small region of the Pareto front. Lack of diversifying is a major drawback for practical use, which explains why almost all archivers reviewed here have the diversifying property (see Table~\ref{Table:classifcation}).}
	%\miqing{}{MIQING: Will update this after the meeting.}
	
	%\MANUEL{My suggestion would be the following:
	%  \begin{itemize}
	%  \item Class I: Archivers with not theoretical guarantees (remove IBEA?)
	%  \item Class II: Archivers with that never deteriorate but not useful
	%   in practice due to not diversifying (Adom) or not using their full capacity to store nondominated solutions (e-approx, e-Pareto).
	%  \item Class III: archivers with good theoretical properties but not limit-optimal (MOEA/D variants)
	%  \item Class IV: archivers that are limit-optimal and diversify.
	%  \end{itemize}}
	%\MIQING{Very good idea, will implement it. In addition, will add some indicator-based archivers into Class III like R2-based archiver or IGD$^+$-based archiver and state their differences from our Algorithm 1.}
	
	%\subsection{Class III: Archivers Capable of Preventing Solutions from Deterioration}

	Another type of archivers in this class are those that are not able to control its archive size: they either cannot respect a fixed maximum capacity $N$ or archive too few nondominated solutions.
	% \MIQING*{i am wondering if an archiver not respecting a fixed bound $N$ has no chance to be practice. One might construct an archiver whose archive size is not fixed, but close to $N$}\MANUEL*{By fixed \emph{bound}, I mean that the capacity must be $\leq N$ and $N$ is fixed. Also, what do you mean by ``practical''? }
	% \MIQING*{OK, no problem, I see what you mean.}
	In such archivers,
	the objective space is divided into regions and only one solution can be accepted within each region.
	Such a region can be a hyper-box~\cite{LTDZ2002b,DebMohMis2005epsilon,SchLauCoeDelTal2008,LuoBos2012elitist,HerSch2022archive} (defined by a parameter $\epsilon$ or its variant~\cite{JinWon2010adapt}).
	%\MANUEL{what is the ``variant'' of a box?}
	Since the superiority relation between nondominated solutions in such archivers is not changed
	(old one always being regarded better than new one if they are in the same region),
	they can hold many theoretical desirables
	including \textit{point-monotone}\footnote{In \cite{LopKnoLau2011emo},
		the archiver $\epsilon$-approx~\cite{LTDZ2002b} was said not to hold the \textit{point-monotone} property, which is not the case. A corrected version is available as a technical report~\cite{IRIDIA-2011-001}.},
	%\MANUEL{$\epsilon$-approx is not point monotone according to \cite{LopKnoLau2011emo}}
	%\MIQING{I think $\epsilon$-approx is point-monotone since one solution cannot enter the archive if it is dominated by some solution that previously was removed because there must be one solution that $\epsilon$-dominates the removed solution.}
	%\MANUEL{I think now you are right. This is my proof:
	%The $\epsilon$-approx archiver only removes points from the archive if they are dominated by an accepted point. To remove y1, we have to accept another point y3 that dominates it. If a later point y2 is dominated by the removed y1, then y2 is also dominated by y3, which is still in the archive. If y3 dominates y2 then y3 also $\epsilon$-dominates y2, so y2 will not be accepted into the archive. I believe this is sufficient to prove that the $\epsilon$-approx archiver is Monotone. We should say something here to make people aware that there is a mistake in the original paper. }
	\textit{set-monotone} and \textit{limit-stable}.
	%\MANUEL{This explanation is a bit too general. It is unclear how the ``age'' of a solution is sufficient to provide set-monotonicity and limit-stable. I think it may be worth discussing these properties.}
	%\MIQING{It seems a bit straightforward according to \cite{Laumanns2002}.}
	They can also be extended to hold the \textit{Pareto-subset} and \textit{limit-Pareto-subset} properties,
	provided that Pareto dominance is considered in the archiving update.
	%\MANUEL{What does this mean? Is this any of the rules in the algorithms already shown? reference?}
	However,
	they do not respect the property \textit{limit-optimal}
	as a nondominated solution may not be allowed to enter the archive even if the archive is not full, because there is already one solution in the same region.
	%\MANUEL{Isn't this proven in \cite{KnoCor2004lnems}? I need to double-check}
	
	\begin{algorithm}[tb]
		\caption{$\epsilon$-Pareto Archiver}
		\label{Alg:epsilon_Pareto}
		\small%
		\DontPrintSemicolon
		\SetAlgoLined
		\KwIn{$A^{(t-1)}$, $s^{(t)}$, $\epsilon$}
		\tcc{$box(a)$ is the box index vector of $a$ that discretises the space into boxes based on $\epsilon$,
			where $box_i(a) = \floor*{\frac{\log a_i}{\log (1+\epsilon)}}$ for $i=1,\dots,d$
			and $d$ is the number of objectives.}
		$D \assign \{a \in A^{(t-1)} \mid box(s^{(t)}) \prec box(a)\}$\;
		\uIf{$D \neq \emptyset$}
		{
			$A^{(t)} \assign A^{(t-1)} \cup \{s^{(t)}\} \setminus D$
		}
		\uElseIf{$\exists a\in A^{(t-1)}, box(a) = box(s^{(t)}) \wedge s^{(t)} \prec a$}
		{
			$A^{(t)} \assign A^{(t-1)} \cup \{s^{(t)}\} \setminus \{a\}$
		}
		\uElseIf{$\nexists a\in A^{(t-1)}, box(a) \preceq box(s^{(t)})$}
		{
			$A^{(t)} \assign A^{(t-1)} \cup \{s^{(t)}\}$
		}
		\Else{$A^{(t)} \assign A^{(t-1)}$}
		\KwOut{$A^{(t)}$}
	\end{algorithm}
	
	One of the most representative archivers in this class is $\epsilon$-Pareto~\cite{LTDZ2002b}.
	It is the only known archiver that guarantees storing a subset of the Pareto-optimal solutions seen so far (\emph{Pareto-subset}), while also capable of  diversifying.
	Algorithm~\ref{Alg:epsilon_Pareto} gives the procedure of $\epsilon$-Pareto.
	A new solution is accepted if it meets one of the three conditions:
	(1) it $box$-dominates some solution in the archive (lines~1--3),
	(2) it is located in the same box as another solution but dominates the latter (lines~4--5),
	or (3) there is no any other occupied box weakly dominating it (lines~6--7).
	
	A practical weakness of $\epsilon$-Pareto is that the size of the archive is not controllable
	(despite bounded~\cite{PapYan2000focs}),
	%\MANUEL{Why does this citation say in this context?}
	but determined by the interplay between the value of $\epsilon$,
	the optimisation problem and the search algorithm.
	Even if $\epsilon$ is adapted such that the archive size never surpasses a given capacity, the actual number of solutions at the end of the archiving process are often much fewer than this capacity~\cite{KnoCor2004lnems}. In some cases, it is not even possible to find a value of $\epsilon$ such that the size of the archive approximates the given capacity~\cite{LiYanLiuShe2013many}. Since an optimal set of size close to the user-specified maximum capacity is of primary interest, an archive of uncontrollable size is not very practical.

	\subsection{Class III: Archivers Holding Theoretical Properties and Being of Practical Use, but not Limit-Optimal}\label{sec:classIII}
	
	This class includes archivers that perform well in practice
	and also hold some theoretical properties (under certain conditions).
	They can further be divided into two types of methods,
	decomposition-based methods and indicator-based methods.
	
	The decomposition-based archiving methods,
	represented by MOEA/D~\cite{ZhaLi07:moead},
	decompose the original multi-objective problem
	into a number of single-objective subproblems
	through a set of weight vectors and a scalarising function.
	Algorithm~\ref{Alg:MOEAD} gives the procedure of the archiving algorithm based on MOEA/D, which follows a rather different template than other archivers. 
	The archiver in MOEA/D manipulates a set of $N$ weights (rather than $N$ solutions),
	and each weight is associated with a solution, which has the best value on that weight.
	Since a solution may associate with multiple weights,
	the total number of unique nondominated solutions in the archive may be significantly less than $N$,
	particularly on problems with irregular Pareto fronts~\cite{TriSriSanGho2016survey}.
	\begin{algorithm}[tb]
		\caption{Archiver based on MOEA/D's selection rules}
		\label{Alg:MOEAD}
		\small
		\DontPrintSemicolon
		\SetAlgoLined
		\KwIn{$A^{(t-1)}$, $s^{(t)}$, $W=\{w_1,w_2,\dots, w_N\}$ (set of weights), 
			$r$ (reference point)}
		\tcp{$A^{(t-1)}$ is the set of solutions associated with each weight,}
		\tcp{i.e., $A^{(t-1)} = \{a^{(t-1)}_{w_1}, \dots, a^{(t-1)}_{w_N}\}$.}
		$r \assign \text{update\_refpoint}(r, s^{(t)})$
		\tcp*{Update the reference point.}
		\ForEach{$w_i \in W$}{
			\uIf{$\textup{Scalarize}(s^{(t)},w_i,r) < \textup{Scalarize}(a^{(t-1)}_{w_i}, w_i,r)$}{
				\tcc{Replace the current solution associated with $w_i$ with the new solution $s^{(t)}$
					if $s^{(t)}$ has better scalar value on $w_i$.}
				$a^{(t)}_{w_i} \assign s^{(t)}$
			}\Else{
				$a^{(t)}_{w_i} \assign a^{(t-1)}_{w_i}$}}
		\KwOut{$A^{(t)}$, $r$}
	\end{algorithm}

	Depending on the scalarising function used (e.g., TCH or PBI~\cite{ZhaLi07:moead}),
	MOEA/D archivers may hold different theoretical properties.
	For example,
	MOEA/D-PBI holds none of the three desirable anytime properties,
	since the PBI scalarising function,
	which is an aggregation of the distance of a solution to the weight vector
	and the distance of its projection on the vector,
	may regard a dominated solution as better than the solution dominating it.
	In contrast,
	the Tchebycheff scalarising function, which is weakly in line with Pareto dominance, i.e., \mbox{$a \prec b \implies \textit{TCH}(a) \leq \textit{TCH}(b)$}, 
	makes MOEA/D-TCH hold the \textit{point-monotone} and \textit{set-monotone} properties conditionally, i.e., as long as the ideal point used for the calculation of the Tchebycheff function does not actually change during the archiving process.
	As for the limit properties,
	since the ideal point can always be settled in the limit sense,
	%\MANUEL{I'm not sure what this means. Do you mean that the point will eventually converge? That depends how the point is updated! It doesn't seem to converge in SMS-EMOA if there is deterioration. Also, isn't PBI limit-stable because the total distance is always reduced if the archive changes?}
	%\MIQING{In MOEA/D, they use the ideal point; with finite samples, then ideal point will be settled eventually.}
	the two archivers MOEA/D-PBI and MOEA/D-TCH are \textit{limit-stable}
	and the latter is also \textit{limit-Pareto-subset}.
	Yet, they are not \textit{limit-optimal}
	as a nondominated solution may not be able to enter the archive
	even if the archive is not full, because it cannot lead to a better scalarising function value on any weight vector.
	%\MANUEL{I understand that this can happen because the same solution is the optimum for 2 weight vectors, thus there are fewer solutions in the archive than weight vectors, however, the new nondominated solution is optimum for none, right?}
	%\MIQING{Yes!}
	
	%This may make the archive/population contain some dominated solutions.
	%In this regard,
	%the two-step archiving procedure which first considers Pareto dominance is better
	%(e.g., in NSGA-III),
	%at least it not allowing nondominated solutions to be removed earlier in the presence of dominated ones
	%(despite both respecting no theoretical desirables),
	%as shown empirically in~\cite{LiGroYanLiu2018multi}.

	Another type of archivers in this class are indicator-based archivers.
	They use a quality indicator to measure the quality of the whole archive,
	such that the quality contribution of a solution is  the difference of the indicator values between the archive with and without the solution.
	Most existing indicator-based archivers belong to this class,
	but not the one presented in the original IBEA~\cite{ZitKun2004ppsn}.
	This is because IBEA does not use an indicator to measure the quality of the whole set, but rather uses an indicator to define a measure
	(e.g., based on the $\epsilon$-indicator) between two solutions,
	hence not holding these theoretical properties.
	% \MANUEL*{Do we have a proof of this statement?}
	% \MIQING*{IBEA is not a set-based algorithm - for each solution, it calculates a measure (e.g., based on $\epsilon$ relation) between it and any of the remaining solutions in the population, and then aggregates all of these measures non-linearly. So, two nondominated solutions may have different rankings at different timesteps, depending on other solutions in the population; this is like NSGA-II.}
	% \MANUEL*{I'm not completely convinced of this but let's leave it for now}

	\begin{algorithm}[tb]
		\caption{Archiver based on a common indicator-based MOEA's selection rules}
		\label{Alg:IBEA}
		\small
		\SetAlgoLined
		\DontPrintSemicolon
		\KwIn{$A^{(t-1)}$, $s^{(t)}$}
		\tcp{Partition all the solutions into different nondominated fronts and identify the last front $F_l$.}
		$(F_1,F_2,\dots,F_l) \assign \text{nondom\_sorting}(A^{(t-1)} \cup s^{(t)})$\;
		\tcp{Find the solutions whose removal would minimize the indicator value  of the nondominated set $F_l$.}
		$D \assign \argmin_{a \in F_l} I(F_l \setminus \{a\})$\;
		\uIf {$s^{(t)} \in D$}{
			$A^{(t)} \assign A^{(t-1)}$
		}\Else{
			$a' \assign \text{sample}(D)$ \tcp*{Draw a solution randomly from $D$.}
			$A^{(t)} \assign  A^{(t-1)} \cup \{s^{(t)}\} \setminus \{a'\}$\;
		}
		\KwOut{$A^{(t)}$}
	\end{algorithm}
	% \begin{algorithm}[tb]
	%   \caption{Archiver based on a common indicator-based MOEA's selection rules (OLD)}
	%   \label{Alg:IBEA}
	%   \small
	%   \SetAlgoLined
	%   \DontPrintSemicolon
	%   \KwIn{$A^{(t-1)}$, $s^{(t)}$}
	%   $(F_1,F_2,\dots,F_l) \assign nondom\_sorting(A^{(t-1)} \cup s^{(t)})$
	%   \tcp*{Partition all the solutions into different nondominated fronts and identify the last front $F_l$}
	%   \ForEach{$a \in F_l$}{
	%   $fitness(a) \assign I(F_l\setminus \{a\})$\;
	%   \tcp*{Obtain the ``fitness'' of solution $a$ by calculating the indicator value of the nondominated set $F_l$ excluding $a$}
	%    }
	%   $D \assign \{a\in F_l \mid \argmin fitness(a)\}$\;
	%   \uIf {$s^{(t)} \in D$}
	%         {
	%           $A^{(t)} \assign A^{(t-1)}$
	%         }
	%         \Else
	%         {
	%           $a' \assign rand(D)$
	%           \tcp*{Draw a solution randomly from $D$}
	%           $A^{(t)} \assign  A^{(t-1)} \cup \{s^{(t)}\} \setminus \{a'\}$\;
	%         }
	%   \KwOut{$A^{(t)}$}
	% \end{algorithm}
	
	Indicator-based archivers, such as those found in SMS-EMOA~\cite{BeuNauEmm2007ejor}, MO-CMA-ES~\cite{IgeHanRot2007ec} (hypervolume-based) and R2-EMOA~\cite{BroWagTrau2015r2}, typically follow a two-step process:
	% Like Pareto-based ones,
	% archivers appeared in indicator-based MOEAs typically follow a two-step pattern:
	solutions are first ranked based on Pareto dominance and ties (nondominated solutions) are then broken based on a quality indicator
	(instead of the density metric in Pareto-based archivers).
	% That is,
	% firstly compare the dominance between solutions to identify the critical front and then for all the solutions in that front,
	% compare their contributions to the indicator.
	Algorithm~\ref{Alg:IBEA}
	% \MIQING*{Shall we comment out Algorithm 7 and replace it with Algorithm 6 throughout the paper}
	gives the procedure of indicator-based archivers from the literature. As can be seen from the algorithm,
	the archiver determines the set $D$ of solutions least contributing to the indicator value, that is, solutions whose removal from the considered nondominated set would lead to the best indicator value compared to the removal of any other solution (line~2).
	Afterwards,
	if the new solution belongs to the set $D$ of least-contributing solutions, then the archive is unchanged (line~4); otherwise, a solution from $D$ is removed randomly (lines~6--7).
	In some indicator-based MOEAs,
	the new solution and old ones are not distinguished, e.g., in R2-EMOA~\cite{BroWagTrau2015r2}.
	That is,
	one of the least-contributing solutions will be randomly selected to remove, whether it is the new solution or the old one.
	Removing an old solution when it has the same indicator value as the new one may cause a cyclic behavior
	(i.e., solutions may enter and exit the archive many times during the archiving process) for some indicators like R2~\cite{HanJas1998}. The cyclic behavior prevents convergence and, thus, any limit properties. Moreover, randomly removing solutions with the same value of a weakly Pareto compliant indicator prevents set-monotonicity even when the archiver considers the dominance relation between solutions first (like in most indicator-based archivers). %\MIQING*{I add this since without saying first considering dominance relation, people may think that $a$ can be directly edged out by $c$ in the following example, i.e., no need of the ``intermediate'' solution $b$}
	For example, let us consider three solutions $\{a,b,c\}$ with the same indicator value, and $a$ dominates $c$ but $b$ is mutually nondominated with $a$ and $c$. Imagine an archive of capacity one that at $t=1$ only contains $a$. At $t=2$, the archiver receives $b$ and (randomly) removes $a$. At $t=3$, the archive receives $c$ and (randomly) removes $b$. As result, the archive at $t=3$ is dominated by the archive at $t=1$.
	% \MANUEL*{I see why the cyclic behavior may prevent limit properties but why does it prevent monotonicity?}
	% \MIQING*{Randomly removing solutions with the same value of a weakly Pareto compliant indicator may prevent monotonicity. For example, let us say three solutions $a, b, c$ having the same indicator value, and $a$ dominates $c$ (this can happen since the indicator is weakly Pareto compliant) but $b$ is nondominated to the other two solutions. Let us say the archive capacity is one. We feed the archive with the order $a, b, c$, then the archive will first accept $a$, but when $b$ is fed, $a$ may be removed since they are nondominated to each other. When $c$ is fed, then $b$ may be removed since they are nondominated to each other as well. Now we can see at timestep $t=3$ the archive has solution $c$ whereas at timestep $t=1$ the archive has better solution $a$.}
	% \MANUEL*{Wouldn't the same effect happen when solutions are removed randomly just with a longer cycle? In fact, removing always the newest solution would avoid the cycle and guarantee limit-stable, no?}
	
	Indicator-based archivers have different properties depending on whether the indicator used is Pareto compliant, weakly Pareto compliant or neither.
	In any case, all of them hold the \emph{limit-stable} property since they always maximise/minimise the indicator value, as long as they remove the newest solution when two solutions have the same indicator contribution.
	%\MANUEL*{If they remove randomly, then this is not true, no?}
	%\MIQING*{Yes, if they remove randomly, none of them respect any limit property. So we may need to add the rule of always removing newest solution when they have the same indicator value.}\MANUEL*{Is my newly added text clear enough? Do we need to modify Algorithm 6?}
	%\MIQING*{yes, it is enough. I don't think we need to modify Algorithm 6 since it is a common practice in indicator-based archivers that solutions with the same value will be removed randomly.}
	If the indicator used is not Pareto compliant,
	then the archiver will not have any Pareto dominance-related properties like \emph{point-monotone}, \emph{set-monotone} and \emph{limit-Pareto-subset}.
	Representative example are IGD-based~\cite{CoeSie2004igd} archivers,
	e.g., \cite{SunYenYi2019igd,TiaCheZha2017indicator}.
	If the indicator is Pareto compliant like hypervolume,
	the archiver may hold most of the theoretical properties including \emph{limit-optimal}, thus we will discuss them in the next section.
	
	If the indicator is weakly Pareto compliant,
	the archiver (based on Algorithm~\ref{Alg:IBEA}) may hold many properties but not \emph{limit-optimal} since the indicator may not be able to distinguish between solution sets subject to the \Better-relation (Definition~\ref{Def:weak_compliance}).
	Such an example is \ARtwo~\cite{BroWagTrau2015r2} in Table~\ref{Table:classifcation}.
	When the ideal point is unchanged,
	\ARtwo holds the anytime properties \emph{point-monotone} and \emph{set-monotone}.
	Since the ideal point can always be settled in the limit sense,
	\ARtwo always holds the limit property \emph{limit-Pareto-subset}.
	
	Lastly,
	it is worth noting that despite using a weakly Pareto compliant indicator,
	the archiver proposed in Algorithm~\ref{Alg:Weak_compliance} is \emph{limit-optimal}, thus it belongs to the class discussed next.
	% \MANUEL*{Should we include Algorithm~\ref{Alg:Weak_compliance} in Table I? How much we can say about its properties? Can we implement a concrete example, maybe using IGD+ or $\epsilon$-multiplicative?}
	% \MIQING*{I am not quite sure here; it depends on the indicator working with it. For example, when it works IGD$^+$, it respects all the properties except point-monotone; when it works with $R2$, it respects three limit properties, but ``$-$'' respects point-monotone and set-monotone since the idea point of the evolutionary population can change.}
	The essential difference between Algorithm~\ref{Alg:Weak_compliance} and existing indicator-based archivers (presented in Algorithm~\ref{Alg:IBEA}) is that Algorithm~\ref{Alg:Weak_compliance} does not accept duplicate solutions in its archive,
	which makes the archive always ``tight'' and have room for accommodating different nondominated solutions.
	In contrast,
	in Algorithm~\ref{Alg:IBEA} the archiver allows duplicate solutions.
	A duplicate solution may not be able to be replaced by a new nondominated solution since adding that nondominated solution into the archive may not necessarily lead to a better indicator value of the archive for a weakly Pareto compliant indicator.
	
	%\MIQING{Will add some indicator-based archivers here, including both based on weakly Pareto compliant indicators like the IGD$^+$-based or $R2$-based archivers and based on indicators that are not (weakly) compliant with Pareto dominance like the IGD-based archivers. Need to note that IGD$^+$ is weakly Pareto compliant, but existing IGD$^+$-based archivers don't have the property of \emph{limit-optimal}. They are different from our Algorithm 1 (e.g., they allow duplicate nondominated solutions in the archive).}

	\subsection{Class IV: Archivers Holding the Limit-Optimal Property and also Being of Practical Use}\label{sec:classIV}
	
	\newcommand{\boxb}[1]{\ensuremath{box^{(#1)}}}
	\newcommand{\barb}{\ensuremath{\bar{b}}}

	Archivers in this class hold the critical property \textit{limit-optimal} as well as being capable of diversifying their solutions.
	%\manuel{\textit{set-monotone} and \textit{limit-optimal} and in the meantime are capable of diversifying its solutions}{\emph{limit-optimal}, which implies \emph{set-monotone}, and \emph{diversifying}}.
	%\manuel{There do not exist many archivers having these desirables currently in the area; two representatives are}{To the best of our knowledge, only two archivers of this class have been proposed so far:}
	There do not exist many known archivers having these two desirable properties.
	Three representatives are
	\AHV~\cite{KnoCor2003tec}, SMS-EMOA~\cite{BeuNauEmm2007ejor}, and MGA~\cite{LauZen2011ejor},
	though one may expect more to emerge in the future
	since an archiver based on a weakly Pareto compliant indicator can also hold these desirables (if designed properly),
	as we proved previously.
	%\MANUEL{Assuming one can find a way to compute the indicator value without knowing the true Pareto front and not break its assumptions during the archiving process. (I believe this is the major challenge when devising such archivers)}
	%\MIQING{I suspect that many weakly Pareto compliant indicators can easily be used in Algorithm 1 (without a need of knowing the true Pareto front and not breaking its assumptions) like R2. The resulting archiver then has the \emph{limit-optimal} property.}
	
	%%%% Fig. 3 %%%%
	\begin{figure}[tbp]
		\centering
		\includegraphics[scale=0.25]{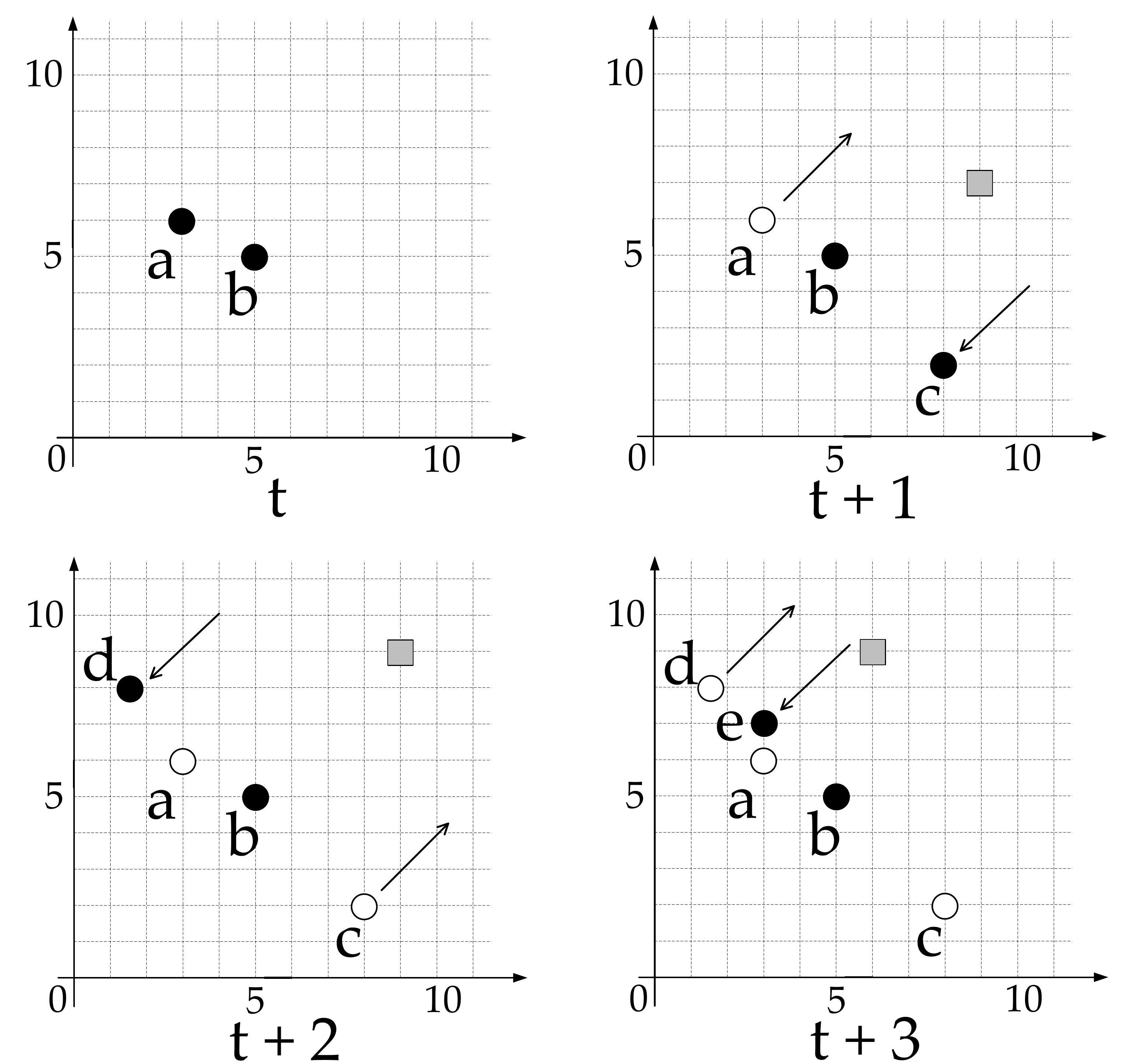}
		\caption{Illustration of the deterioration of hypervolume-based archiving
			with adaptive reference point. In this example, the reference point is the worst objective values in
			all solutions at the current timestep increased by one, as in
			SMS-EMOA~\cite{BeuNauEmm2007ejor}) and the capacity of the archive is \mbox{$N=2$}.
			Black circles denote solutions in the archive, and hollow circles denote
			solutions removed from the archive.  The grey square denotes the reference
			point used in the hypervolume calculation.
			% here the nadir point plus $(1,1)$ is used
			% (which is adopted in SMS-EMOA).
			At the timestep \mbox{$t+1$}, solution~$c$ arrives and $a$ is removed since its
			hypervolume contribution (HVC) is the lowest:
			\mbox{$\text{HVC}(a)=\text{HV}\{a, b,c\} - \text{HV}\{b,c\} = 2$},
			\mbox{$\text{HVC}(b)= \text{HV}\{a,b,c\}$} ${} - \text{HV}\{a,c\}=3$,
			$\text{HVC}(c)=\text{HV}\{a, b,c\} - \text{HV}\{a,b\} = 3$.
			%\manuel{}{This could be explained only looking at maximising the HV of the possible sets, because $HV\{a, b,c\}$ is constant.}
			At the timestep \mbox{$t+2$}, solution $d$ arrives and solution $c$ is removed
			since it now has the lowest HV contribution: $HVC(b)=9$, $HVC(c)=3$,
			$HVC(d)=3.5$.  At the timestep \mbox{$t+3$}, solution $e$ arrives and solution
			$d$ is removed since it now has the lowest HV contribution: \mbox{$HVC(b)=2$},
			\mbox{$HVC(d)=1.5$}, \mbox{$HVC(e)=2$}.  However, $e$ is dominated by $a$, thus the
			archive at the timestep \mbox{$t+3$} is dominated by the one at the timestep $t$, i.e.,
			the archive shows \emph{set-deterioration}.}\label{fig:SMS}
	\end{figure}
	
	The first two archivers \AHV and SMS-EMOA are both based on the hypervolume indicator, which is Pareto compliant.
	The difference between them is that in \AHV the reference point is fixed during the archiving process,
	whereas in SMS-EMOA the reference point is adaptive.
	%\MANUEL*{Here we say that this is the main difference, thus, do we also say in Table I  that adapts the reference point? }
	%\MIQING*{Sorry - do you want to ask if we considered the same SMS-EMOA in the table and text? SMS-EMOA always adapt its reference point in the hypervolume calculation according to the current population.}\MANUEL*{I mean that Table I says that SMS-EMOA is set-monotone conditionally, but the condition is that the reference point does not change, but SMS-EMOA does adapt the reference point, so it cannot ensure set-monotone.}
	%\MIQING*{yes, this is the case for problems with three or more objectives. For the bi-objective case, SMS-EMOA always preserves the two extreme solutions of the nondominated front (whatever their hypervolume contribution), so the reference point (which is determined by them) may not change very often.}
	%\MANUEL*{My comment is about SMS-EMOA not holding set-monotone, not even conditionally, because the condition is equivalent to "if we are lucky and the reference point does not change". But let's leave it as it is and see what reviewers say.}
	
	The archiver \AHV~\cite{KnoCor2003tec,KnoCorFle2003} is arguably the earliest archiving algorithm using the hypervolume indicator.
	It presets a reference point for the hypervolume calculation
	and eliminates the solution with the least hypervolume contribution.
	The greedy nature of its update in the one-at-a-time case means that the resulting archive cannot maximise the hypervolume in the anytime scenario~\cite{BrigFri2009foga},
	so it does not meet \emph{Pareto-subset} nor \emph{point-monotone}.
	In the limit case, however, the archive will converge to a set of maximum hypervolume among all sets of size $N$,
	which implies that all its elements will be Pareto-optimal~\cite{BriFri2010gecco}.
	Nevertheless,
	its use in practice is not without challenges since setting an appropriate reference point a priori may require problem-specific knowledge.
	%\MANUEL{I think it is enough if the new reference point is always strictly dominated by the previous one, no? \cite{AugBadBroZit2009hv}}
	%\MIQING{Even it is, but updating the reference point in this way may lead to the reference point is far away from the true Pareto front, particularly when the initial population is far away from the problem's Pareto front.}
	%\MIQING{will amend the text a bit later.}
	%and
	%(2) the time complexity of computing the hypervolume increases exponentially with the number of objectives~\cite{GueFonPaq2021hv}.
	%\MANUEL{however, fast approximation approaches may be sufficient in practice~\cite{BriFri2009emo}}
	%\MIQING{using an approximation may make the theoretical property invalid.}
	
	In SMS-EMOA, like many well-established hypervolume-based archivers,
	the reference point is adapted,
	usually set to be a slightly worse vector than the nadir point of the nondominated set obtained.
	However,
	changing the reference point may lead to the archiver \emph{set-deteriorating}.
	%Only if the reference point used in the archiving process is kept fixed,
	%\MANUEL{Does it need to be fixed or is it sufficient if a later reference point is always dominated by a former one? That is, the HV can grow when changing the reference point but not decrease.}
	%the archivers can respect the anytime-property property \textit{set-monotone}.
	%Contrary to what is claimed in the original publication~\cite{BeuNauEmm2007ejor},
	%proofs of convergence of hypervolume-based archivers that kept the reference point fixed do not hold for the original version of SMS-EMOA.
	Figure~\ref{fig:SMS} gives an example of the deterioration of
	the hypervolume-based archiving with an unfixed reference point.
	As the figure shows,
	the archive at the timestep $t+3$ is dominated by its past version at timestep $t$,
	in which the reference point is determined adaptively by solutions in the archive and the new arrival. Such a HV-based archiver holds the three limit properties only if the nadir point settles down in the limit, which may or may not happen depending on the problem.
	% \MANUEL*{conditional on what? on the nadir point settling down? Also Table I says that set-monotone is also conditional for SMS-EMOA, but it should not be, no? }
	% \MIQING*{SMS-EMOA holds the three limit properties on the condition that the reference point can settle down in the limit sense, which may happen or may not happen, dependent largely on the optimisation problem.}
	%\MANUEL{I think this explanation is very good and illuminating. But I disagree that we should include SMS-EMOA in this class because if one changes the reference point in this way, then one is not really using the hypervolume metric but something else that doesn't have the same properties. Also, it is very easy to fix SMS-EMOA to not behave like this: Do not change the reference point.}
	%\MIQING{Yes, I agree that it would be more accurate to put SMS-EMOA into the last class. Will do later. One thing worth mentioning is that this is in fact how SMS-EMOA updates its reference point for problems with more than two objectives.}
	
	\begin{algorithm}[tb]
		\let\oldnl\nl% Store \nl in \oldnl
		\newcommand{\nonl}{\renewcommand{\nl}{\let\nl\oldnl}} % Remove line number for one line
		\caption{Multi-level Grid Archiver (MGA)}
		\label{Alg:MGA}
		\small
		\DontPrintSemicolon
		\SetAlgoLined
		\KwIn{$A^{(t-1)}$, $s^{(t)}$}
		\uIf{$\exists a \in A^{(t-1)}, a \prec s^{(t)}$}
		{
			$A^{(t)} \assign A^{(t-1)}$
		}
		\Else
		{
			$A' \assign \min(A^{(t-1)} \cup \{s^{(t)}\}, \prec)$\;
			\uIf{$\abs{A'} \leq N$}{
				$A^{(t)} \assign A'$
			}
			\Else
			{
				\tcp{$\barb$ is the largest box index possible in $A'$.}
				$\barb \assign \floor*{\log_2(\max_{a\in A'}\max_{i\in\{1,\dots,d\}}\abs{a_i})}+1$\;
				$\mathcal{Z} \assign \{b\in\mathbb{Z}, b\leq \barb \mid \exists a, a'\in A,$\;
				\nonl\hspace{9em}$\boxb{b}(a) \preceq \boxb{b}(a') \land a \neq a'\}$\;
				\tcc{$\boxb{b}(a)$ is the box index vector of $a$ at the coars\-eness level $b\in\mathbb{Z}$,
					i.e., $\boxb{b}(a)_i = \floor{a_i \cdot 2^{-b}}$, for $i=1,\dots,d$ and $d$ is the number of objectives.}
				\uIf{$\mathcal{Z} = \emptyset$}
				{
					$A^{(t)} \assign A^{(t-1)}$
				}
				\Else 
				{
					$\beta \assign \min\mathcal{Z}$\;
					$D \assign \{a\in A' \mid \exists a'\in A',$\;
					\nonl\hspace{7em}$\boxb{\beta}(a') \preceq \boxb{\beta}(a) \land a' \neq a\}$\;
					\uIf {$s^{(t)} \in D$}
					{
						$A^{(t)} \assign A^{(t-1)}$
					}
					\Else(\tcp*[h]{Draw a solution randomly from $D$.})
					{
						$a \assign \text{sample}(D)$ \;
						$A^{(t)} \assign A' \setminus \{a\}$
					}
				}
			}
		}
		\KwOut{$A^{(t)}$}
	\end{algorithm}
	
	MGA or multi-level grid archiver~\cite{LauZen2011ejor} can be seen as
	an improved version of the $\epsilon$-Pareto archiver~\cite{LTDZ2002b}.
	%\MANUEL*{Where does the definition of $\barb$ come from? The implementation that I have only looks at $A^{(t)} \cup \{s^{(t)}\}$. }
	% \MIQING*{It seems to me that $\barb$ is the largest box index in the entire feasible objective space $Y$, which is an input of MGA. Not quite sure if it is accurate after changing $\barb$ to the largest box index possible in $A'$.}
	% \MANUEL*{The definition of box does not depend on $\barb$, so the boxes themselves don't change. The only thing that changes is the elements in $Z$. We also know that $\boxb{b}(x) = 0$, when $b \geq \floor{\log_2(x)} + 1$. Thus, we know that if $b' = \floor{\log_2(\max_{a\in A'}\max_{i\in\{1,\dots,d\}}\abs{a_i})} + 1$, then $\boxb{b''}(a_i) = 0$  $\forall a \in A'$, $i \in \{1,\dots,d\}$ and $\forall b'' \geq b'$. This means that considering a $\barb$ larger than $b'$ is not useful it will not change any decision (note that line 11 does the same as line 16 in Algorithm 9). In fact, in the implementation that I have, lines 9-21 are a while-loop that starts at $b=1$ and checks every $b \leq \barb$.}
	% \MIQING*{I see!}
	It compares solutions using
	a hierarchy of boxes of different coarseness over the space.
	Algorithm~\ref{Alg:MGA} gives the procedure of MGA.
	As can be seen in the algorithm,
	the standard Pareto dominance relation is first used to compare solutions (lines~1--6),
	and, when the number of nondominated solutions exceeds the capacity of the archive,
	solutions are compared using \emph{box-dominance}, i.e., applying the Pareto dominance relation to their box indices  (lines~7--23).
	The size of the boxes is not set by a parameter, unlike the $\epsilon$-approx and $\epsilon$-Pareto archivers,
	but determined by the smallest coarseness level $\beta$
	that leads to at least one solution being weakly box-dominated (lines~7--14).
	If the new solution $s^{(t)}$ belongs to such weakly dominated boxes,
	then it is rejected (lines~15--16),
	otherwise an arbitrary solution from such boxes is eliminated (lines~17--20).
	
	MGA does not respect the property \textit{point-monotone}, as shown by~\cite{LopKnoLau2011emo},
	since
	% \manuel{it does not use the box to preclude the entry of new solutions in the box
	%   and}{}\MANUEL*{the second reason is sufficient to prevent point-monotonicity} \MIQING*{that makes sense.}
	any nondominated solution can enter the archive if the archive is not full, even if this solution was dominated by a solution previously removed.
	However, MGA cannot \emph{set-deteriorate} because it implicitly optimises a Pareto compliant indicator such that accepting a new solution into the archive,
	possibly replacing an existing one,
	will always lead to a better value of the indicator~\cite{LauZen2011ejor}.
	It also eventually converges to an archive that minimises this indicator value,
	i.e., an optimal approximation of bounded size,
	hence, it is \emph{limit-optimal}, which implies all other limit properties.
	%Yet, MGA respects the property \textit{set-monotone} and the three limit properties
	%since one can create a Pareto compliant indicator for it
	%such that the archive only accepts a new solution that leads a better value of the indicator,
	%as proved in~\cite{LauZen2011ejor}.
	
	%\cite{LauZen2011ejor}.
	% \MANUEL*{Does the paper give this time complexity? I don't see this. I think it was only mentioned in \cite{LopKnoLau2011emo}}
	% \MIQING*{I thought it was given in the original paper, but it seems I was not correct!}
	
	Despite the above desirable properties, a recent study has shown that
	MGA is unlikely to preserve  boundary solutions~\cite{Li2021telo}.
	%\MANUEL{Isn't this true for any archiver that does not explicitly preserve boundary solutions? Hypervolume, MOEA/D, crowding distance, SPEA2, etc. I'm not aware of any archiver that preserves boundary solutions unless explicitly keeping them.}
	%\MIQING{since MGA rejects solutions in weakly-dominated boxes, boundary solutions are unlikely to be preserved. This happens to archivers which use the box dominance relation.}
	This is a common problem of all archivers using box- or $\epsilon$-based dominance,
	such as $\epsilon$-MOEA~\cite{DebMohMis2005epsilon} and GrEA~\cite{YanLiLiuZhe2013}.
	In addition,
	the archive maintained by MGA is not uniformly distributed along the Pareto front~\cite{Li2021telo}.
	%\MANUEL*{What does it mean ``uniformity is not very good'', do you mean that solutions are not uniformly distributed along the Pareto front?}
	%\MIQING*{Yes, solutions are not distributed very uniformly along the Pareto front.}
	This occurrence can be attributed to the facts that (1) MGA picks one solution randomly to remove
	when there are multiple  solutions at the $\beta$ level,
	%\MANUEL{If one is willing to sacrifice some speed, one could always choose the one that maximises the hypervolume}
	and (2) the new solution is not allowed to enter the archive
	if it is at the same level as some of the solutions in the archive
	(lines~15--16 in Algorithm~\ref{Alg:MGA}).
	% \MIQING*{I suspect this should be at the same level rather than in the same box. The same box implies the same level, but the same level doesn't imply that solutions need to be in the same box. It seems to me that If the new solution is at the $\beta$ level, then it is not allowed to enter the archive.}
	%\MANUEL{Doesn't it compare dominance as well within boxes?}
	%\MIQING{It does, but apart from it nothing, so even the new solution is much better on most of objectives, but slightly worse on one objective, the new solution is not allowed to enter the archive.}
	
	\subsection{Computational Complexity of Archivers}

		Archivers proposed as part of an MOEA (such as NSGA-II, MOEA/D, and SMS-EMOA) often dominate the computational complexity of the MOEAs (when ignoring the cost of solution evaluations), thus their complexity is the same as their corresponding MOEA and can be found in the original papers~\citep{Deb02nsga2,ZitLauThie2002spea2,BeuNauEmm2007ejor,ZhaLi07:moead,Jen03}.
		For other archivers,
		the computational complexity per solution update is as follows.
		\Adom, $\epsilon$-approx and $\epsilon$-Pareto require $O(mN)$, where $N$ is the archive capacity and $m$ is the number of objectives. 
		The computational complexity of the hypervolume archivers (e.g., \AHV ) strongly depends on the algorithm employed and the number of objectives~\citep{GueFonPaq2021hv}. 
		% The best algorithms available are reported by \citet{GueFonPaq2021hv}.
		Despite having the same properties as hypervolume-based archivers,
		MGA has less computational cost in general;
		its time complexity is $O(mNL)$,
		where $L$ is the length of the binary encoded input~\cite{LopKnoLau2011emo}.

	\section{Important Issues in Archiving}\label{sec:issues}
	
	In this section,
	we discuss several important issues of archiving,
	including its performance, various attributes as well as connection with research topics in other fields.

	\subsection{Theoretical Desirables vs Practical Use}
	
	It is certainly helpful that archivers hold desirable theoretical properties, but it is more important that archivers are of practical use.
	% It is certainly helpful for archivers to hold theoretical desirables,
	% but it is more important for them to be of practical use.
	In particular,
	avoiding convergence to a small region of the Pareto front is a critical, practical desirable
	(i.e., the first practical desirable in Table~\ref{Table:classifcation}).
	%\MANUEL{This distinction between theoretical and practical is not clear to me. Diversifying can be explicitly defined and mathematically proved, so it is as much theoretical as any other property, e.g., an archiver may theoretically not diversify and still work well in practice (if the archive is seeded with widespread boundary solutions)}
	%\MIQING{will go back to this after the meeting.}
	Archivers that fail to diversify are not very useful in practice, since they produce a poor approximation of the actual Pareto front, even if they may hold most theoretical desirables, e.g., \Adom~\cite{RudAga2000cec}.
	Having an archive of controllable size, that is, with a user-specified maximum capacity and that stores as many nondominated solutions as can fit in that capacity, is also an important practical desirable,
	since one may not want to end up with a population during the search that is a too small or too big.
	%\manuel{}{Moreover, the maximum capacity may be due to memory or storage limits that must be respected.}
	%\MANUEL{How one can end with a population larger than the maximum size of the archive?}
	%\MIQING{For example, use the $\epsilon$-Pareto archiver as the search population like in \cite{Laumanns2002}}
	The lack of this property may explain why archivers in Class II, e.g., $\epsilon$-Pareto~\cite{LTDZ2002b}, are not widely used in practice.
	In addition,
	archivers may need to compromise practical desirables in order to meet theoretical ones.
	For example,
	in contrast to SMS-EMOA~\cite{BeuNauEmm2007ejor} which adaptively sets the reference point according to the input sequence,
	the archiver \AHV~\cite{KnoCorFle2003} requires setting a fixed reference point  a priori, with may lead to poor diversification if the reference point is either too far or too close to the nadir point of the Pareto front~\cite{IshImaSet2018refpoint}.
	%\MANUEL{I have to read this paper in detail: Does it really say that adapting in the way that SMS-EMOA suggests is better than a fixed value.}
	%\MANUEL{Is this when used as an external archiver or as a population. I am afraid that we may be confounding the role as a storage of the elite set (best-so-far result) and as a source of search directions/points. As the latter, being limit-stable or set-monotone or even limit-optimal may be undesirable, since it may lead to the search stagnating.}
	In short,
	we cannot say that an archiver without any theoretical property (i.e., those in Class I) performs worse
	than those with some of them (Classes II--IV) in practice.
	%\MANUEL{We can say that they perform worse in terms of storing the best-so-far set found and being able to find an optimal set of bounded size with sufficient iterations, no?}

	Yet,
	equally,
	we would never say that an archiver without any theoretical property is the best in practice,
	even when used to manage  the population in an MOEA.
	%\MANUEL{I think we need to define ``best'' to be able to make such claims. It is difficult to compare archivers when they are embedded in different MOEAs}
	Indeed, a lack of theoretical properties may harm the search progress,
	as reported on various synthetic and practical scenarios~\cite{LTDZ2002b,FieEveSing2003tec,AguZapLieVer2016many,Fie2017gecco,CheLiYao2019}.
	%\MANUEL{what does it mean ``hold back''? Lose the best set found? converge too fast? quality depends on when you stop the search?}
	The \emph{set-monotone} property, in particular, prevents the oscillation of the archive's performance~\cite{LiYao2019emo}
	and leads to the eventual convergence of the archive, which can be used as a stopping condition of an MOEA.
	%\MANUEL{This is not strictly true. Set-monotone does not imply any limit property. The archive may fluctuate between mutually nondominated archives without converging to any particular one  (if this is not the case, it may be worth proving it, because it doesn't sound obvious to me).}
	%\MIQING{Good point! I wonder if we could come up with an archiver which has the set-monotone property and also allows the fluctuation between mutually nondominated archives as the example is about an archive rather than an archiver.}
	%\MANUEL{For example, an archiver than, whenever full, removes the oldest nondominated point and replaces it with the new one.}
	%\MIQING{But this one does not seem to guarantee the set-monotone property for any sequence?}
	
	%Finally, it is worth mentioning that one may not be able to say that archivers in the Class III are certainly better in practice (e.g., have better values on an indicator) than those in Class IV, though most of the archivers in the former are popular in the area whereas some in the latter are not (e.g., \AHV and MGA~\cite{LauZen2011ejor}).
	%Archivers with the \textit{limit-optimal} property
	%%\MANUEL{Limit-optimal implies set-monotone (does it not?). I believe the interesting combination is limit-optimal and diversifying, which combines convergence, quality and spread.}
	%under any condition are surely preferable provided that they do not sacrifice much practical performance (e.g., convergence and diversity), such as MGA.
	%%\MANUEL{What does it mean practical performance?}

	\subsection{What Is An Ideal Archiver?}\label{sec:ideal}
	
	Apart from holding theoretical desirables,
	one may ask what an ideal archiver is in practice.
	In general,
	an archiver can be called ``ideal''
	if it can maintain a representative subset of all Pareto-optimal solutions of any sequence at any time.
	There are three major attributes with respect to sequences that can affect the performance of an archiver: the dimensionality of solution vectors in the sequence,
	the shape of the Pareto-optimal solutions of the sequence,
	and the order of the solutions in the sequence.
	As such,
	an ideal archiver needs to work well on various sequences
	and be (almost) unaffected by the dimensionality, shape and order of solutions.
	
	In this sense,
	existing archivers unfortunately are far from being ideal.
	Pareto-based archivers,
	which use Pareto dominance and a density estimator as the selection rules,
	fail to scale up with the number of objectives~\cite{WagBeuNau2007:many,PurFle2007tec}.
	In contrast,
	some modifications, which aim to make Pareto-based archivers work in a high-dimensional objective space,
	may be detrimental to their performance in a low-dimensional space.
	For example,
	shift-based density estimation (SDE)~\cite{LiYanLiu2014shift},
	which enables Pareto-based archivers to work well in many-objective optimisation,
	may affect their ability to maintain the boundary solutions
	when dealing with bi- or tri-objective problems~\cite{LiTanLiYao2016stochastic,Liu2020,Xue2022}.
	%\MANUEL{Do we need to cite an arxiv paper if we already have 3 examples of this?}
	
	Indicator-based and decomposition-based archivers are more effective in dealing with increasing number of objectives.
	%as a result of making nondominated solutions comparable by an indicator or using the weight vectors, respectively.
	However,
	such an approach makes them sensitive to the shape of the Pareto front of the sequence.
	It is known that decomposition-based methods (e.g., MOEA/D)
	may not be able to maintain a well-distributed archive for irregular Pareto front shapes~\cite{IshSetMas2017shape}.
	Indicator-based methods may also struggle on some shapes,
	depending on the characteristics of their indicators.
	For example,
	SMS-EMOA has been found to be less effective on problems with inverted simplex-like Pareto front shapes~\cite{Shang2021}, with highly degenerate Pareto fronts~\cite{LiGroYanLiu2018multi},
	or with many dominance resistance solutions~\cite{LiGroYanLiu2018multi}.\footnote{Dominance resistant solutions
		are those with a extremely poor value in one
		objective but with (near) optimal values in the others~\cite{IkeKitShi2001cec}.}
	%\MANUEL{do we know if this is the case even for a fixed reference point? This behavior may be due to the violation of the hypervolume indicator assumptions.}

	The effect of the order of solutions fed to the archive has been rarely studied empirically.
	%despite in theory any order existing.
	%\MANUEL{do you mean studied empirically? because the one-pass behavior does not assume an order, so if an archiver has a certain property, it must have it for any sequence order.}
	A recent study has shown that the order matters in the sense that
	different sequences of the same set of solutions can produce very different archiving results~\cite{Li2021telo}.
	Archivers,
	not only from Class I (e.g., NSGA-II and NSGA-III) but also from Class III (e.g., MOEA/D-TCH) and Class IV (e.g., SMS-EMOA and MGA),
	may struggle to maintain a well-distributed archive when facing ``interesting'' sequences of solutions,
	even on low-dimensional problems with regular Pareto front shapes (i.e., simplex shapes)~\cite{Li2021telo}.
	
	\subsection{Batch Size}\label{sec:size}
	
	Batch size in the archiving process refers to the number of solutions fed to the archive at one step.
	It is often set to either one (e.g., in most theoretical studies~\cite{KnoCor2003tec,LopKnoLau2011emo} and some MOEAs~\cite{BeuNauEmm2007ejor,ZhaLi07:moead})
	or to the archive/population size (e.g., in many MOEAs~\cite{Deb02nsga2,ZitKun2004ppsn,DebJain2014:nsga3-part1}).
	In the context of EC,
	the former is called steady-state evolution mode (i.e., $\mu+1$)
	and the latter is called generational evolution mode (i.e., $\mu+\mu$), where $\mu$ denotes the archive/population size.
	It is worth noting that the evolution mode of an optimiser is orthogonal to the batch-size of the archiver: a ($\mu+1$) archiver can always handle a ($\mu+\mu$) evolution mode by processing the $\mu$ offspring one at time, but doing so the ($\mu+1$) archiver will not gain any of the properties of a  $(\mu+\mu)$ archiver. Similarly, a ($\mu+\mu$) archiver can always be combined with a ($\mu+1$) evolution mode, but at the cost of losing all the properties that result from updating the archive with many solutions at a time.
	
	For a given sequence of solutions,
	the question of which size is better may be trivial
	%\MANUEL{may be or is?}
	%\MIQING{in an ongoing study, we have surprisingly found the $\mu+1$ mode generally better than the $\mu+\mu$ mode empirically.}
	since a bigger batch size always gives the archiver more knowledge about future input,
	so that the archiver can make more informed decision.
	%and be less likely to regret.
	% Some theoretical studies showed that the ($\mu + 1$) archiving cannot be ``effective'' in terms of the hypervolume indicator
	%\MANUEL{what does effective mean here?}
	% while the ($\mu + \mu$) archiving can~\cite{ZitThiBad2010tec,BriFri2014convergence}.
	%
	Yet, when being used in the process of generating offspring solutions in MOEAs,
	a  ($\mu + 1$) archiver  can be more suitable in some cases
	since an instant update of the source archive forming the mating pool
	may be helpful in generating better offspring,
	particularly when evaluating solutions  is expensive.
	
	The batch size has important effects of the properties of archivers.
	\Citet{ZitThiBad2010tec} proved that there is no $(\mu+1)$ archiver that never decreases the hypervolume of the archive, which implies \emph{set-monotone} (Property~\ref{prop:set_monotone}),
	%\MANUEL{All hypervolume-nondecreasing archivers are set-monotone but there may be set-monotone archivers that are not hv-nondecreasing, is that right?}
	and ends up with an archive of maximum hypervolume when given the best possible input sequence starting from any sub-optimal archive.
	\citet[Th.~2]{BriFri2014convergence}
	confirmed this result
	and extended it to $(\mu+\lambda)$ with $\lambda < \mu$ \cite[Th.~5]{BriFri2014convergence}. There are $(\mu+\mu)$ archivers, however, that are able to reach the maximum hypervolume when given the best possible input sequence \cite[Th. 3.4]{ZitThiBad2010tec}, \cite[Th. 3]{BriFri2014convergence}.
	
	One of the main practical reasons for preferring a $(\mu+1)$ archiver instead
	of a $(\mu+\mu)$ one is the additional computational cost of choosing the
	optimal subset from all $\binom{2\mu}{\mu}$  subsets, however,
	there are efficient algorithms for both hypervolume and
	$\epsilon$-indicators on bi-objective
	problems~\cite{BriFriKli2014subset} and further improvements in higher dimensions are possible~\cite{GueFonPaq2021hv}.

	\subsection{Unbounded Archive}\label{sec:unbounded_archive}
	
	Archivers discussed so far maintain an archive of bounded capacity, that is,  when the number of nondominated solutions exceeds the capacity of the archive,
	the archiver needs to remove a solution.
	Since the archiver does not know the future input,
	it is an online algorithm whose decisions cannot be guaranteed to be optimal~\cite{KnoCor2003tec}, thus no archiver can guarantee an optimal approximation of bounded size for any finite sequence~\cite{KnoCor2004lnems}.
	Furthermore,
	most archivers will deliver a final archive that consists of many solutions that are dominated by solutions removed previously (\emph{point-deteriorate})~\cite{LiYao2019emo}.
	
	An unbounded archive that stores all nondominated solutions ever generated does not have the same limitations as an online archiver of bounded size.
	If at each timestep an archiver selects  a small subset of the Pareto optimal solutions (e.g., to present to the decision maker) from an unbounded archive of the input sequence seen so far, then the selected subset would never point-deteriorate. % \MANUEL*{I changed this because the claim is stronger: Even if the archiver has to present a bounded size subset to the DM at each timestep (so it doesn't know the future sequence), the archive will not deteriorate and it will be an optimal approximation of bounded size.}
	% \MIQING*{Good point!}
	For some problems, modern computers may be able to keep  hundreds of thousands of their solutions in memory, thus an unbounded archive becomes increasingly viable for some applications~\cite{FieEveSing2003tec,BroTus2019bench,IshPanSha2020unbounded}.
	% \MANUEL*{why \cite{Brockhoff2019}}
	% \MIQING*{Fixed - it should be another of Brockhoff et al which compares MOEAs on bi-BBOB using an unbounded archive.}
	Research involving an unbounded archive includes
	directly using it to store high-quality solutions generated by an MOEA~\cite{Ish1998,RudTraSen2013evenly,RudSchGri2016coa,BroTus2019bench,PanIshSha2020,BezLopStu2019gecco},
	incorporating it into an MOEA as an important algorithm component~\cite{WanSunJin2019multi},
	% \MANUEL*{\cite{Zhong2021} this paper does not use an unbounded archive, just a large one. It doesn't seem that large to me (10 times the population size).}
	% \MIQING*{You are right - the paper removed.}
	benchmarking various MOEAs~\cite{TanIshOya2017,TanOya2017gecco,BroTus2019bench},
	benchmarking bounded versus unbounded archivers~\cite{BezLopStu2019gecco},
	using it to identify if the search stagnates~\cite{Li2023},
	and designing efficient data structures for it~\cite{FieEveSing2003tec,YueGaoWag2012cec,Gla2017fast,JasLus2018ndtree,NanShaIsh2020reverse,Fieldsend2020data}.
	%\MANUEL{We already discussed parts of this at the end of section III. Perhaps there is a way to minimize the redundancy?}
	In addition,
	selecting from an unbounded memory can be seen as an offline archiving algorithm, in particular, a subset selection problem,
	where the archiver knows the whole input sequence and its task is to select a specified number of solutions to represent the whole archive.
	%\MANUEL{But this is also true for the online many-at-a-time, no? Specially if the size of the batch is much larger than the archive size. Can we use any of these results to say something about the many-at-a-time case?}
	%\MIQING{Yes, many-at-a-time is like a online version of subset selection problem.}
	Several subset selection methods~\cite{BriFriKli2014generic,SinBahRay2019distance,Gu2023},
	along with benchmarking test data~\cite{ShaShuIsh2023is},
	have been proposed.
	They consider various indicators as selection criteria,
	ranging from common ones used in the area
	such as hypervolume~\cite{BriFriKli2014generic,KuhFonPaqRuz2016hvsubset,BriCabEmm2018maximum,GroMan2019hvsubset}, $\epsilon$-indicator~\cite{BriFriKli2014generic} and \IGDplus~\cite{CheIshSha2020subset}
	to similarity-based metrics such as distance-based~\cite{SinBahRay2019distance,ShaIshNan2021subset} and clustering-based ones~\cite{CheIshSha2021clustering}, and bi-criteria (i.e., multiobjectivisation)~\cite{Gu2023}.
	
	Despite the capacity of modern computers,
	a downside of using an unbounded external archive is still its computational cost.
	When the optimisation problem is computationally hard but solution evaluation is fast,
	an algorithm may produce millions of nondominated solutions thus leading to a very slow archiving process, particularly for continuous MOPs which typically have infinitely many Pareto optimal solutions.
	On the other hand,
	costly solution evaluations may imply a simulation process that generates large amounts of data on top of the decision and objective vectors,
	which increases the memory requirements for storing such solutions.
	In some real-world problems,
	solutions may actually map to a particular chemical or physical object,
	whose construction is economically costly,
	and thus the archive is bounded by how many of those objects it can store in the real-world~\cite{Kno2009closed}.
	%\MIQING{Is it possible to provide some references here?}\MANUEL{Done. Do you think we need more?}

	An unbounded archive is mainly used for storing the best solutions found so far.
	It is rarely used as a population to generate new solutions,
	except in very few cases~\cite{KraGlaHan2016unbounded}.
	An unbounded population made up of all nondominated solutions generated may cause harmful genetic drift phenomenon due to over-representation of some areas in the search space,
	especially when the mapping of the search space to the objective space is not uniform.

	%\MANUEL{This does not seem a strong argument against unbounded archives. If querying/updating the archive was instantaneous, there would be no reason to not store every solution (I'm 99\% sure that NSGA-III with an unbounded population is better than NSGA-III  with a bounded population, however, I would not be surprised if NSGA-III with a bounded population and an unbounded archive is better than both.). Perhaps better argument is that:
	%(1) an unbounded archive becomes increasingly computationally costly during optimization, more so, if the archive is involved in the search process, (2) if the optimization problem is computationally hard but solution evaluation is fast, an algorithm may produce hundreds of thousands of nondominated solutions thus leading to large archives; on the other hand, costly solution evaluations often imply a simulation process that generates large amounts of data on top of the decision and objective vectors, which increases the memory requirements for storing such solutions. In some real-world problems, solutions may actually map to a particular chemical or physical object, whose construction is economically costly, and thus our archive is bounded by how many of those objects we can store in the real-world.}

	\subsection{Related Problems in Theoretical Computer Science}
	Speaking of subset selection,
	there is a similar research problem in the field of theoretical computer science:
	given a specified accuracy $\epsilon$,
	determine a minimum set of solutions
	such that any solution of a given set (or of a multi-objective problem) can be $\epsilon$-dominated
	by at least one of its solutions~\cite{PapYan2000focs,VasYan2005effic,KolPap2007approx,DiaYan2008succint,DiaYan2009small}.
	%\manuel{Speaking of subset selection, there is a similar research problem in the field of theoretical computer science: given a specified accuracy $\epsilon$, determine a minimum set of solutions  such that any solution of a given set (or of a multi-objective problem) can be $\epsilon$-dominated by at least one of its solutions}
	%{The following problem studied in theoretical computer science is related to archiving (either offline or online): given a specified accuracy $\epsilon$ and a set of solutions, determine a minimum subset such that any solution in the given set is $\epsilon$-dominated by at least one minimum solution }~\cite{PapYan2000focs,VasYan2005effic,KolPap2007approx,Diakonikolas2008,DiaYan2009small}.
	%\MANUEL{This is essentially the problem tackled by  $\epsilon$-approx and $\epsilon$-Pareto when given a fixed $\epsilon$}
	Its dual problem tends to be more relevant (essentially, an offline archiving problem):
	given a set of nondominated solutions,
	find a specified number of $N$ solutions
	that provide the best approximation to the Pareto optimal set with respect to the $\epsilon$-dominance.
	%\MANUEL{This is essentially the problem tackled by  $\epsilon$-approx and $\epsilon$-Pareto when given a fixed $N$ and adapting $\epsilon$}
	In contrast to using the hypervolume as the selection criterion,
	using the $\epsilon$-dominance provides the decision-maker with a measure of the approximation error $\epsilon$ of the subset selected.
	% to what extent\MANUEL*{with to what extent? Not sure what this means. Do you mean ``provides the DM with a measure of the approximation error of the subset selected''?}
	% \MIQING*{Yes, in which $\epsilon$ is used as the approximation error.}
	However,
	like hypervolume-based subset selection~\cite{BriFri2012tcs},
	solving this problem is difficult~\cite{VasYan2005effic,DiaYan2009small}.
	When the number of objectives is two,
	the problem is already NP-hard despite having a polynomial time approximation~\cite{VasYan2005effic,BriFriKli2014subset};
	%\MANUEL{There is also the algorithm presented in \cite{BriFriKli2014subset}}
	when the number of objectives is larger than two,
	any multiplicative approximation is impossible, unless P$=$NP~\cite{DiaYan2009small}.
	
	%\MANUEL{I think there are several important and recent works missing:
	%  \begin{itemize}
	%  \item \url{https://dl.acm.org/doi/abs/10.1145/3520304.3534076} summarises various results
	%  \item \url{https://dl.acm.org/doi/abs/10.1145/3512290.3528840} a recent archiver based on $\epsilon$-dominance
	%  \item \url{https://www.mdpi.com/2297-8747/27/3/48} an archiver that reaches a Hausdorff approximation
	%    \item \url{https://link.springer.com/chapter/10.1007/978-3-319-44427-7_4} ``SMPSO using different archiving strategies (hypervolume, cosine distance, and aggregation)'' and ``hypervolume contribution based archive, shows the overall best performance''
	%   \end{itemize}}
	%\MIQING{Will add those work as well as other recent work in relevant sections.}
	
	%That is,
	%find a subset of $N$ solutions with the minimal $\epsilon$
	%such that any solution in the nondominated set is $\epsilon$-dominated by at least one solution in the subset.

	\section{Future Research Directions}\label{sec:future}
	
	After providing an overview of important issues in archiving in the previous section,
	this section suggests several research directions that deserve attention in coming years.
	
	\subsection{Developing Archivers with Theoretical Desirables and Practical Use}
	
	Most existing work in the EMO area focuses on the practical performance of archiving algorithms
	(e.g., with respect to the hypervolume and IGD indicators),
	ignoring their theoretical properties.
	However, an archiving algorithm %
	with any of the limit properties avoids that the same solutions enter and exit the archive repeatedly, which causes fluctuation of the quality of the archive~\cite{FieEveSing2003tec}, while the point- and set-monotonicity properties avoid that the quality of the archive/population deteriorates over time~\cite{Li2021telo}.
	
	Fortunately,
	archivers in Classes III and IV (cf. Table~\ref{Table:classifcation})
	have the potential to hold both theoretical and practical desirables,
	in contrast to those in Classes I and II where either of them is missing.
	%Amongst these three classes,
	%C3 is the only one respecting the property \textit{point-monotone}.
	%However,
	%the major issue of the archivers in this class is
	%that the size of their archive is not controllable.
	%\MANUEL{we should define what this means} and typically a predefined $\epsilon$ is required to specify.
	%Despite their practical shortcomings, archivers in Class II are the only ones respecting the property \textit{point-monotone}. Within Class II,  $\epsilon$-Pareto and $\epsilon$-approx tend to result in an archive size far smaller than the capacity available, thus preventing them from being limit-optimal. It remains an open question whether there exist archivers with similar theoretical properties and whose archive size is always close to the given capacity or, if possible, limit-optimal. Some recent studies are interesting attempts to address this question~\cite{HerSch2022archive,HerSch2022hausdorff}.
	Archivers in Class III strike a good balance between theoretical and practical desirables,
	as evidenced by their wide use in the EMO area. 
	However, they may not hold \emph{set-monotonicity}, 
	which has the risk of the archive/population deteriorating over time.
	%\MIQING*{Many archivers in Class III can respect \emph{set-monotone} (at least conditionally).}\MANUEL*{Is it ok now?}
	%\MANUEL{Not being set-monotone is a really bad downside: they will deteriorate! Wide use is not evidence of being good.}
	%\MIQING{will modify this after re-categorising classes.}
	Class IV is a class having high potential to be explored. It is the only class that guarantees limit-optimal, diversification and a controllable size.
	Given that archivers based on weakly Pareto compliant indicators
	can hold the same theoretical desirables as those based on Pareto compliant indicators,
	we expect that more archivers from Class IV will emerge in the near future.
	%In fact,
	%the archiver MGA in C5 has shown its strength on synthetic and realistic problems~\cite{LopKnoLau2011emo,LopLieVer2014ppsn,BezLopStu2019gecco};
	%MGA has been found to be very competitive to or even better than hypervolume-based archivers in terms of the $\epsilon$-indicator~\cite{LopKnoLau2011emo,LopLieVer2014ppsn}.
	%\MANUEL{This paragraph could be a bit more precise on what is shown and what is not shown: \cite{LopKnoLau2011emo} only tested synthetic sequences. \cite{BezLopStu2019gecco} considered DTLZ and WFG benchmarks and \cite{LopLieVer2014ppsn}  multi-objective NK-landscapes (the quality of MGA and $\AHV$ is similar in terms of $\epsilon$-indicator, while the latter is sometimes better in terms of hypervolume.}

	%though there is room for improvement as discussed previously
	%(e.g., its mechanism of random removal of solutions at the level $\beta$).

	\subsection{Order of Solutions Arriving}\label{sec:order}
	
	In contrast to extensive archiving studies on the effect of the number of objectives and the Pareto front shape,
	there are very few works studying how the order of solutions in the sequence affects archivers.
	It has been shown~\cite{LopKnoLau2011emo,Li2021telo} that commonly-used archivers,
	such as NSGA-II, SPEA2, MOEA/D, SMS-EMOA and NSGA-III,
	may not be reliable on even the simplest Pareto fronts (i.e., 1D/2D simplex shapes)
	if solutions arrive one-at-a-time ($\mu+1$) in pathological orders.
	%\MANUEL{For NSGA-II and SPEA2, it was shown in \cite{LopKnoLau2011emo}}
	%\MANUEL{Is this only for the one-at-a-time case?}
	
	Although solutions generated by a search algorithm are expected to get better over time,
	input sequences may greatly differ in practice depending on the optimisation problem and search algorithm.
	In particular, the
	landscape of optimisation problems may  produce quite different patterns of solution sequences.
	For example,
	well-established test problems KUR~\cite{Kur1990variant} and UF~\cite{ZhaZhoZha2009cec} typically lead MOEAs to start their search from a particular region and gradually move to others.
	Problems involving many local optima in the search space (e.g., DTLZ3~\cite{DebThiLau2005dtlz}, ML-DMP~\cite{LiGroYanLiu2018multi} and MNK-landscapes~\cite{VerLieJou2013ejor})
	easily lead to MOEAs generating dominance resistant solutions during the search.
	In many real-world problems,
	particularly problems with strict constraints or prioritised objectives,
	the search often starts from a tiny feasible region and then gradually expands to large regions,
	such as in the test suite generation for software product line~\cite{HieLiLiuPar2020}
	and in resource allocation for software testing~\cite{SuZhaYue2021enhanced}.
	
	On top of various sequences resulting from optimisation problems,
	there exist many multi-objective optimisers that tend to search for solutions in a certain order.
	For example,
	the algorithm in~\cite{PaqStu2003tpls},
	developed for the bi-objective TSP problem,
	starts the search from an extreme solution
	and then gradually moves to the other extreme.
	The algorithm presented in~\cite{DubLopStu2011amai} generates search directions that aim to fill the largest gap in the current approximation of the PF.
	The algorithms in~\cite{LusTeg2009tpls,HeYen2016many} search first for all extreme solutions of the Pareto front and then trade-off solutions between those.
	Pareto local search algorithms generate solutions that are neighbours in the decision space of a single solution taken from its archive,
	and thus the generated sequences often consist of very similar solutions~\cite{PaqChiStu2004mmo,DubLopStu2015ejor,LopLieVer2014ppsn}.
	%\MANUEL{However, as far as I can remember, all of these works, except \cite{LopLieVer2014ppsn}, use unbounded archives. Not sure if we need to mention it here but it would be good to mention it somewhere. }
	Similar search strategies are also common in conventional mathematical optimisation~\cite{SteRad2008computing,DasDiaYan2016chord}.
	%% MANUEL: I'm not sure what "based on their locations in the space" means because there are lots of different approaches: based on scalarization, epsilon-constraint, hyper-boxes, etc. But we don't need to give so many details anyway. 
	%as they typically compute solutions one by one based on their locations in the space~\cite{SteRad2008computing,DasDiaYan2016chord}.

	In short,
	the variety of optimisation problems' nature and search algorithms' behaviour
	(as well as the stochasticity of MOEAs) may lead to different types of solution sequences.
	Investigating their effect on archivers and, hence, developing reliable algorithms on various sequences
	are a potential direction waiting to be explored.

	\subsection{Archiving Based on Specific Indicators}
	Archivers based on a specific indicator represent the archive's quality through a scalar value and
	aim to find the archive that maximises/minimises that value.
	Frequently used indicators for this end include hypervolume~\cite{ZitThi1998ppsn},
	IGD~\cite{CoeSie2004igd}, $\epsilon$-indicator~\cite{ZitThiLauFon2003:tec}, Hausdorff indicator~\cite{SchEsqLarCoe2012tec},
	R2~\cite{HanJas1998}, and
	\IGDplus~\cite{IshMasTanNoj2015igd},
	%\MANUEL{We have already listed these indicators above AND the survey. It seems redundant to list them twice.}
	which can cover both the proximity to the Pareto front and the diversity along the front.
	As discussed in Section~\ref{sec:classIII}, such archivers
	% \MANUEL*{Which archivers? We list indicators above, but it is not clear which archivers, except for hypervolume and R2, maximise those indicators}\MIQING*{I guess all indicators (which map a solution set into a scalar value to be maximised/minimised) may have potential to have some properties like \emph{limit-stable}.}
	hold (or can
	% \MANUEL*{can because someone has done it? or could because some could do it in the future?}\MIQING*{As discussed in Section V-C, setting the rule of always removing the new solution when it and an old solution have the same worst indicator contribution enables the archiver to have some limit properties}
	be modified to hold) the limit properties and (weakly) Pareto compliant indicators enable the archiver to hold the three limit desirables. Indicators that are not compliant with Pareto dominance may enable the design of an archiver that is \textit{limit-stable}, as long as the archive is monotone with respect to the indicator value.
	%\MANUEL*{Only if there is a maximum value for the indicator and only if the archive is strictly increasing with respect to the indicator value, which is probably not the case as such archivers may be willing to reduce the indicator value by replacing a dominated solution by a nondominated one. Perhaps we should delete the text after ``and indicators\ldots''?}
	%\MIQING*{Yes, I agree, so I wrote ``can be modified to hold'' :-), but anyway, what I want to say is that if designed properly, all (weakly) Pareto compliant indicators can hold the three limit properties and all Pareto non-compliant indicators can hold \emph{limit-stable}.}
	%\MIQING{Existing indicator-based archivers may not have these desirables, but one should be able to design new archivers that have these desirables based on the corresponding indicators.}
	
	Since an indicator-based archiver aims to maximise (or minimise) the indicator value,
	it is always of interest to know how good a value can be achieved by the archiver theoretically.
	There are several studies on this topic~\cite{BriFri2010ppsn,BriFri2014convergence,RudSchGri2016coa}.
	However, far more work is needed to understand the theoretical limitations of such archivers.
	%\MANUEL*{We already mention \cite{SchHerTal2019archiver} earlier and it doesn't really fit here since it is a different type of archiver. I would just not include in the list.}
	%\MIQING*{Good point!}
	
	%\MANUEL{Maybe we should remark that this archiver \cite{SchHerTal2019archiver} aims to store a set whose elements are not  $-\epsilon$-dominated, which is a different goal from other archivers. It is difficult to apply the analysis done here to it.}
	
	In addition,
	a well-established concept in the theory of online algorithms,
	called competitive analysis~\cite{BorEly1998online,Albers2003online},
	fits nicely in evaluating indicator-based archivers
	(as well as other archivers as long as a scalar quality indicator is used to perform competitive analysis).
	%\MANUEL{why only indicator-based? It would be helpful for any online archiver!}
	%\MIQING{in online archiving, do we need a quality indicator that converts a set of vectors into a scalar so we can compare relative quality between online and offline algorithms?}
	%\MANUEL*{Yes, we need a scalar quality metric to perform competitive analysis but the archivers being analysed do not need to be indicator-based.}
	%\MIQING*{will amend it later.}
	Competitive analysis compares the relative
	performance of an online algorithm and an offline algorithm for the same
	sequence, i.e., how much worse the online algorithm performs due to not knowing the
	future input.  Specifically, the competitive ratio of an algorithm is defined
	as the worst-case ratio of its quality divided by the optimal quality, over all
	possible sequences.  The optimal quality can be defined by using the
	unbounded archive \cite{LopKnoLau2011emo} or using the best possible
	bounded size archive or perhaps something else that is actually achievable.
	\citet{LopKnoLau2011emo} suggested ``\emph{to use competitive analysis
		techniques from the field of online algorithms to obtain worst-case bounds,
		in terms of a measure of `regret' for archivers}'', yet, to the best of our
	knowledge, the only analysis available is the work of
	\citet{BriFri2014convergence}, who defined the competitive ratio based on the hypervolume metric, proved upper and lower bounds of this competitive ratio  for different classes of hypervolume-based
	archivers and presented a computationally-efficient hypervolume-based archiver with a constant competitive ratio.  Their analysis is based on the best-case and worst-case input sequences
	and they pointed out that an average-case analysis may lead to a different choice of the archiver. A similar analysis for other types of archivers and competitive ratios based on other quality metrics, such as the $\epsilon$-indicator or \IGDplus, remains to be done.
	
	Theoretical analysis on competitive ratios or regret according to various quality indicators could be complemented by empirical analysis that is not restricted to archivers explicitly using the quality indicators being measured.
	\citet{LopKnoLau2011emo} measured the ratio between the quality, in terms of hypervolume and $\epsilon$-indicator, of various archivers and of the unbounded archiver.
	
	% \MANUEL{This is not correct. They analyse hypervolume-based regret but several of their results apply to any type of archiver not just hypervolume-based ones (Th VI.2 and VI.3). They also introduce other theoretical properties, like locally optimal, greedy, non-decreasing and decreasing that are defined in terms of hypervolume but are applicable to any archiver. I have the impression that some of these can be related to existing properties, such as diversifying, set-monotone, etc. }\MANUEL{The conclusions of \cite{BriFri2014convergence} are very interesting as they point out to several open questions that seem to be still open.}
	% \MIQING{Good point! Manuel could you please help to add and edit this part based on the above observations? thank you!}
	
	Further theoretical development would be welcome. For example, the limit properties (Props.~\ref{def:limit-stable}, \ref{def:limit-pareto} and~\ref{def:limit-optimal}) are not very useful in practice unless the time to converge to the limit is tractable. Thus, bounds on the number of steps/input solutions required to reach the limit would be of practical interest.
	
	% \MANUEL{Further theoretical developments would be welcome. For example, the limit properties, although nice are not very useful in practice unless the convergence time is bounded. Thus, runtime convergence bounds would be extremely useful. Another possible analysis is regret in the one-pass online archiving, how much can the archive deteriorate with respect to hypervolume metric or other metrics over the best possible hypervolume. Similar analysis for $\epsilon$ would be interesting. \cite{LopKnoLau2011emo} measured this regret and it does seem to be bounded for \AHV and MGA. Best case bounds would also be interesting if there are some archivers that cannot achieve the best quality possible even if the sequence is in the best possible order.}
	% \MIQING{I guess runtime compleixty analysis is more about a search algorithm, i.e., for a given problem, how quick the search algorithm finds its optimal solution(s). But I do agree regret is very important in the one-pass online archiving, which seems less relevant to this section. So I think we can add a new subsection about it and other theoretical issues if any. Manuel could you please help on this?}

	\subsection{Internal Archive vs External Archive}
	
	Two major roles of archiving in EMO are to
	(1) store a set of representative Pareto optimal solutions for \emph{a posteriori} decision-making and
	(2) maintain a set of high quality solutions as the source to generate offspring.
	The different purposes of the two roles may need different archiving algorithms,
	though existing work usually does not distinguish them,
	e.g., the hypervolume-based archiver is widely used for both roles.
	A recent study has shown that a combination of relatively small internal archive/population of Class I
	and a large external archive of Class IV
	may be a good choice~\cite{BezLopStu2019gecco}. 
	The internal archive is focused on searching for promising solutions,
	while the external archive is focused on storing the best solutions found. In this setup, it may not matter
	if the internal archive \emph{set-deteriorates} or it is not \emph{limit-optimal}
	(an invariant population is not helpful for the search) as long as the external archive is.
	Nevertheless,
	much more studies are needed to investigate which are the best combinations of internal and external archivers.
	
	%\MANUEL{I added some comments about this earlier before reading this part. Nevertheless, it would be interesting to point out earlier that the analysis of theoretical/practical properties only focuses on role (1).}
	%In addition,
	%different archiving criteria may be suitable for different roles.\MANUEL{Isn't this sentence saying the same as ``The different purposes of the two roles may need different archiving algorithms''?}
	%The pattern of ``Pareto dominance'' $+$ ``density'' may work for an external archive
	%as a set of solutions that well represents the problem's Pareto front is ideal for the decision-making,
	%while indicator-based or decomposition-based criteria may work better for an internal archive
	%as they can provide ``fine grained'' rules
	%that are able to identify high quality solutions as the parents to generate offspring.\MANUEL{I don't think this is correct for density versus indicator-based. In fact, indicator-based is probably ideal for (1) because it can guarantee some degree of quality compliant with Pareto-optimality, whereas density and decomposition-based may help escape from local Pareto-optimality and enforce diversity/niching even if the result is not necessarily Pareto-optimal. But I don't know of any results in this direction.}
	
	An external archive can also be used to
	monitor the evolutionary status of the internal archive/population.
	For example,
	in~\cite{LiYao2020ec} an external archive based on Pareto dominance and density criteria is used to check
	if the decomposition-based internal archive is trapped in partial regions of the optimal front.
	In this regard,
	it is beneficial that the different archivers consider complementary archiving criteria.
	This is one of the major reasons behind the development of various two-archive MOEAs~\cite{WanJiaYao2015twoarch2,QiMaLiu2014moead,LiYanLiu2016tec,LiCheFuYao2018twoarch,LiuYenGon2018twoarch,CheLiYao2017dynamic}.
	Moreover,
	automatically-designed MOEAs show that diverse choices of archiving criteria for environmental selection and external archiving often outperform well-known popular MOEAs,
	even after tuning the parameters of the latter~\cite{BezLopStu2015tec,BezLopStu2019ec}.

	\subsection{Interplay between Archiving and Solution Generation}
	As aforementioned,
	in the absence of an external archive that does not participate in the search,
	the population of an MOEA is used not only to store the best solutions found so far
	but also to generate new solutions.
	This is also the case for some multi-objective local search algorithms using bounded archives~\cite{LopLieVer2014ppsn,BloPerJouKesHoo2017emo}.
	That means that the sequence of solutions fed to the population is generated by itself.
	In this case,
	we may need to consider other factors in the archiving operation on top of solutions' quality,
	e.g., the life cycle of solutions in the archive.
	We want to exploit the very best solutions in the archive
	but may also want to explore new areas in which newly-generated, next-best solutions are located.
	This is essentially a problem of balancing exploration and exploitation. In multi-objective local search, solutions are marked as ``explored'' after their neighborhood is (partially or fully) explored~\cite{PaqChiStu2004mmo,DubLopStu2015ejor,LieHumMes2011}.
	Some MOEAs introduce the concept of ``ageing'' to prevent old high-quality solutions in the population from generating new solutions in order to help the search jump out of local optima~\cite{SchLip2011age}.

	Given the above, perhaps
	using a population serving both roles of storing and generating solutions is not ideal,
	despite the fact that it is the common practice in the EMO area.
	The final population returned to the decision-maker may contain many dominated solutions with respect to the sequence generated (i.e., all the solutions generated),
	while nondominated solutions may be discarded in the middle of the search~\cite{LiYao2019emo}.
	Therefore,
	an external archive that stores best solutions found is desirable.
	It is worth mentioning that this observation also applies to the two-archive/population approach
	since both archives/populations participate in the search.
	
	\section{Guidelines}\label{sec:guide}
	In this section, we provide guidance on how to choose an appropriate archiver for a given problem and also on how to identify which category an archiver belongs to for an MOEA. 
	
	\subsection{Choosing an Appropriate Archiver}

		The choice of an appropriate archiver depends on properties of the optimisation problem in hand as well as the search algorithm used. 
		If the optimisation problem is somewhat expensive in the sense that not many solutions can be generated in practice, the unbounded archiver is the best option since it never deteriorates and preserves all nondominated solutions ever generated.

		On the other hand, 
		if the optimisation problem is cheap in the sense that hundreds of thousands (or millions) of solutions can be potentially generated, archivers in Classes III and IV, which have theoretical and practical desirables, are better options.
		This is particularly the case for some recent MOEAs in which a large population size and many generations are required to guarantee good performance, e.g., MOEAs based on stochastic population update~\cite{Bian2023}. 
		Taking the non-elitist MOEA developed in~\cite{Liang2023} as an example, a population of 10,000 solutions and 5,000 generations were used.
		
		Amongst the two classes that have both theoretical and practical desirables (i.e., Classes III and IV), Class IV is in general more preferable since the archivers hold the property \emph{limit-optimal}, that is, a bounded archiver will eventually converge to an optimal approximation of bounded size if any Pareto-optimal solution may appear in the sequence an infinite number of times.   
		Currently, there are two groups of archivers in Class IV: the hypervolume-based archivers and MGA.
		A hypervolume-based archiver (e.g., the one in SMS-EMOA) is recommended if the quality of solution sets is more important than computational effort.
		From a theoretical and efficiency point of view, MGA is a good option; however, as reported in a recent empirical study~\cite{Li2021telo}, it may lead to not-uniformly-distributed archives and may lose boundary solutions. 
		Fortunately,
		we have proven here that any archiver based on a weak Pareto compliant indicator (e.g., $\epsilon$, \IGDplus and R2) can achieve the property \emph{limit-optimal}, if designed properly (e.g., following the steps of  Algorithm~\ref{Alg:Weak_compliance}). 
		Such archivers can be used when the hypervolume-based archiver is not suitable or when the decision maker's preferences are not in line with the hypervolume indicator.
		
		In addition,
		it is worth stressing that archivers having better theoretical desirables do not always imply better performance in practice. 
		For example, \AHV~\cite{KnoCor2003tec}, which sets a fixed reference point in the calculation of hypervolume, unconditionally holds \emph{set-monotone}; however, in many cases it would perform worse than the archiver in SMS-EMOA~\cite{BeuNauEmm2007ejor}, which conditionally holds that property, since a too-far or too-close reference point may lead to poor diversification of the solution set~\cite{IshImaSet2018refpoint}.
		As another example, MOEA/D-TCH, which holds more theoretical desirables than MOEA/D-PBI, may perform worse than the latter in some cases~\cite{ZhaLi07:moead}. 
		In short, when choosing an archiver, one needs to consider not only its theoretical properties but also its practical use, and the latter is often relevant to the implementation of archivers, the nature of the considered problem, and the decision-maker's preferences.
		
		The above guidance is for external archiving in EMO when the archive is not involved in the search, i.e., solutions in the sequence do not depend on which solutions are archived.
		In MOEAs, the population update process (the internal archiving)
		not only aims to preserve high-quality solutions in the population, but also to identify promising areas that are not well explored. 
		An option for designing archivers for MOEAs is to consider very different rules in the external and internal archivers; 
		for example, to allow the internal archiver to focus on diversity and exploration rather than theoretical guarantees of quality and convergence. Unfortunately, the best choice of an internal archiver when combined with an external one is still an open research question, but automatic design approaches can help in this choice~\citep{BezLopStu2019ec}.

	\subsection{Identifying the Archive Category}

		Given an MOEA, we may identify its category in Table~\ref{Table:classifcation} (more precisely, the category of its population update method) as follows. 
		We can distinguish three types of population update methods: Pareto-based, indicator-based and decomposition-based methods. 
		Pareto-based methods, which rank solutions by Pareto dominance first and a density estimator second, do not meet any theoretical desirables and belong to Class~I.
		
		Indicator-based methods, which use an indicator to measure the quality of a solution set (i.e., the population) and aim to minimise (or maximise) the indicator value, will at least meet some theoretical desirables, under certain rules\footnote{The archiving/selection rules need to ensure rejecting a new solution if the solution has the same contribution to the indicator value as the worst solution in the set to (see \ARtwo in Table~\ref{Table:classifcation}).}. 
		%(see the end of Section~\ref{sec:theory_weakly}).
		MOEAs guided by an indicator that is not weakly Pareto compliant, such an IGD-based one, will be \emph{limit-stable} as long as they remove the newest solution when two solutions have the same indicator contribution.  
		If the indicator is weakly Pareto compliant and the archive may contain duplicated solutions,
		then the corresponding MOEA will meet many theoretical desirables except \emph{limit-optimal}, such as R2-based MOEAs.
		Both types of MOEAs belong to Class~III. On the other hand, if the indicator is at least weakly Pareto-compliant and the archive never contains  duplicated solutions like Algorithm~\ref{Alg:Weak_compliance}, then the corresponding MOEA will meet \emph{limit-optimal}, thus belonging to Class~IV.
		
		Decomposition-based methods, which decompose the space through a number of weight vectors, may or may not hold theoretical desirables, depending on whether its archiving rules implicitly optimise a unary quality indicator that induces a total order between solution sets.
		% there exists a quality indicator that provides a total order of solution sets and
		%(i.e., the output being a scalar value, like those in indicator-based algorithms)
		% that
		% can represent the archiving rules.
		For example, in MOEA/D-TCH, the archiving rules can be converted into a scalar value that is essentially the R2 indicator, so the algorithm belongs to Class~III.
		In NSGA-III, there does not exist an indicator that reflects the archiving rules (as both Pareto dominance and the distance to the weight vectors are considered), so no theoretical desirables are met.
		
		To sum up, an MOEA will meet
		some theoretical desirables (at least \emph{limit-stable}) if we can prove that its archive never decreases some quality indicator according to which the set of possible solutions sets can be ordered.
		%
		%
		% This, in fact, applies to any MOEA. 
		% If the archiving rules of an MOEA can be represented by an indicator, according to which the set of possible solutions can be ordered, then the algorithm meets some theoretical desirables (certainly including \emph{limit-stable}).

	\section{Conclusion}\label{sec:conclusion}
	% \manuel{Multi-objective optimisers generate solutions sequentially,
	% and an archive that stores a set of representative high-quality solutions is usually needed,
	% particularly for those that generate solutions on the basis of the current solutions in a probabilistic way.}{Multi-objective optimisers may generate thousands of nondominated solutions in a single run and, despite advances in computing technology, it is often impractical to store all those solutions. If solution evaluation is fast, storing those evaluated solutions is often cheap as well, however, we may wish to evaluate hundreds of thousands of them. If solution evaluation requires hours or days, then storing those evaluated solutions (and all their simulated or physical by-products) is often costly, and we may be able to store only a few of them. In both scenarios there is  the need of a bounded-size archive that stores a set of representative high-quality solutions.}
	% \MANUEL*{I'm not convinced that being probabilistic and generating solutions sequentially are the reasons why we need archivers. Please let me know what you think of my rephrasing.}
	Bounded archiving, i.e., storing a bounded set of representative high-quality 
	solutions is of good use in multi-objective optimisation. Not only is unbounded 
	archiving computationally impractical in many scenarios, but also a bounded 
	archive or population may help the search.
	% \MIQING*{I am wondering if we might make it more general and high-level in the conclusion. The above seemingly makes an impression to the reader that the need of having an archive is only for computational cost consideration; actually, archiving could be of multiple uses, as explained in the paper and also most of what discussed here is applicable to other components in EMO, e.g., environmental selection.}
	% \MANUEL*{Is it better now?}
	In this paper,
	we conducted a systematic survey of multi-objective archiving,
	including
	\begin{itemize}
		\item We reviewed the formalisation of the archiving problem and of six desirable theoretical properties of bounded archivers. We extended these definitions to archivers that may store dominated solutions.
		\item We showed analytically that archivers based on a weakly Pareto compliant indicator (e.g., $\epsilon$, IGD$^+$ and R2)
		can achieve the same theoretical properties as archivers based on a Pareto compliant indicator (e.g., hypervolume).
		\item We exemplified representative archivers (including those in well-established MOEAs) and
		classified these archivers into four classes based on their theoretical and practical properties.
		\item We discussed important issues in designing and analysing multi-objective archivers.
		\item We suggested future research lines and pointed out several open questions. 
		%that require further theoretical and empirical research.
		\item We provided guidance on choosing appropriate archivers and identifying the archiver category for a given MOEA.
	\end{itemize}

	%We described six theoretical properties that an archiver can hold and
	%based on them we classified existing archiving algorithms into four classes.
	%We then discussed important issues in multi-objective archiving.
	%Lastly,
	%we suggested several future research lines,
	%in which developing archivers with theoretical desirables and practical use plays a central role.
	%In this regard,
	%the class C4,
	%which holds the \textit{limit-optimal} desirable,
	%is the one having plenty of room to be explored,
	%since like those based on a Pareto compliant indicator,
	%archivers based on a weakly Pareto compliant indicator may also fall into it if designed properly,
	%as we showed in the paper.
	%\MANUEL{I think we also pointed out several open questions that require further theoretical and empirical research.}
	%\MIQING{Xin suggested that in the Conclusion we may make our contributions clearer.}
	%\MIQING{Will revise the Conclusion part. One of the contributions is that we raise some open questions. Manuel, do we need to explicitly state those open questions somewhere in the Further Research Direction section?}
	
	\section{Acknowledgment}
	
	This work was partially supported by the NSFC (Grant No. 62250710682), Research Institute of Trustworthy Autonomous Systems, the Guangdong Provincial Key Laboratory (Grant No. 2020B121201001), the Program for Guangdong Introducing Innovative and Entrepreneurial Teams (Grant No. 2017ZT07X386), and Shenzhen Science and Technology Program (Grant No. KQTD2016112514355531).

	\scriptsize
	\bibliographystyle{IEEEtranN}
	% MANUEL: I don't have the IEEEabrv file
	%\bibliography{IEEEabrv,bibfile}
	%\bibliography{bibfile,optbib/abbrev,optbib/journals,optbib/authors,optbib/biblio,optbib/articles,optbib/crossref}

% Generated by IEEEtranN.bst, version: 1.13 (2008/09/30)
\providecommand{\MaxMinAntSystem}{{$\cal MAX$--$\cal MIN$} {Ant} {System}}
\providecommand{\rpackage}[1]{{#1}}
\providecommand{\softwarepackage}[1]{{#1}}
\providecommand{\proglang}[1]{{#1}}

\end{document}